\documentclass{article}

\usepackage{iclr2023_conference,times}
\usepackage{amsmath,amsfonts,bm}

\def\eqref#1{equation~\ref{#1}}
\def\1{\bm{1}}

\DeclareMathAlphabet{\mathsfit}{\encodingdefault}{\sfdefault}{m}{sl}
\SetMathAlphabet{\mathsfit}{bold}{\encodingdefault}{\sfdefault}{bx}{n}

\usepackage{xcolor}         % colors

\usepackage[utf8]{inputenc} % allow utf-8 input
\usepackage[T1]{fontenc}    % use 8-bit T1 fonts
\usepackage{url}            % simple URL typesetting
\usepackage{booktabs}       % professional-quality tables
\usepackage{amsfonts}       % blackboard math symbols
\usepackage{nicefrac}       % compact symbols for 1/2, etc.
\usepackage{microtype}      % microtypography
\usepackage{xcolor}         % colors
\usepackage{longtable}
\usepackage{booktabs}
\usepackage{graphicx}

\usepackage{subfigure}
\usepackage{caption}
\usepackage{textcomp}

\usepackage{booktabs}
\usepackage{multirow}
\usepackage{wrapfig}
\usepackage{graphicx}
\usepackage{enumitem}
\usepackage{soul}
\usepackage{amsmath,amscd,amsbsy,amssymb,amsfonts,latexsym,url,bm,amsthm}
\def\BibTeX{{\rm B\kern-.05em{\sc i\kern-.025em b}\kern-.08em
		T\kern-.1667em\lower.7ex\hbox{E}\kern-.125emX}}
\usepackage{algorithm}
\usepackage{algorithmic}
\usepackage[algo2e,ruled,linesnumbered,vlined,boxed,commentsnumbered]{algorithm2e}
\usepackage{threeparttable}
\usepackage{listings}

\usepackage{amssymb}% http://ctan.org/pkg/amssymb
\usepackage{pifont}% http://ctan.org/pkg/pifont

\usepackage{multirow}
\usepackage{multicol}
\usepackage[vlined,boxed,commentsnumbered,linesnumbered,ruled]{algorithm2e}

\newtheorem{theorem}{Theorem}

\newtheorem{proposition}{Proposition}

\usepackage[colorlinks,
            linkcolor=black,
            anchorcolor=blue,
            citecolor=purple,
            ]{hyperref}

\newcommand{\modify}[1]{{\color{black} #1}}

\newcommand{\model}{\textsc{DIFFormer}\xspace}

\usepackage{makecell}
\title{Discovering Data Geometry: \\
Energy-Driven Neural Diffusion}
\title{Diffusion Improves Learning Beyond Graphs with Implicit Energy Regularization}
\title{Energy-Driven Diffusion Graph Networks for Semi-supervised Learning}
\title{Graph-Optimal Diffusion Networks: \\
Discover Data Geometry through Energy Constraint}
\title{Geo-Former: Energy-Driven Geometric Diffusion Networks}
\title{Energy-Guided Geometric Diffusion Networks for Semi-supervised Learning}
\title{Geometric Networks Induced by Energy Constrained Diffusion}
\title{DIFFormer: Scalable (Graph) Transformers Induced by Energy Constrained Diffusion}

\author{
  Qitian Wu$^\dagger$, Chenxiao Yang$^\dagger$, Wentao Zhao$^\dagger$, Yixuan He$^\ddagger$, David Wipf$^\S$, Junchi Yan$^\dagger$\thanks{Corresponding author: Junchi Yan who is also affiliated with Shanghai AI Lab. The work was in part supported by National Key Research and Development Program of China (2020AAA0107600), National Natural Science Foundation of China (62222607), STCSM (22511105100).} \\
  $\dagger$ Department of CSE \& MoE Lab of Artificial Intelligence, Shanghai Jiao Tong University\\
  $\ddagger$ Department of Statistics, University of Oxford \\
  $\S$ Amazon Web Service\\
    \small\texttt{\{echo740,chr26195,permanent,yanjunchi\}@sjtu.edu.cn,}\\ 
    \small\texttt{yixuan.he@stats.ox.ac.uk, davidwipf@gmail.com}
}

\iclrfinalcopy % Uncomment for camera-ready version, but NOT for submission.
\begin{document}

\maketitle

\begin{abstract}
    Real-world data generation often involves complex inter-dependencies among instances, violating the IID-data hypothesis of standard learning paradigms and posing a challenge for uncovering the geometric structures for learning desired instance representations. To this end, we introduce an energy constrained diffusion model which encodes \emph{a batch of instances} from a dataset into evolutionary states that progressively incorporate other instances' information by their interactions. The diffusion process is constrained by descent criteria w.r.t.~a principled energy function that characterizes the global consistency of instance representations over latent structures. We provide rigorous theory that implies closed-form optimal estimates for the pairwise diffusion strength among arbitrary instance pairs, which gives rise to a new class of neural encoders, dubbed as \model (diffusion-based Transformers), with two instantiations: a simple version with linear complexity for prohibitive instance numbers, and an advanced version for learning complex structures. Experiments highlight the wide applicability of our model as a general-purpose encoder backbone with superior performance in various tasks, such as node classification on large graphs, semi-supervised image/text classification, and spatial-temporal dynamics prediction. The codes are available at \url{https://github.com/qitianwu/DIFFormer}.
\end{abstract}

\vspace{-4pt}
\section{Introduction}\label{sec:intro}
\vspace{-4pt}
%Recent years witness many successful adaptation of physics laws in natural science, e.g., conservation~\cite{Lagrangian-networks,Hamiltonian-networks} and symmetry~\cite{geometriclearning-2017,cohen2019gauge,hamzi2021dynamics}, to machine learning models. They either provide principled views for inspiring novel architectures or serve as inductive bias for reducing learning difficulties~\cite{physics-ml-nature}.

Real-world data are generated from a convoluted interactive process whose underlying physical principles are often unknown. %The observational data used for model development or experimental purposes can be viewed as a subset from a hypothetical pool composed of an exhaustive enumeration of all concepts or objects. 
Such a nature violates the common hypothesis of standard representation learning paradigms assuming that data are IID sampled. The challenge, however, is that due to the absence of prior knowledge about ground-truth data generation, it can be practically prohibitive to build feasible methodology for uncovering data dependencies, despite the acknowledged significance. To address this issue, prior works, e.g., \cite{pointcloud-19,LDS-icml19,jiang2019glcn,Bayesstruct-aaai19}, consider encoding the potential interactions between instance pairs, but this requires sufficient degrees of freedom that significantly increases learning difficulty from limited labels~\citep{fatemi2021slaps} and hinders the scalability to large systems~\citep{wunodeformer}. 

%This paper targets a general learning problem that involves a mix of labeled and unlabeled data, a.k.a., semi-supervised learning (SSL)~\cite{semi-2006}. One technical challenge lies in how to leverage the relations among data points, induced by the potential non-i.i.d. nature of data generation process, for informed prediction on individual instances~\cite{belkin2006manireg,weston2012semiemb,Panetoid-icml19}. Current methods make remarkable progress on designing expressive architectures, e.g., graph neural networks (GNNs)~\cite{scarselli2008gnnearly,GCN-vallina,GAT,SGC-icml19,gcnii-icml20,RWLS-icml21}, that harness instance-level inter-connection as geometric prior when observed structure is given as input. However, the structure from observation may not be consistent with the underlying geometry of data, which presumably elicits the performance bottleneck for GNNs' applications to real data.

%Another obstacle stems from the unavailability of observed structure, led by privacy concerns or prohibitive labeling costs, posing a challenge for uncovering the latent relations among instances~\cite{pointcloud-19,gong2020geometrically}. For such a goal, the model would require sufficient degrees of freedom for accommodating or parameterizing the interaction between arbitrary instance pairs~\cite{LDS-icml19,jiang2019glcn,IDGL-neurips20,fatemi2021slaps} that significantly increases the learning difficulty from limited labeled data and hinders the scalability to large systems. 

Turning to a simpler problem setting where putative instance relations are instantiated as an observed graph, remarkable progress has been made in designing expressive architectures such as graph neural networks (GNNs)~\citep{scarselli2008gnnearly,GCN-vallina,GAT,SGC-icml19,gcnii-icml20,RWLS-icml21} for harnessing inter-connections between instances as a geometric prior~\citep{geometriclearning-2017}. 
%when observed structure is given as input. graph-based SSL where instance relations are observed as an input graph and used for regularization~\cite{belkin2006manireg,weston2012semiemb,Panetoid-icml19} or inductive bias that paves the way for modern graph neural networks (GNNs)~\cite{scarselli2008gnnearly,GCN-vallina,GAT}. 
%However, the observational data used for model development generally represents only a subset of an exhaustive pool of all objects/concepts from the physical world, and thus the observed relations can be incomplete or biased. 
However, the observed relations can be incomplete/noisy, due to error-prone data collection, or generated by an artificial construction independent from downstream targets. The potential inconsistency between observation and the underlying data geometry would presumably elicit systematic bias between structured representation of graph-based learning and true data dependencies. While a plausible remedy is to learn more useful structures from the data, this unfortunately brings the previously-mentioned obstacles to the fore.%embodied by emerging evidence showing GNNs' deficiency for heterophily~\cite{}, long-range dependence~\cite{} and noisy/incomplete edges~\cite{}. A plausible remedy is to learn more useful structure from data~\cite{LDS-icml19,jiang2019glcn,gong2020geometrically}, which, however, brings the previous obstacles to the fore.

%Another obstacle stems from the unavailability of observed structure, probably led by privacy issues or economic concerns, posing a challenge for uncovering the relations among instances~\cite{}. 

%\qitian{To resolve the dilemma, our general idea in this paper is to endow the representation model with useful \emph{inductive bias}, which allows information accommodation for arbitrary inter-instance dependence, and proper \emph{regularization} that guides the model for learning desired instance representations}. 
To resolve the dilemma, we propose a novel general-purpose encoder framework that uncovers data dependencies from observations (a dataset of partially labeled instances), proceeding via two-fold inspiration from physics as illustrated in Fig.~\ref{fig:model}. Our model is defined through feed-forward continuous dynamics (i.e., a PDE) involving all the instances of a dataset as locations on Riemannian manifolds with \emph{latent} structures, upon which the features of instances act as heat flowing over the underlying geometry~\citep{hamzi2021learning}. Such a diffusion model serves an important \emph{inductive bias} for leveraging global information from other instances to obtain more informative representations. Its major advantage lies in the flexibility for the \emph{diffusivity} function, i.e., a measure of the rate at which information spreads~\citep{rosenberg1997laplacian}: we allow for feature propagation between arbitrary instance pairs at each layer, and adaptively navigate this process by pairwise connectivity weights.  %The diffusion process plays an inductive bias for representing instances into evolving states that progressively integrate information flows from other particles in the system. %and compose a continuous dynamics (i.e., a PDE) with instantaneous diffusion strength, a.k.a. \emph{diffusivity}, as \emph{latent} variables. The diffusion process encodes all the instances into a series of evolutionary states that progressively integrate information flows from other particles in the system. 
Moreover, for guiding the instance representations towards some ideal constraints of internal consistency, we introduce a principled energy function that enforces layer-wise \emph{regularization} on the evolutionary directions. The energy function provides another view (from a macroscopic standpoint) into the desired instance representations with low energy that are produced, i.e., soliciting a steady state that gives rise to informed predictions on unlabeled data.%We call our model Energy Constrained Neural Diffusion (\model) whose dynamics are \emph{implicitly} constrained by descent criteria w.r.t. the regularized energy, i.e., soliciting a final state that gives rise to informed prediction on unlabeled data.}

\begin{figure}[tb!]
			\centering
			\hspace{5pt}
			\includegraphics[width=0.85\textwidth]{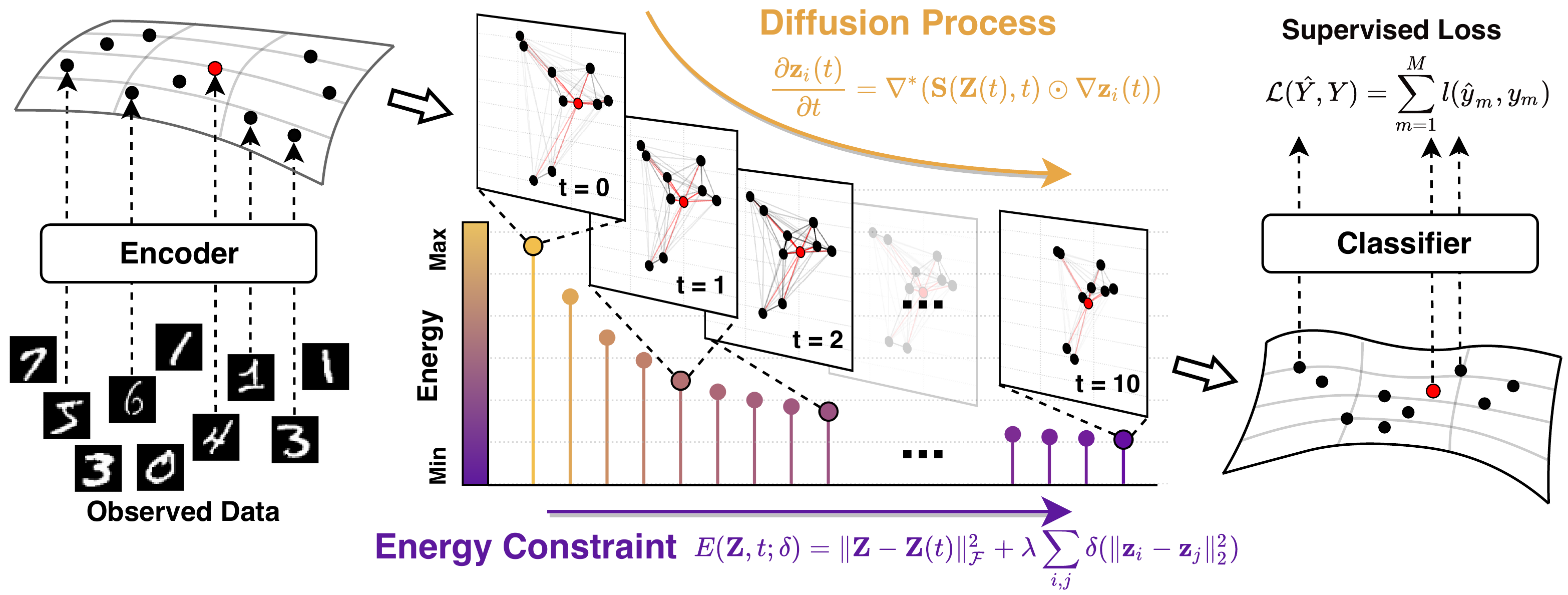}
	    \vspace{-5pt}
		\caption{\looseness=-1 An illustration of the general idea behind \model which takes a whole dataset (or a batch) of instances as input and encodes them into hidden states through a diffusion process aimed at minimizing a regularized energy. This design allows feature propagation among arbitrary instance pairs at each layer with optimal inter-connecting structures for informed prediction on each instance.}
		\label{fig:model}
	\vspace{-20pt}
\end{figure}

%While the above system seems non-trivial to solve at the first glance, 
%To organically integrate the two schools of thoughts into a unified framework, 
As a justification for the tractability of above general methodology, our theory reveals the underlying equivalence between finite-difference iterations of the diffusion process and unfolding the minimization dynamics for an associated regularized energy. This result further suggests a closed-form optimal solution for the diffusivity function that updates instance representations by the ones of all the other instances towards giving a rigorous decrease of the global energy. Based on this, we also show that the energy constrained diffusion model can serve as a principled perspective for unifying popular models like MLP, GCN and GAT which can be viewed as special cases of our framework.
%we theoretically show there exists an explicit estimates for optimal pairwise instantaneous diffusion strength that allows a discretized solution of the implicitly regularized PDE. 

On top of the theory, we propose a new class of neural encoders, Diffusion-based Transformers (\model), and its two practical instantiations: one is a simple version with $\mathcal O(N)$ complexity ($N$ for instance number) for computing all-pair interactions among instances; the other is a more expressive version that can learn complex latent structures. We empirically demonstrate the success of \model on a diverse set of tasks. It outperforms SOTA approaches on semi-supervised node classification benchmarks and performs competitively on large-scale graphs. It also shows promising power for image/text classification with low label rates and predicting spatial-temporal dynamics.

\vspace{-4pt}
\section{Related Work}\label{sec:related}
\vspace{-4pt}

\textbf{Graph-based Semi-supervised Learning.}
Graph-based SSL~\citep{GCN-vallina} aims to learn from partially labeled data, where instances are treated as nodes and their relations are given by a graph. The observed structure can be leveraged as regularization for learning representations~\citep{belkin2006manireg,weston2012semiemb,Panetoid-icml19} or as an inductive bias of modern GNN architectures~\citep{scarselli2008gnnearly}. However, there frequently exist situations where the observed structure is unavailable or unreliable~\citep{LDS-icml19,jiang2019glcn,IDGL-neurips20,fatemi2021slaps,lao2022variational}, in which case the challenge remains how to uncover the underlying relations. This paper explores a new Transformer-like encoder for discovering data geometry to promote learning through the inter-dependence among instances (either labeled or unlabeled).

\textbf{Neural Diffusion Models.} Several recent efforts explore diffusion-based learning where continuous dynamics serve as an inductive bias for representation learning~\citep{hamzi2021learning}. These works can be generally grouped into PDE-based learning, where the model itself is a continuous diffusion process described by a differential equation (e.g., \cite{grand,PDE-GCN,GRAND++}), and PDE-inspired learning, where the diffusion perspective is adopted to develop new models~\citep{atwood2016diffusion,klicpera2019diffusion} or characterize graph geometric properties~\citep{yang2022geometric}. Our work leans on the later category and the key originality lies in two aspects. First, we introduce a novel diffusion model whose dynamics are implicitly defined by optimizing a regularized energy. Second, our theory establishes an equivalence between the numerical iterations of diffusion process and unfolding the optimization of the energy, based on which we develop a new class of neural encoders for uncovering latent structures among a large number of instances.

\modify{\textbf{Transformers.} Transformers serve as a model class of wide research interest showing competitive efficacy for modeling the dependency among tokens of inputs through all-pair attention mechanism. While the original architecture~\citep{transformer} is motivated by the inter-dependence among words of a sentence in NLP tasks, from different aspects, this work targets a Transformer-like architecture for modeling the inter-dependence among instances in a dataset. In the terminology of Transformers in NLPs, one can treat the whole dataset as an extremely long `sequence' (with the length equaling to the dataset size), and each instance (with a label to predict) acts as the `token' in such a sequence. In the context of graph machine learning, our goal can be framed as learning latent interaction graphs among nodes beyond the observed graph (if available), and this can be generally viewed as an embodiment of node-level prediction~\citep{ogb-nips20}. For the latter case, critically though, it remains under-explored how to build a scalable and expressive Transformer for learning node-pair interactions given the prohibitively large number of node instances~\citep{wunodeformer}.}

% \textbf{Instance/Node-Level v.s.~Graph-Level Prediction.} Our goal is to learn latent interaction graphs among instances, and this can be generally viewed as an embodiment of \textit{node-level} prediction (NP) tasks widely studied in the graph learning community. This is distinct from \textit{graph-level} prediction (GP) whereby each graph itself generally has a single label to predict and a dataset contains many graph instances. These two problems are typically tackled separately in the literature~\citep{ogb-nips20} with disparate technical considerations.  This is because input instances are inter-dependent in NP (due to the instance interactions involved in the data-generating process), while in GP tasks the instances can be treated as IID samples. Although there exist some recent models considering all-pair feature propagation among the nodes of each graph instance~\citep{pointcloud-19,graphtransformer-2020}, such prior work largely focuses on GP with relatively small graphs. It therefore remains under-explored how to build an efficient and expressive model for learning node-pair interactions for NP tasks involving latent graphs that can be prohibitively large~\citep{wunodeformer}.

\vspace{-4pt}

\section{Energy Constrained Geometric Diffusion Transformers}\label{sec:model}
\vspace{-4pt}

% For a graph data $G = (V, E)$ with node feature matrix $\mathbf X = \{\mathbf x_u\}_{u\in V}\in \mathbb R^{N\times D}$, adjacency matrix $\mathbf A = \{a_{uv}\}_{u,v\in V}\in \{0,1\}^{N\times N}$ and node label matrix $\mathbf Y = \{\mathbf y_u\}_{u\in V} \in \{0,1\}^{N\times C}$. Denote $\mathbf z_u\in \mathbf R^d$ as the representation for node $u$ and $\mathbf Z = \{\mathbf z_u\}_{u\in V} \in \mathbb R^{N\times d}$. 

% We commence by \cx{introducing} notations: consider a set of instances $\{\mathbf x_i\}_{i=1}^N$ (where $\mathbf x_i\in \mathbb R^D$ denotes input features) that are partially labeled, denoted as $\{(\mathbf x_j, y_j)\}_{j=1}^M$ (where $y_j$ denotes instance label and often $M \ll N$). In some cases there exists relational structure that connects instances as a graph $\mathcal G = (\mathcal V, \mathcal E)$, where the node set $\mathcal V$ contains all the instances and the edge set $\mathcal E=\{e_{ij}\}$ consists of observed relations. Without loss of generality, the main body of this section does \emph{not} assume graph structure as input, and we will later discuss how to trivially incorporate it if available. Moreover, our model is applicable for both transductive and inductive learning scenarios. The former assumes testing instances are within the unlabeled data given at training stage (so is the graph structure if available), and the latter aims at handling new unseen instances during test.

Consider a set of partially labeled instances $\{\mathbf x_i\}_{i=1}^N$, whose labeled portion is $\{(\mathbf x_j, y_j)\}_{j=1}^M$ (often $M \ll N$). In some cases there exist relational structures that connect
instances as a graph $\mathcal G = (\mathcal V, \mathcal E)$, where the node set $\mathcal V$ contains all the instances and the edge set $\mathcal E=\{e_{ij}\}$ consists of observed relations. Without loss of generality, the main body of this section does \emph{not} assume graph structures as input, but we will later discuss how to trivially incorporate them if they are available.

\vspace{-4pt}

\subsection{Geometric Diffusion Model}
\vspace{-4pt}

The starting point of our model is a diffusion process that treats a dataset of instances as a whole and produces instance representations through information flows characterized by an anisotropic diffusion process, which is inspired by an analogy with heat diffusion on a Riemannian manifold~\citep{rosenberg1997laplacian}. We use a vector-valued function $\mathbf z_i(t): [0, \infty) \rightarrow \mathbb R^d$ to define an instance's state at time $t$ and location $i$. The anisotropic diffusion process describes the evolution of instance states (i.e., representations) via a PDE with boundary conditions~\citep{freidlin1993diffusion,medvedev2014nonlinear}:
\begin{equation}\label{eqn-diffuse}
    \frac{\partial \mathbf Z(t)}{\partial t} = \nabla^*\left(\mathbf S(\mathbf{Z}(t), t) \odot \nabla \mathbf Z(t)\right), ~~~ \mbox{s. t.} ~~ \mathbf Z(0) = [\mathbf x_i]_{i=1}^N, ~~ t\geq 0, %+ f(\mathbf z_u^{(0)}; W_F),
\end{equation}
where $\mathbf Z(t) = [\mathbf z_i(t)]_{i=1}^N \in \mathbb R^{N\times d}$, $\odot$ denotes the Hadamard product, and the function $\mathbf S(\mathbf Z(t), t): \mathbb R^{N\times d} \times [0, \infty) \rightarrow [0, 1]^{N\times N}$ defines the \emph{diffusivity} coefficient controlling the diffusion strength between any pair at time $t$. The diffusivity is specified to be dependent on instances' states.
% The gradient and divergence operators in Eq.~\ref{eqn-diffuse} \cx{(i.e., $\nabla$ and $\nabla^*$)} are defined over the discrete space \cx{consisting} of $N$ locations. 
The gradient operator $\nabla$ measures the difference between source and target states, i.e., $(\nabla \mathbf Z(t))_{ij} = \mathbf z_j(t) - \mathbf z_i(t)$, and the divergence operator $\nabla^*$ sums up information flows at a point, i.e., $(\nabla^*)_i = \sum_{j=1}^N \mathbf S_{ij}(\mathbf Z(t), t) \left(\nabla \mathbf Z(t)\right)_{ij}$. Note that both operators are defined over a discrete space consisting of $N$ locations. The physical implication of Eq.~\ref{eqn-diffuse} is that the temporal change of heat at location $i$ equals to
the heat flux that spatially enters into the point. %Notice that we do not assume ground-truth dependence is available or the diffusivity functions reflect the correct data dependence, and instead, we expect the diffusion process can propagate useful information for yielding informative representation/prediction, which can be achieved by learning with labeled data in the dataset.
Eq.~\ref{eqn-diffuse} can be explicitly written as
% the heat on each location $i$ temporally exchanges with others and the instantaneous variation quantity equals to the one spatially outspread.
% , i.e., a discretized analogy of continuous ones: $\mathbf g_{ij}(t) = \mathbf z_j(t) - \mathbf z_i(t)$, $\mathbf d_i(t) = \sum_{j=1}^N \mathbf S_{ij}(t) \mathbf g_{ij}(t)$ where $\mathbf S_{ij}(t)$ denotes the diffusivity between $i$ and $j$ given by $\mathbf S(t)$. 
\begin{equation}\label{eqn-diffuse2}
    \frac{\partial \mathbf z_i(t)}{\partial t} = \sum_{j=1}^N \mathbf S_{ij}(\mathbf Z(t), t) (\mathbf z_j(t) - \mathbf z_i(t)).
\end{equation} 
Such a diffusion process can serve as an inductive bias that guides the model to use other instances' information at every layer for learning informative instance representations. 
%\cx{As we can note from the equation,} the diffusivity function $\mathbf S(t)$ reflects the spatial proximity between any location pairs at a given time $t$, and determines the next states within an infinitesimal step $\mathbf Z(t+\Delta t)$. 
%a measure of the rate at which particles or heat or fluids can spread
We can use numerical methods to solve the continuous dynamics in Eq.~\ref{eqn-diffuse2}, e.g., the explicit Euler scheme involving finite differences with step size $\tau$, which after some re-arranging gives:
% \begin{equation}\label{eqn-diffuse-iter}
%     \frac{\mathbf z_i^{(k+1)} - \mathbf z_i^{(k)}}{\tau} = \sum_{j=1}^N \mathbf S_{ij}^{(k)} (\mathbf z_j^{(k)} - \mathbf z_i^{(k)}),
% \end{equation}
% and we can rewrite Eq.~\ref{eqn-diffuse-iter} as an iterative updating rule:
\begin{equation}\label{eqn-diffuse-iter}
    \mathbf z_i^{(k+1)} = \left (1 - \tau \sum_{j=1}^N \mathbf S_{ij}^{(k)} \right ) \mathbf z_i^{(k)} + \tau \sum_{j=1}^N \mathbf S_{ij}^{(k)} \mathbf z_j^{(k)}.
\end{equation}
The numerical iteration can stably converge for $\tau\in(0, 1)$. We can adopt the state after a finite number $K$ of propagation steps  and use it for final predictions, i.e., $\hat y_i = \mbox{MLP}(\mathbf z_i^{(K)})$.

% \textbf{\emph{Remark.}} The diffusivity coefficient in Eq.~\ref{eqn-diffuse} is a measure of the rate at which particles, heat or fluids can spread over the space. From Eq.~\ref{eqn-diffuse2}, we can see that $\mathbf S_{ij}(\mathbf Z(t), t)$ determines how information flows over instances and the evolutionary direction of instance states. It is thus crucial to design the diffusivity function carefully.
% % Much flexibility remains for its specification. 
% The basic choice is to fix $\mathbf S$ as an identity matrix which constrains the information flows to self-loops and the model degrades to a MLP. Also, one could also specify $\mathbf S(\mathbf Z(t), t)$ as observed graph structure if it is available. In this way, however, the information flows are restricted by neighbored nodes given by the graph, which could lead to undesired over-squashing or bottleneck issues~\cite{oversquashing-iclr21,oversquashing-iclr22}. 
% \todo{Pushing further the limits of geometric models, we are going to unfreeze the topological constraint and assume the diffusivity could propagate information between arbitrary pairs. Also deviating from other related models learning static structure from data~\cite{LDS-icml19,jiang2019glcn,Bayesstruct-aaai19}, our design enables better capacity as $\mathbf S_{ij}(\mathbf Z(t), t)$ could evolve with time, i.e., the structure adapts to specific feed-forward layers.}

\looseness=-1\textbf{\emph{Remark.}} The diffusivity coefficient in Eq.~\ref{eqn-diffuse} is a measure of the rate at which heat can spread over the space~\citep{rosenberg1997laplacian}. Particularly in Eq.~\ref{eqn-diffuse2}, $\mathbf S(\mathbf Z(t), t)$ determines how information flows over instances and the evolutionary direction of instance states. Much flexibility remains for its specification. For example, a basic choice is to fix $\mathbf S(\mathbf Z(t), t)$ as an identity matrix which constrains the feature propagation to self-loops and the model degrades to an MLP that treats all the instances independently. One could also specify $\mathbf S(\mathbf Z(t), t)$ as the observed graph structure if available in some scenarios. In such a case, however, the information flows are restricted by neighboring nodes in a graph. An ideal case could be to allow $\mathbf S(\mathbf Z(t), t)$ to have non-zero values for arbitrary $(i, j)$ and evolve with time, i.e., the instance states at each layer can efficiently and adaptively propagate to all the others. %\modify{This is exactly what we will adopt in this paper (see Section~\ref{sec:model}.3 and \ref{sec:inst} for details), contributing to superior flexibility and capacity of our diffusion model. We also compare the performance of using different $\mathbf S$'s in Section~\ref{sec:exp}.4.}

\vspace{-4pt}
\subsection{Diffusion Constrained by a Layer-wise Energy}
\vspace{-4pt}
As mentioned previously, the crux is how to define a proper diffusivity function to induce a desired diffusion process that can maximize the information utility and accord with some inherent consistency. Since we have no prior knowledge for the explicit form or the inner structure of $\mathbf S^{(k)}$, we consider the diffusivity as a time-dependent latent variable and introduce an \emph{energy function} that measures the presumed quality of instance states at a given step $k$:
\begin{equation}\label{eqn-energy}
    E(\mathbf Z, k; \delta) = \|\mathbf Z - \mathbf Z^{(k)}\|_{\mathcal F}^2 + \lambda\sum_{i,j} \delta(\|\mathbf z_i - \mathbf z_j\|_2^2),
\end{equation}
where $\delta:\mathbb R^+ \rightarrow \mathbb R$ is defined as a function that is \emph{non-decreasing} and \emph{concave} on a particular interval of our interest, and promotes robustness against large differences~\citep{RWLS-icml21} among any pair of instances. Eq.~\ref{eqn-energy} assigns each state in $\mathbb R^d$ with an energy scalar which can be leveraged to regularize the updated states (towards lower energy desired). The weight $\lambda$ trades two effects: 1) for each instance $i$, the states not far from the current one $\mathbf z_i^{(k)}$ have low energy; 2) for all instances, the smaller differences their states have, the lower energy is produced. 

\textbf{\emph{Remark.}} Eq.~\ref{eqn-energy} can essentially be seen as a robust version of the energy introduced by \cite{globallocal-2003}, inheriting the spirit of regularizing the global and local consistency of representations. %The `robust' 
``Robust" here particularly implies that the $\delta$ adds uncertainty to each pair of the instances and could \emph{implicitly} filter the information of noisy links (potentially reflected by proximity in the latent space). %The second term of Eq.~\ref{eqn-energy} also serves as a generalized version of the graph-based regularization~\cite{graphkernel} which has no uncertainty on edges, i.e., $\delta(x) = x$, and only sums up the penalty for linked nodes, i.e., $(i,j)\in \mathcal E$. In contrast, our model does not rely on input graphs and could uncover new structure.  %across spurious potential links.\footnote{We use `potential' here to emphasize that the pairwise interaction is allowed for any node pair in our model beyond the observed topology.} 
%We do not specify the form for $\delta$ at present. The energy function reflects two-fold implications 

 %thus alleviating over-smoothing to a certain degree.

\textbf{Energy Constrained Diffusion.} The diffusion process describes the \emph{microscopic} behavior of instance states through evolution, while the energy function provides a \emph{macroscopic} view for quantifying the consistency. In general, we expect that the final states could yield a low energy, which suggests that the physical system arrives at a steady point wherein the yielded instance representations have absorbed enough global information under a certain guiding principle. 
% Thereby, we integrate two schools of thoughts into a unified model where the diffusion process with latent diffusivity function driving the evolution of instance states towards low produced energy. 
Thereby, we unify two schools of thoughts into a new diffusive system where instance states would evolve towards producing lower energy, e.g., by finding a valid diffusivity function.
Formally, we aim to find a series of $\mathbf S^{(k)}$'s whose dynamics and constraints are given by 
\begin{equation}\label{eqn-diffuse-appx}
    \begin{split}
    &\mathbf z_i^{(k+1)} = \left (1 - \tau \sum_{j=1}^N \mathbf S_{ij}^{(k)} \right ) \mathbf z_i^{(k)} + \tau \sum_{j=1}^N \mathbf S_{ij}^{(k)} \mathbf z_j^{(k)} \\
    &\mbox{s. t.}~~\mathbf z_i^{(0)} = \mathbf x_i, \quad E(\mathbf Z^{(k+1)}, k; \delta) \leq E(\mathbf Z^{(k)}, k-1; \delta), \quad k\geq 1.
    \end{split}
\end{equation}
The formulation induces a new class of geometric flows on latent manifolds whose dynamics are \emph{implicitly} defined by optimizing
a time-varying energy function (see Fig.~\ref{fig:model} for an illustration). 
% \begin{equation}\label{eqn-diffuse-final}
% \begin{split}
%     &\frac{\partial \mathbf z_i(t)}{\partial t} = \sum_{j=1}^N \mathbf S_{ij}(\mathbf Z(t)) (\mathbf z_j(t) - \mathbf z_i(t)) \\
%     &\mbox{s. t.}~~\mathbf z_i(0) = \mathbf x_i, \quad t \geq 0, \quad E(\mathbf Z(t+\Delta t), t; \delta) \leq E(\mathbf Z(t), t - \Delta t; \delta),
% \end{split}
% \end{equation}
% where $\Delta t$ is a infinitesimal time step.
% \begin{equation}\label{eqn-energy-reg}
%     \left\{ 
%     \begin{array}{l}
%          \frac{\partial \mathbf z_u(t)}{\partial t} = \sum_{v\in V} \mathbf S_{uv}(\mathbf Z(t), \mathbf A) (\mathbf z_u(t) - \mathbf z_v(t))\\
%     \mbox{s. t.}~~\mathbf z_u(0) = \mathbf z_u^{(0)}, \quad 0 \leq t\leq T, \quad \mathbf Z(T) = \arg\min_{\mathbf Z} E(\mathbf Z; \delta).
%     \end{array}
%     \right. 
% \end{equation}
%While the initial states can be assumed (e.g., as input node features) and the final states are defined as the solution of energy minimization, the intermediate diffusion process is \emph{implicitly} given with unknown latent diffusivity function. We next discuss how to solve this energy-regularized dynamic system using an unfolding perspective.

\vspace{-4pt}
\subsection{Tractability of Solving Diffusion Process with Energy Minimization}
\vspace{-4pt}

Unfortunately, Eq.~\ref{eqn-diffuse-appx} is hard to solve since we need to infer the value for a series of coupled $\mathbf S^{(k)}$'s that need to satisfy $K$ inequalities by the energy minimization constraint.
The key result of this paper is the following theorem that reveals the underlying connection between the geometric diffusion model and iterative minimization of the energy, which further suggests an explicit closed-form solution for $\mathbf S^{(k)}$ based on the current states $\mathbf Z^{(k)}$ that yields a rigorous decrease of the energy.
% \begin{theorem}\label{thm-main}
%     For $0 < \tau < 1$, when the diffusivity between pair $(i,j)$ at the $k$-th step are given by
%     \begin{equation}\label{eqn-optimal-diffuse}
%         \mathbf{\hat S}_{ij}^{(k)} = \frac{\omega_{ij}^{(k)}}{\sum_{l=1}^N \omega_{il}^{(k)}}, \quad \omega_{ij}^{(k)} = \left. \frac{\partial \delta(z^2)}{\partial z^2} \right |_{z^2 = \|\mathbf z_i^{(k)} - \mathbf z_j^{(k)}\|_2^2},
%     \end{equation}
%     then the instance states given by the one-step finite-difference iteration of the diffusion process using Eq.~\ref{eqn-diffuse-iter} with a step size $\tau$ is equivalent to the ones obtained by one-step gradient descent of a particular energy defined by Eq.~\ref{eqn-energy} with $\lambda(\tau)$ depending on $\tau$. In particular, the feed-forward diffusion process with diffusivity given by Eq.~\ref{eqn-optimal-diffuse} yields a descent step on the energy, i.e., $E(\mathbf Z^{(k+1)}, k; \delta) \leq E(\mathbf Z^{(k)}, k-1; \delta)$ for any $k\geq 1$.
% \end{theorem}
\begin{theorem}\label{thm-main}
    For any regularized energy defined by Eq.~\ref{eqn-energy} with a given $\lambda$, there exists $0<\tau<1$ such that the diffusion process of Eq.~\ref{eqn-diffuse-iter} with the diffusivity between pair $(i,j)$ at the $k$-th step given by
    \begin{equation}\label{eqn-optimal-diffuse}
        \mathbf{\hat S}_{ij}^{(k)} = \frac{\omega_{ij}^{(k)}}{\sum_{l=1}^N \omega_{il}^{(k)}}, \quad \omega_{ij}^{(k)} = \left. \frac{\partial \delta(z^2)}{\partial z^2} \right |_{z^2 = \|\mathbf z_i^{(k)} - \mathbf z_j^{(k)}\|_2^2},
    \end{equation}
    yields a descent step on the energy, i.e., $E(\mathbf Z^{(k+1)}, k; \delta) \leq E(\mathbf Z^{(k)}, k-1; \delta)$ for any $k\geq 1$.
\end{theorem}
% \begin{hproof}
% We manage to establish an equivalance between i) unfolding the minimization of the energy and ii) finite-difference iteration of the diffusion process. Firstly, using Fenchel duality result, we could arrive at an upper bound for the energy $E(\mathbf Z, k; \delta)\leq \tilde E(\mathbf Z, k, \{\omega_{ij}^{(k)}\}; \tilde \delta)$, where $\tilde \delta$ is the concave conjugate of $\delta$ and the equality holds iff the variational parameter $\omega_{ij}^{(k)} = \frac{\partial \delta(z)}{\partial z}$. We can thereby convert minimizing $E(\mathbf Z, k; \delta)$ as minimizing the surrogate $\tilde E(\mathbf Z, k, \{\omega_{ij}^{(k)}\}; \tilde \delta)$ with the condition Eq.~\ref{eqn-optimal-diffuse}. The proof is concluded by showing that the one-step iteration by Eq.~\ref{eqn-diffuse-iter} is equivalent to one-step gradient descent of $\tilde E(\mathbf Z, k, \{\omega_{ij}^{(k)}\}; \tilde \delta)$. See Appendix~\ref{appx-proof} for complete reasoning.
% \end{hproof}

Theorem~\ref{thm-main} suggests the existence for the optimal diffusivity in the form of a function over the $l_2$ distance between states at the current step, i.e., $\|\mathbf z_i^{(k)} - \mathbf z_j^{(k)}\|_2$. The result enables us to unfold the implicit process and compute $\mathbf S^{(k)}$ in a feed-forward way from the initial states. We thus arrive at a new family of neural model architectures with layer-wise computation specified by:
{\small
\vspace{-5pt}
\begin{center}
\fcolorbox{black}{gray!10}{\parbox{0.97\linewidth}{
\vspace{-5pt}
\begin{equation}\label{eqn-model}
    \begin{split}
        & \mbox{Diffusivity Inference:} \quad \mathbf {\hat S}_{ij}^{(k)} = \frac{f(\|\mathbf z_i^{(k)} - \mathbf z_j^{(k)}\|_2^2)}{\sum_{l=1}^N f(\|\mathbf z_i^{(k)} - \mathbf z_l^{(k)}\|_2^2)}, \quad 1 \leq i, j \leq N,\\
        & \mbox{State Updating:} \quad \mathbf z_i^{(k+1)} = \underbrace{\left (1 - \tau \sum_{j=1}^N \mathbf {\hat S}_{ij}^{(k)} \right ) \mathbf z_i^{(k)} }_{\mbox{state conservation}} + \underbrace{\tau \sum_{j=1}^N \mathbf {\hat S}_{ij}^{(k)} \mathbf z_j^{(k)} }_{\mbox{state propagation}}, \quad 1 \leq i \leq N.
    \end{split}
\end{equation}
\vspace{-10pt}
}
}
\end{center} }
\vspace{-5pt}
\textbf{\emph{Remark.}} The choice of function $f$ in above formulation is not arbitrary, but needs to be a non-negative and decreasing function of $z^2$, so that the associated $\delta$ in Eq.~\ref{eqn-energy} is guaranteed to be non-decreasing and concave w.r.t. $z^2$. Critically though, there remains much room for us to properly design the specific $f$, %(that is essentially connected to the regularized energy with particular $\delta$, given by the result of Theorem~\ref{thm-main}), 
so as to provide adequate capacity and scalability. Also, in our model presented by Eq.~\ref{eqn-model} we only have one hyper-parameter $\tau$ in practice, noting that the weight $\lambda$ in the regularized energy is implicitly determined through $\tau$ by Theorem \ref{thm-main}, which reduces the cost of hyper-parameter searching.

% Based on the result of Eq.~\ref{eqn-optimal-diffuse}, the $f$ is derived from the first-order derivative of $\delta(z^2)$. 

\vspace{-4pt}

\section{Instantiations of \model}\label{sec:inst}
\vspace{-4pt}

% In this section, we go into detailed discussions w.r.t. model instantiations based on our theory.

\subsection{Model Instantiations}
\vspace{-4pt}

We next go into model instantiations based on the above theory, with two specified $f$'s as practical versions of our model. Due to space limits, we describe the key ideas concerning the model design in this subsection, and defer the details of model architectures to the self-contained Appendix~\ref{appx-alg}. %For other possible forms potentially working well, we provide some non-exhaustive discussions in Appendix~\ref{} and leave more thorough studies for future works.
First, because $\|\mathbf z_i - \mathbf z_j\|_2^2 = \|\mathbf z_i\|_2^2 + \|\mathbf z_j\|_2^2 - 2\mathbf z_i^\top \mathbf z_j$, we can convert $f(\|\mathbf z_i - \mathbf z_j\|_2^2)$ into the form $g(\mathbf z_i^\top \mathbf z_j)$ using a change of variables on the condition that $\|\mathbf z_i\|_2$ remains constant. And we add layer normalization to each layer to loosely enforce such a property in practice.

\textbf{Simple Diffusivity Model.} A straightforward design is to adopt the linear function $g(x) = 1+x$:
\begin{equation}\label{eqn-S1}
    \omega_{ij}^{(k)} = f(\|\tilde{\mathbf z}_i^{(k)} - \tilde{\mathbf z}_j^{(k)}\|_2^2) = 1 + \left (\frac{\mathbf z_i^{(k)}}{\|\mathbf z_i^{(k)}\|_2} \right)^\top 
    \left (\frac{\mathbf z_j^{(k)}}{\|\mathbf z_j^{(k)}\|_2} \right ),
\end{equation}
%where $\mathbf W_Q^{(k)}$, $\mathbf W_K^{(k)}$ are learnable weight matrices of the $k$-th layer.  
Assuming $\tilde{\mathbf z}_i^{(k)} = \frac{\mathbf z_i^{(k)}}{\|\mathbf z_i^{(k)}\|_2}$, $\tilde{\mathbf z}_j^{(k)} = \frac{\mathbf z_j^{(k)}}{\|\mathbf z_j^{(k)}\|_2}$ and $z = \|\tilde{\mathbf z}_i^{(k)} - \tilde{\mathbf z}_j^{(k)}\|_2$, Eq.~\ref{eqn-S1} can be written as $f(z^2) = 2 - \frac{1}{2}z^2$, which yields a non-negative result and is decreasing on the interval $[0, 2]$ in which $z^2$ lies.
%where we assume $\mathbf z_u^{(l)}$'s are layer-normalized, i.e., $\| \mathbf z_u^{(l)} \|_2 = 1$. This guarantees that once $W_Q^{(l)}$ and $W_K^{(l)}$ are orthonormal, then the gradient of $\delta(z)$ w.r.t. $z$ is strictly decreasing and consequently, there exists $\delta(z)$ satisfying the requirements. 
One scalability concern for the model Eq.~\ref{eqn-model} arises because of the need to compute pairwise diffusivity and propagation for each individual, inducing $\mathcal O(N^2)$ complexity. Remarkably, the simple diffusivity model allows a significant acceleration by noting that the state propagation can be re-arranged via %(assume $\overline{\mathbf z}_i = \tilde{\mathbf z}_i^{(l)} / \|\tilde{\mathbf z}_i^{(l)}\|_2$ for clean notation)
\begin{equation}
     \sum_{j=1}^N \mathbf S_{ij}^{(k)} \mathbf z_j^{(k)} = \sum_{j=1}^N \frac{1 + (\tilde{\mathbf z}_i^{(k)})^\top \tilde{\mathbf z}_j^{(k)}}{\sum_{l=1}^N\left (1 + (\tilde{\mathbf z}_i^{(k)})^\top \tilde{\mathbf z}_l^{(k)} \right ) } \mathbf z_j^{(k)} = \frac{\sum_{j=1}^N\mathbf z_j^{(k)} + \left (\sum_{j=1}^N \tilde{\mathbf z}_j^{(k)}\cdot  (\mathbf z_j^{(k)})^\top \right ) \cdot \tilde{\mathbf z}_i^{(k)} }{N + (\tilde{\mathbf z}_i^{(k)})^\top \sum_{l=1}^N \tilde{\mathbf z}_l^{(k)}}.
\end{equation}
The two summation terms above can be computed once and shared to every instance $i$, reducing the complexity in each iteration to $\mathcal O(N)$ (see Appendix~\ref{appx-alg} for how we achieve linear complexity w.r.t. $N$ in the matrix form for model implementation). We refer to this implementation as \model-s. %\modify{Similar strategies for complexity reduction is utilized by NodeFormer~\citep{wunodeformer}, a recently proposed scalable graph Transformer based on random feature map, while in contrast we resort to a more straightforward design that does not require any stochastic component or approximation.}

\textbf{Advanced Diffusivity Model.} The simple model facilitates efficiency/scalability, yet may sacrifice the  capacity for complex latent geometry. We thus propose an advanced version with $g(x) = \frac{1}{1+\exp(-x)}$:
\begin{equation}\label{eqn-S2}
    \omega_{ij}^{(k)} = f(\|\tilde{\mathbf z}_i^{(k)} - \tilde{\mathbf z}_j^{(k)}\|_2^2) = \frac{1}{1+\exp{\left(- (\mathbf z_i^{(k)} )^\top 
    (\mathbf z_j^{(k)} ) \right ) }},
\end{equation}
which corresponds with $f(z^2) = \frac{1}{1+e^{z^2/2-1}}$ guaranteeing monotonic decrease and non-negativity. We dub this version as \model-a. Appendix~\ref{appx-inst} further compares the two models (i.e., different $f$'s and $\delta$'s) through synthetic results. Real-world empirical comparisons are in Section~\ref{sec:exp}.

\begin{table}[t!]
\centering
\caption{A unified view for MLP, GCN and GAT from our energy-driven geometric diffusion framework regarding %. The table provides a head-to-head comparison with \model in terms of 
energy function forms, diffusivity specifications and algorithmic complexity. \label{tbl-existing}}
\vspace{-5pt}
\small
    \resizebox{0.99\textwidth}{!}{
\begin{tabular}{@{}c|c|c|c@{}}
\toprule
Models & Energy Function $E(\mathbf Z, k; \delta)$ & Diffusivity $\mathbf S^{(k)}$ & Complexity \\
\midrule
MLP & $\|\mathbf Z - \mathbf Z^{(k)}\|_2^2$ & $\mathbf S^{(k)}_{ij} = \left\{ 
    \begin{aligned}
         &1, \quad \mbox{if} \; i = j  \\
         &0, \quad otherwise
    \end{aligned}
    \right. $ & $\mathcal O(NKd^2)$ \\
\hline
GCN & $\sum_{(i,j)\in \mathcal E} \|\mathbf z_i - \mathbf z_j\|_2^2$ & $\mathbf S^{(k)}_{ij} = \left\{ 
    \begin{aligned}
         &\frac{1}{\sqrt{d_id_j}}, \quad \mbox{if} \; (i,j) \in \mathcal E  \\
         &0, \quad otherwise
    \end{aligned}
    \right. $ & $\mathcal O(|\mathcal E|Kd^2)$  \\
\hline
GAT & $\sum_{(i,j)\in \mathcal E} \delta(\|\mathbf z_i - \mathbf z_j\|_2^2)$ & $\mathbf S^{(k)}_{ij} = \left\{ 
    \begin{aligned}
         &\frac{f(\|\mathbf z_i^{(k)} - \mathbf z_j^{(k)}\|_2^2)}{\sum_{l: (i,l)\in \mathcal E} f(\|\mathbf z_i^{(k)} - \mathbf z_l^{(k)}\|_2^2)}, \quad \mbox{if} \; (i,j) \in \mathcal E  \\
         &0, \quad otherwise
    \end{aligned}
    \right.  $& $\mathcal O(|\mathcal E|Kd^2)$  \\
%\hline
%GLCN~\cite{jiang2019glcn} & $\sum_{i,j} \delta(\|\mathbf z_i - \mathbf z_j\|_2^2)$ & $\mathbf S^{(k)}_{ij} = \frac{f(\|\mathbf z_i^{(k)} - \mathbf z_j^{(k)}\|_2^2)}{\sum_{l=1}^N f(\|\mathbf z_i^{(k)} - \mathbf z_l^{(k)}\|_2^2)}, \quad 1\leq i,j \leq N$& $\mathcal O(N^2Kd)$  \\
\hline
\model & $\|\mathbf Z - \mathbf Z^{(k)}\|_2^2 + \lambda \sum_{i,j} \delta(\|\mathbf z_i - \mathbf z_j\|_2^2)$ & $\mathbf S^{(k)}_{ij} = \frac{f(\|\mathbf z_i^{(k)} - \mathbf z_j^{(k)}\|_2^2)}{\sum_{l=1}^N f(\|\mathbf z_i^{(k)} - \mathbf z_l^{(k)}\|_2^2)}, \quad 1\leq i,j \leq N$ & $\begin{aligned} &\mbox{\model-s}:\quad \mathcal O(NKd^2) \\ &\mbox{\model-a}:\quad \mathcal O(N^2Kd^2) \end{aligned}$ \\
\bottomrule
\end{tabular}}
\vspace{-10pt}
\end{table}

\subsection{Model Extensions and Further Discussion}
%\david{I feel like this section is combining things that don't fit so well together, especially under the category of `further discussion'.  I might recommend breaking up into two separate subsections, e.g., ``Model Extensions" which involves the layer-wise transformations and incorporating input graphs (two parts like you have now), and separately ``Connections with Existing Models" as the last subsection of Section 3.}
\textbf{Incorporating Layer-wise Transformations.} Eq.~\ref{eqn-model} does not use feature transformations for each layer. To further improve the representation capacity, we can add such transformations after the updating, i.e., $\mathbf z_i^{(k)} \leftarrow h^{(k)}(\mathbf z_i^{(k)})$ where $h^{(k)}$ can be a fully-connected layer (see Appendix~\ref{appx-alg} for details). In this way, each iteration of the diffusion yields a descent of a particular energy $E(\mathbf Z, k;\delta, h^{(k)}) = \|\mathbf Z - h^{(k)}(\mathbf Z^{(k)})\|_2^2 + \sum_{i,j} \delta(\|\mathbf z_i - \mathbf z_j\|_2^2)$ dependent on $k$. %The diffusion process becomes a composition of a series of steps, each of which optimizes a locally adaptive energy. 
%Although it breaks the layer-wise consistency via a globally shared energy function, 
The trainable transformation $h^{(k)}$ can be optimized w.r.t. the supervised loss to map the instance representations into a proper latent space. Our experiments find that the layer-wise transformation is not necessary for small datasets, but contributes to positive effects for datasets with larger sizes. Furthermore, one can consider non-linear activations in the layer-wise transformation $h^{(k)}$ though we empirically found that using a linear model already performs well. We also note that Theorem~\ref{thm-main} can be extended to hold even when incorporating such a non-linearity in each layer (see Appendix~\ref{appx-non} for detailed discussions).
%We call the model with layer-wise transformation as \model-trans.

% not mention oversmoothing, limitation is the single global function used for entire process

% better to have locally adaptive energy with local structure

% redefine Eqn. 5, add a transformation. Although it breaks the layer-wise consistency, the trainable matrix can be optimized towards the supervised. 

% for exp, try both

\textbf{Incorporating Input Graphs.} For the model presented so far, we do \emph{not} assume an input graph for the model formulation. For situations with observed structures as available input, we have $\mathcal G = (\mathcal V, \mathcal E)$ that can be leveraged as a geometric prior. We can thus modify the updating rule as:
\begin{equation}\label{eqn-diffuse-graph}
    \mathbf z_i^{(k+1)} = \left (1 - \frac{\tau}{2} \sum_{j=1}^N \left(\mathbf {\hat S}_{ij}^{(k)} + \tilde{\mathbf A}_{ij} \right ) \right ) \mathbf z_i^{(k)} + \frac{\tau}{2} \sum_{j=1}^N \left (\mathbf {\hat S}_{ij}^{(k)} + \tilde{\mathbf A}_{ij}  \right )\mathbf z_j^{(k)},
\end{equation}
where $\tilde{\mathbf A}_{ij} = \frac{1}{\sqrt{d_id_j}}$ if $(i,j)\in \mathcal E$ and 0 otherwise, and $d_i$ is instance $i$'s degree in $\mathcal G$. The diffusion iteration of Eq.~\ref{eqn-diffuse-graph} is essentially a descent step on a new energy additionally incorporating a graph-based penalty~\citep{graphkernel}, i.e., $\sum_{(i,j)\in \mathcal{E}} \|\mathbf z_i^{(k)} - \mathbf z_j^{(k)}\|_2^2$ (see Appendix~\ref{appx-inst} for details). %We call the model versions using input graphs as \model-s-g or \model-a-g for the sake of distinction.

\modify{\textbf{Scaling to Large Datasets.} Another advantage of \model over GNNs is the flexibility for mini-batch training. For datasets with prohibitive instance numbers that make it hard for full-batch training on a single GPU, the common practice for GNNs resorts to subgraph sampling~\cite{graphsaint} or graph clustering~\cite{clustergcn-kdd19} to reduce the overhead. These strategies, however, require extra time and tricky designs to preserve the internal structures. In contrast, thanks to the less reliance on input graphs, for training \model, we can naturally partition the dataset into random mini-batches and feed one mini-batch for one feed-forward and backward computation. The flexibility for mini-batch training also brings up convenience for parallel acceleration and federated learning. In fact, such a computational advantage is shared by Transformer-like models for cases where instances are inter-dependent, as demonstrated by the NodeFormer~\citep{wunodeformer}.}

% \subsection{Connection with Existing Models}
% \vspace{-5pt}
\textbf{Connection with Existing Models.} Our theory in fact gives rise to a general diffusion framework that unifies some existing models as special cases of ours. As a non-exhaustive summary and high-level comparison, Table~\ref{tbl-existing} presents the relationships with MLP, GCN and GAT (see more elaboration in Appendix~\ref{appx-connection}). %As increasing evidence suggesting the comparable empirical competitiveness of MLPs and propagation-based models (e.g., GNNs) in some practical tasks, the theoretical understandings for different inductive bias still remain unclear. One concurrent work \cite{yang2023graph} points out the fundamental connections among MLPs and GNNs with regard to their generalization effects, while our result in this paper serves as another aspect for understanding the implicit regularization effects w.r.t. embeddings by different propagation layers.
\modify{Specifically, the multi-layer perceptrons (MLP) can be seen as only considering the local consistency regularization in the energy function and only allowing non-zero diffusivity for self-loops. For graph convolution networks (GCN)~\citep{GCN-vallina}, it only regularizes the global consistency term constrained within 1-hop neighbors of the observed structure and simplifies $\delta$ as an identity function. The diffusivity of GCN induces information flows through observed edges. For graph attention networks (GAT)~\citep{GAT}, the energy function is similar to GCN's except that the non-linearity $\delta$ remains (as a certain specific form), and the diffusivity is computed by attention over edges.
In contrast with the above models, \model is derived from the general diffusion model that enables non-zero diffusivity values between arbitrary instance pairs.}

% \vspace{-3pt}
% \noindent $\circ$\; \emph{1) Generality}: \model is derived from the general diffusion model whose target regularized energy integrates two-fold consistency for a balance between individuals' and others' information.

% \vspace{-3pt}
% \noindent $\circ$\; \emph{2) Capacity:} \model enables non-zero diffusivity values between arbitrary instance pairs instead of constraining the information flows within self-loop edges (MLP) or neighboring nodes (GCN/GAT), which promotes the efficacy for capturing long dependency and unobserved interactions. 

% \vspace{-3pt}
% \noindent $\circ$\; \emph{3) Scalability:} \model-s only requires linear algorithmic complexity, the same as MLP models, and could be run efficiently and potentially scale to large datasets.}

\begin{table}[t]
\centering
\caption{Mean and standard deviation of testing accuracy on node classification (with five different random initializations). All the models are split into groups with a comparison of non-linearity (whether the model requires activation for layer-wise transformations), PDE-solver (whether the model requires PDE-solver) and Input-G (whether the propagation purely relies on input graphs).\label{tbl-bench}}
  \vspace{-5pt}
\small
    \resizebox{\textwidth}{!}{
\begin{tabular}{@{}c|c|cccccc@{}}
\toprule
\textbf{Type} & \textbf{Model} & \textbf{Non-linearity} & \textbf{PDE-solver} & \textbf{Input-G}  & \textbf{Cora} & \textbf{Citeseer} & \textbf{Pubmed}   \\ 
\midrule
\multirow{3}{*}{Basic models} & MLP & R   & -  & -   & 56.1 $\pm$ 1.6 & 56.7 $\pm$ 1.7 & 69.8 $\pm$ 1.5  \\
& LP & -  & -  & R   & 68.2 & 42.8 & 65.8 \\
& ManiReg & R    & -  & R  & 60.4 $\pm$ 0.8 & 67.2 $\pm$ 1.6 & 71.3 $\pm$ 1.4 \\
% SemiEmb & -  & -  & 63.1 $\pm$ 0.1 & 68.1 $\pm$ 0.1 & \\
\midrule
\multirow{8}{*}{Standard GNNs} 
& GCN & R   & -  & R   & 81.5 $\pm$ 1.3  & 71.9 $\pm$ 1.9 & 77.8 $\pm$ 2.9 \\
& GAT & R   & -  & R   & 83.0 $\pm$ 0.7 & 72.5 $\pm$ 0.7 & 79.0 $\pm$ 0.3 \\
& SGC & -  & -  & R   & 81.0 $\pm$ 0.0 & 71.9 $\pm$ 0.1 & 78.9 $\pm$ 0.0 \\
& GCN-$k$NN & R   & -  & -    & 72.2 $\pm$ 1.8 & 56.8 $\pm$ 3.2 & 74.5 $\pm$ 3.2 \\
& GAT-$k$NN & R   & -  & -    & 73.8 $\pm$ 1.7 & 56.4 $\pm$ 3.8 & 75.4 $\pm$ 1.3 \\
& Dense GAT & R   & -  & -    & 78.5 $\pm$ 2.5 & 66.4 $\pm$ 1.5 & 66.4 $\pm$ 1.5 \\
& LDS & R   & -  & -    & 83.9 $\pm$ 0.6 & \color{brown}\textbf{74.8 $\pm$ 0.3} & out-of-memory \\
& GLCN & R   & -  & -    & 83.1 $\pm$ 0.5 & 72.5 $\pm$ 0.9 & 78.4 $\pm$ 1.5 \\
% SGC-Pair- rm \\
\midrule
% SIGN-Linear & R   & R   & -  & 81.7 $\pm$ & 72.4 $\pm$ & 78.6 $\pm$\\
\multirow{6}{*}{Diffusion-based models} 
& GRAND-l & -  & R   & R   & 83.6 $\pm$ 1.0 & 73.4 $\pm$ 0.5 & 78.8 $\pm$ 1.7 \\
%& GRAND-nl & R   & R   & R   & 82.3 $\pm$ 1.6 & 70.9 $\pm$ 1.0 & 77.5 $\pm$ 1.8 \\
& GRAND & R   & R   & R    & 83.3 $\pm$ 1.3 & 74.1 $\pm$ 1.7 & 78.1 $\pm$ 2.1 \\
% & GRAND++-nl~\cite{GRAND++} & R   & R   & -  & 83.2 $\pm$ 0.2 & 74.2 $\pm$ 0.7 & 78.4 $\pm$ 0.7 \\
& GRAND++ & R   & R   & R    & 82.2 $\pm$ 1.1 & 73.3 $\pm$ 0.9 & 78.1 $\pm$ 0.9 \\
% \midrule
% \multirow{3}{*}{Diffusion-inspired models} 
& GDC & R   & -  & R   & 83.6 $\pm$ 0.2 & 73.4 $\pm$ 0.3 & 78.7 $\pm$ 0.4 \\
& GraphHeat & R   & -  & R  & 83.7 & 72.5 & 80.5 \\
& DGC-Euler & -  & -  & R   & 83.3 $\pm$ 0.0 & 73.3 $\pm$ 0.1 & 80.3 $\pm$ 0.1 \\
% GCDE & -  & R   & -  & 83.8 $\pm$ 0.5 & 72.5 $\pm$ 0.5 & 79.9 $\pm$ 0.3 \\
% GRAND++ \\
\midrule
\multirow{3}{*}{Graph Transformers} &
NodeFormer & -  & -  & -  & 83.4 $\pm$ 0.2 & 73.0 $\pm$ 0.3 & \color{brown}\textbf{81.5 $\pm$ 0.4} \\
& {\model}-s & -  & -  & -    & \color{purple}\textbf{85.9 $\pm$ 0.4} & 73.5 $\pm$ 0.3 & \color{purple}\textbf{81.8 $\pm$ 0.3} \\
& {\model}-a & -   & -  & -    & \color{brown}\textbf{84.1 $\pm$ 0.6} & \color{purple}\textbf{75.7 $\pm$ 0.3} & 80.5 $\pm$ 1.2 \\
\bottomrule
\end{tabular}}
\vspace{-15pt}
\end{table}

\vspace{-4pt}
\section{Experiments}\label{sec:exp}
\vspace{-4pt}

We apply \model to various tasks for evaluation: 1) graph-based node classification where an input graph is given as observation; 2) image and text classification without input graphs; 3) spatial temporal dynamics prediction. In each case, we compare a different set of competing models closely associated with \model and speciﬁcally designed for the particular task. Unless otherwise stated, for datasets where input graphs are available, we incorporate them for feature propagation as is defined by Eq.~\ref{eqn-diffuse-graph}. Due to space limit, we defer details of datasets to Appendix~\ref{appx-dataset} and the implementations to Appendix~\ref{appx-implementation}. Also, we provide additional empirical results in Appendix~\ref{appx-results}.

\vspace{-4pt}
\subsection{Semi-supervised Node Classification Benchmarks}
\vspace{-4pt}

We test \model on three citation networks \texttt{Cora}, \texttt{Citeseer} and \texttt{Pubmed}. Table~\ref{tbl-bench} reports the testing accuracy. We compare with several sets of baselines linked with our model from different aspects. 1) Basic models: \emph{MLP} and two classical graph-based SSL models Label Propagation (\emph{LP})~\citep{lp-icml2003} and \emph{ManiReg}~\citep{belkin2006manireg}. 2) GNN models: \emph{SGC}~\citep{SGC-icml19}, \emph{GCN}~\citep{GCN-vallina}, \emph{GAT}~\citep{GAT}, their variants \emph{GCN-$k$NN}, \emph{GAT-$k$NN} (operating on $k$NN graphs constructed from input features) and \emph{Dense GAT} (with a densely connected graph replacing the input one), and two strong structure learning models LDS~\citep{LDS-icml19} and GLCN~\citep{jiang2019glcn}. 3) PDE graph models: the SOTA models \emph{GRAND}~\citep{grand} (with its linear variant GRAND-l) and \emph{GRAND++}~\citep{GRAND++}. %(additionally using a source term). 
4) Diffusion-inspired GNN models: \emph{GDC}~\citep{klicpera2019diffusion}, \emph{GraphHeat}~\citep{xu2020heat} and a recent work \emph{DGC-Euler}~\citep{wang2021dissecting}. Table~\ref{tbl-bench} shows that \model achieves the best results on three datasets with significant improvements. Also, we notice that the simple diffusivity model \model-s significantly exceeds the counterparts without non-linearity (SGC, GRAND-l and DGC-Euler) and even comes to the first on \texttt{Cora} and \texttt{Pubmed}. These results suggest that \model can serve as a very competitive encoder backbone for node-level prediction that learns inter-instance interactions for generating informative representations and boosting downstream performance.

%These datasets are small-scale citation networks (with 2K$\sim$20K nodes) and the goal is to classify the topics of documents (instances) based on input features of each instance (bag-of-words representation of documents) and graph structure (citation links). 
% and follow the standard semi-supervised learning setting in \cite{GCN-vallina} for data splitting. %randomly choosing 20 instances per class for training, 500/1000 instances for validation/testing. 

\begin{figure}[t]
\begin{minipage}{0.33\linewidth}
    \centering
	\captionof{table}{Testing ROC-AUC for \texttt{Proteins} and Accuracy for \texttt{Pokec} on large-scale node classification datasets. $*$ denotes using mini-batch training. \label{tbl-largebench}}
	\vspace{-6pt}
    \resizebox{0.98\textwidth}{!}{
\begin{tabular}{@{}c|cc@{}}
\toprule
\textbf{Models} & \textbf{Proteins} & \textbf{Pokec}   \\ 
\midrule
MLP & 72.41 $\pm$ 0.10 & 60.15 $\pm$ 0.03 \\
LP & 74.73 & 52.73\\
SGC &49.03 $\pm$ 0.93 & 52.03 $\pm$ 0.84 \\
% SGC-batch &40.45 ± 2.24 &\\
GCN & 74.22 $\pm$ 0.49$^*$ & 62.31 $\pm$ 1.13$^*$ \\
GAT & 75.11 $\pm$ 1.45$^*$ & 65.57 ± 0.34$^*$  \\
%\modify{GT} & \modify{77.21 $\pm$ 0.42} & \modify{67.18 $\pm$ 0.24} \\ 
NodeFormer & \color{brown}\textbf{77.45 $\pm$ 1.15}$^*$ & \color{brown}\textbf{68.32 $\pm$ 0.45}$^*$ \\
\midrule
{\model}-s & \color{purple}\textbf{79.49 $\pm$ 0.44}$^*$ & \color{purple}\textbf{69.24 $\pm$ 0.76}$^*$ \\
% {\model}-s w/o graph & 73.46 $\pm$ 0.35 & 62.12 $\pm$ 0.04 \\
\bottomrule
\end{tabular}}
    \end{minipage}
    \hspace{5pt}
    \begin{minipage}{0.66\linewidth}
   \subfigure[]{
    \begin{minipage}[t]{0.48\linewidth}
    \centering
    \includegraphics[width=\linewidth]{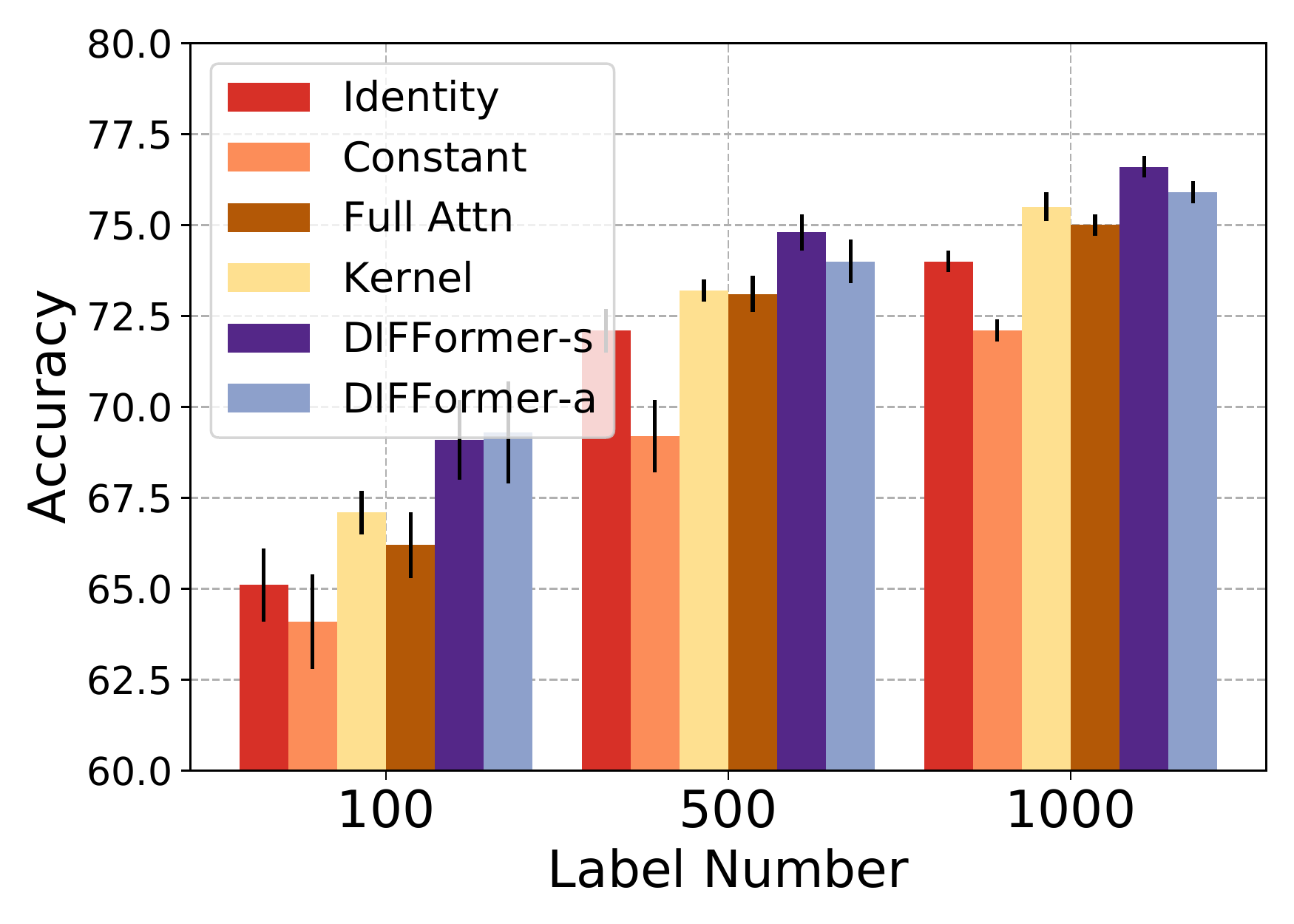}
    \label{fig-abl}
    \vspace{-10pt}
    \end{minipage}
    }
    \subfigure[]{
    \begin{minipage}[t]{0.48\linewidth}
    \centering
    \includegraphics[width=\linewidth]{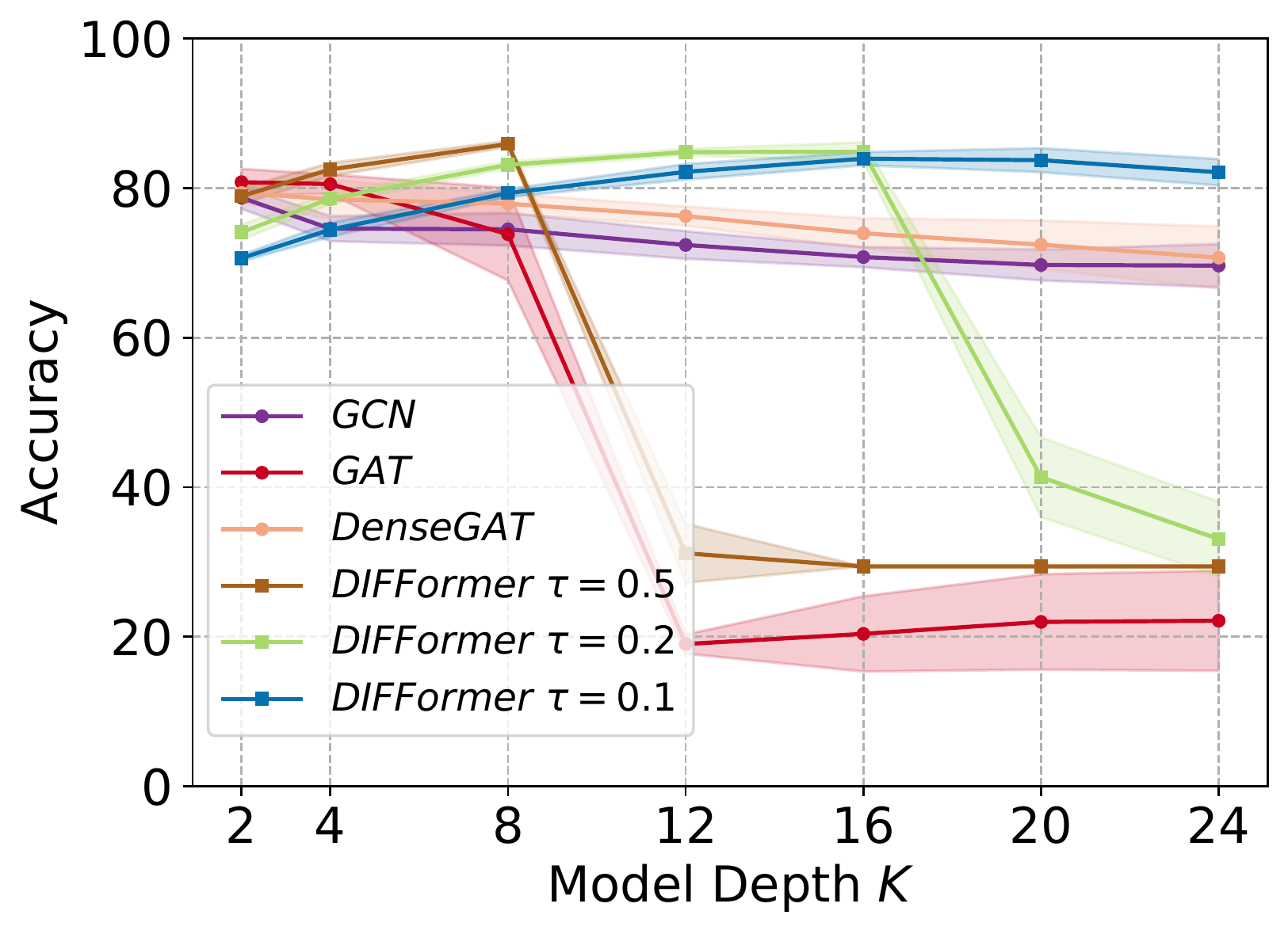}
    \label{fig-hyper}
    \vspace{-10pt}
    \end{minipage}
    }
    \vspace{-10pt}
    \caption{(a) Ablation studies w.r.t. different diffusivity function forms on \texttt{CIFAR}. (b) Impact of $K$ and $\tau$ on \texttt{Cora}.}
    \end{minipage}
    \vspace{-10pt}
\end{figure}

\begin{table}[tb!]
		\centering
		\small
		\caption{Testing accuracy for image (\texttt{CIFAR} and \texttt{STL}) and text (\texttt{20News}) classification.}
  \vspace{-5pt}
		\label{tab:image classification}
		\resizebox{0.95\textwidth}{!}{
		\begin{tabular}{c|c|ccccccc|cc}
		\midrule
		\multicolumn{2}{c|}{\textbf{Dataset}}  & MLP & LP & ManiReg  & GCN-$k$NN & GAT-$k$NN & DenseGAT & GLCN  & \model-s & \model-a  \\
		\midrule
		\multirow{3}{*}{\textbf{CIFAR}} 
		& 100 labels &65.9 ± 1.3&66.2&67.0 ± 1.9&66.7 ± 1.5&66.0 ± 2.1 & out-of-memory &66.6 ± 1.4&\color{brown}\textbf{69.1 ± 1.1}&\color{purple}\textbf{69.3 ± 1.4} \\
		& 500 labels &73.2 ± 0.4&70.6&72.6 ± 1.2&72.9 ± 0.4&72.4 ± 0.5 & out-of-memory &72.8 ± 0.5&\color{purple}\textbf{74.8 ± 0.5}&\color{brown}\textbf{74.0 ± 0.6} \\
		& 1000 labels &75.4 ± 0.6&71.9&74.3 ± 0.4&74.7 ± 0.5&74.1 ± 0.5 & out-of-memory &74.7 ± 0.3&\color{purple}\textbf{76.6 ± 0.3}&\color{brown}\textbf{75.9 ± 0.3} \\
		\midrule
		\multirow{3}{*}{\textbf{STL}} 
		& 100 labels &66.2 ± 1.4&65.2&66.5 ± 1.9&\color{brown}\textbf{66.9 ± 0.5}&66.5 ± 0.8 & out-of-memory &66.4 ± 0.8&\color{purple}\textbf{67.8 ± 1.1}&66.8 ± 1.1 \\
		& 500 labels &\color{brown}\textbf{73.0 ± 0.8}&71.8&72.5 ± 0.5&72.1 ± 0.8&72.0 ± 0.8 & out-of-memory &72.4 ± 1.3&\color{purple}\textbf{73.7 ± 0.6}&72.9 ± 0.7 \\
		& 1000 labels &75.0 ± 0.8&72.7&74.2 ± 0.5&73.7 ± 0.4&73.9 ± 0.6 & out-of-memory &74.3 ± 0.7&\color{purple}\textbf{76.4 ± 0.5}&\color{brown}\textbf{75.3 ± 0.6} \\
		\midrule
		\multirow{3}{*}{\textbf{20News}} 
		& 1000 labels &54.1 ± 0.9&55.9&56.3 ± 1.2&56.1 ± 0.6&55.2 ± 0.8 & 54.6 ± 0.2 &56.2 ± 0.8&\color{brown}\textbf{57.7 ± 0.3}&\color{purple}\textbf{57.9 ± 0.7} \\
		& 2000  labels &57.8 ± 0.9&57.6&60.0 ± 0.8&60.6 ± 1.3&59.1 ± 2.2 & 59.3 ± 1.4 &60.2 ± 0.7&\color{brown}\textbf{61.2 ± 0.6}&\color{purple}\textbf{61.3 ± 1.0}  \\
		& 4000 labels &62.4 ± 0.6&59.5&63.6 ± 0.7&64.3 ± 1.0&62.9 ± 0.7 & 62.4 ± 1.0 &64.1 ± 0.8&\color{purple}\textbf{65.9 ± 0.8}&\color{brown}\textbf{64.8 ± 1.0}  \\
		\midrule
	\end{tabular}				
		}
\vspace{-15pt}
	\end{table}
 
\vspace{-4pt}
\subsection{Large-Scale Node Classification Graphs}
\vspace{-4pt}

We also consider two large-scale graph datasets \texttt{ogbn-Proteins}, a multi-task protein-protein interaction network, and \texttt{Pokec}, a social network. %For \texttt{Proteins}, we follow the splitting and evaluation protocol by \cite{ogb-nips20}; for \texttt{Pokec} we randomly choose 10\%/10\%/80\% instances for training/validation/test and use Accuracy for evaluation following \cite{newbench}.
Table~\ref{tbl-largebench} presents the results. Due to the dataset size (0.13M/1.63M nodes for two graphs) and scalability issues that many of the competitors in Table~\ref{tbl-bench} as well as \model-a would potentially experience, we only compare \model-s with standard GNNs. In particular, we found GCN/GAT/\model-s are still hard for full-graph training on a single V100 GPU with 16GM memory. We thus consider mini-batch training with batch size 10K/100K for \texttt{Proteins}/\texttt{Pokec}. %\modify{Besides, we compare with a recently proposed Graph Transformer (GT) model~\cite{graphformer-neurips21} as a strong competitor that considers all-pair propagation.}
We found that \model outperforms common GNNs by a large margin, which suggests its desired efficacy on large datasets. As mentioned previously, we prioritize the efficacy of \model as a general encoder backbone for solving node-level prediction tasks on large graphs. While there are quite a few practical tricks shown to be effective for training GNNs for this purpose, e.g., hop-wise attention~\citep{adadiffusion-2022} or various label re-use strategies, these efforts are largely orthogonal to our contribution here and can be applied to most any model to further boost performance. For further investigation, we supplement more results using different mini-batch sizes for training and study its impact on testing performance in Appendix~\ref{appx-batch}. Furthermore, we compare the training time and memory costs in Appendix~\ref{appx-time} where we found that \model-s is about 6 times faster than GAT and 39 times faster than DenseGAT on \texttt{Pokec}, which suggests superior scalability and efficiency of \model-s on large graphs.

% which implies that there still exists potential large room for further performance gain with such tricks incorporated, to stay focused in the present work, we leave this orthogonal line of exploration for future works.

%Also we compare with \model w/o graph that removes the input graph. The performance exhibits a degradation yet still performs better than MLP, which verifies the effectiveness of our learned latent interaction.

\vspace{-4pt}

\subsection{Image and Text Classification with Low Label Rates}
\vspace{-4pt}

We next conduct experiments on \texttt{CIFAR-10}, \texttt{STL-10} and \texttt{20News-Group} datasets to test \model for standard classification tasks with limited label rates.
For \texttt{20News} provided by \cite{pedregosa2011scikit}, we take 10 topics and use words with TF-IDF more than 5 as features. %For evaluation purpose, we randomly select 100/200/400 instances per class as training set, 2000 instances for validation and the remaining xxx instances for testing.
For \texttt{CIFAR} and \texttt{STL}, two public image datasets, we first use the self-supervised approach SimCLR~\citep{chen2020simple} (that does not use labels for training) to train a ResNet-18 for extracting the feature maps as input features of instances.
%Then we randomly select 10/50/100 instances per class as training set, 1000 instances for validation and the remaining xxx instances for testing. 
These datasets contain no graph structure, so we use $k$NN to construct a graph over input features for GNN competitors and do \emph{not} use input graphs for \model.

Table~\ref{tab:image classification} reports the testing accuracy of \model and competitors including MLP, ManiReg, GCN-$k$NN, GAT-$k$NN, DenseGAT and GLCN. 
Two \model models perform much better than MLP in nearly all cases, suggesting the effectiveness of learning the inter-dependence over instances. Besides, \model yields large improvements over GCN and GAT which are in some sense limited by the handcrafted graph that leads to sub-optimal propagation. Moreover, \model significantly outperforms GLCN, a strong baseline that learns new (static) graph structures, which demonstrates the superiority of our evolving diffusivity that can adapt to different layers.

\vspace{-4pt}
\subsection{Spatial-Temporal Dynamics Prediction}
\vspace{-4pt}

We consider three spatial-temporal datasets with details in Appendix~\ref{appx-dataset}. Each dataset consists of a series of graph snapshots where nodes are treated as instances and each of them has a integer label (e.g., reported cases for \texttt{Chickenpox} or \texttt{Covid}). The task is to predict the labels of one snapshot based on the previous ones. %For each dataset, to make it more challenging and close to the typical real-world low-data regime, we split the snapshots into training, validation and test sets by 2:2:6. Since the testing snapshots contain new instances (or unseen graph structures), such a scenario requires inductive learning.
Table~\ref{tab:spatial_temporal_results} compares testing MSE of four \model variants (here \model-s w/o g denotes the model \model-s without using input graphs) against baselines. We can see that two \model variants without input graphs even outperform the counterparts using input structures in four out of six cases. This implies that our diffusivity estimation module could learn useful structures for informed prediction, and the input structure might not always contribute to positive effect. In fact, for temporal dynamics, the underlying relations that truly influence the trajectory evolution can be much complex and the observed relations could be unreliable with missing or noisy links, in which case GNN models relying on input graphs may perform undesirably. Compared to the competitors, our models rank the first with significant improvements. %In particular, for \texttt{Covid} \model-s achieves 8.1\% lower MSE than the best competitor GAT-$k$NN and for \texttt{WikiMath} \model-a decreases the MSE by 12.8\% than the runner-up Dense GAT, which demonstrates the superiority of \model for temporal dynamics prediction. 

\begin{table}[t!]
\centering
\caption{Mean and standard deviation of MSE on spatial-temporal prediction datasets.}
  \vspace{-5pt}
\label{tab:spatial_temporal_results}
\small
    \resizebox{1.0\textwidth}{!}{
\begin{tabular}{c|cccccc|cccc}
\toprule
%\multirow{2}{*}{Data/Method}& \multicolumn{4}{c|}{{\model} Variants}&\multicolumn{6}{c}{Baselines}\\
\textbf{Dataset}& MLP&GCN&  GAT&Dense GAT&GAT-$k$NN&GCN-$k$NN&{\model}-s& {\model}-a& {\model}-s w/o g& {\model}-a w/o g \\ 
\midrule
\multirow{2}{*}{\textbf{Chickenpox}}& 0.924&0.923&0.924&0.935&0.926&0.936&\color{purple}\textbf{0.914} & \color{brown}\textbf{0.915} & 0.916 & 0.916
\\
& ($\pm$0.001)&($\pm$0.001)&($\pm$0.002)&($\pm$0.005)&($\pm$0.004)&($\pm$0.004)&\color{purple}\textbf{(0.006)} & \color{brown}\textbf{(0.008)} & (0.006) & (0.006)\\
\midrule
\multirow{2}{*}{\textbf{Covid}}& 0.956&1.080&1.052&1.524&0.861&1.475& 0.779 & \color{brown}\textbf{0.757} & 0.779 & \color{purple}\textbf{0.741}\\
& ($\pm$0.198)&($\pm$0.162)&($\pm$0.336)&($\pm$0.319)&($\pm$0.123)&($\pm$0.560)& (0.037) & \color{brown}\textbf{(0.048)} & (0.028) & \color{purple}\textbf{(0.052)}\\
\midrule
\multirow{2}{*}{\textbf{WikiMath}}& 1.073&1.292&1.339&0.826&0.882&1.023&0.731 & 0.763 & \color{brown}\textbf{0.727} & \color{purple}\textbf{0.716}\\
& ($\pm$0.042)&($\pm$0.125)&($\pm$0.073)&($\pm$0.070)&($\pm$0.015)&($\pm$0.058)&(0.007) & (0.020) & \color{brown}\textbf{(0.025)} & \color{purple}\textbf{(0.030)}\\
\bottomrule
\end{tabular}}
\vspace{-15pt}
\end{table}

\vspace{-4pt}

\subsection{Further Results and Discussions}
\vspace{-4pt}

\textbf{How do different diffusivity functions perform?} Figure~\ref{fig-abl} compares \model with four variants using other diffusivity functions that have no essential connection with energy minimization: 1) \emph{Identity} sets $\mathbf S^{(k)}$ as a fixed identity matrix; 2) \emph{Constant} fixes $\mathbf S^{(k)}$ as all-one constant matrix; 3) \emph{Full Attn} parameterizes $\mathbf S^{(k)}$ by attention networks~\citep{transformer}; 4) \emph{Kernel} adopts Gaussian kernel for computing $\mathbf S^{(k)}$. More results on other datasets are in Appendix~\ref{appx-abl} and they consistently show that our adopted diffusivity forms produce superior performance, which verifies the effectiveness of our diffusivity designs derived from minimization of a principled energy.

\textbf{How do model depth and step size impact the performance?} We discuss the influence of model depth $K$ and step size $\tau$ on \texttt{Cora} in Fig.~\ref{fig-hyper}. More results are deferred to Appendix~\ref{appx-hyper}. The curves indicate that the step size impacts the model performance as well as the variation trend w.r.t. model depth. When using larger step size (e.g., $\tau=0.5$), the model can yield superior performance with shallow layers; when using smaller step size (e.g., $\tau=0.1$), the model needs to stack deep layers for performing competitively. The reason could be that larger $\tau$ contributes to more concentration on global information from other instances in each diffusion iteration, which brings up more beneficial information with shallow layers yet could lead to instability for deep layers. And, using small $\tau$ can enhance the model's insensitivity to deep model depth, which often leads to the performance degradation of other baselines like GCN and GAT.

%\textbf{What is learned by the representation?}
%We next try to shed some insights on how our diffusion models help to learn effective representations that facilitate downstream prediction, via visualizing node representations and layer-wise diffusivity strength in Appendix~\ref{appx-vis}. We found that the diffusivity strength estimates tend to increase the connectivity among nodes with different classes, and thus the updated node embeddings can absorb information from different communities for informative prediction. On the other hand, the produced representations exhibit smaller intra-class distance and larger inter-class distance, making them easier to be distinguished for classification. Comparing \model-s and \model-a on the temporal datasets, we found that \model-s produces more concentrated large weights while \model-a tends to have large diffusivity spreading out more and learn more complex structures. See more discussions in Appendix~\ref{appx-vis}.

%\textbf{Visualization.} Figure~\ref{chickenpox_main_text_visualization} visualizes the diffusivity estimates on \texttt{Chickenpox}. We conclude that large diffusion strength usually exists between nodes with similar ground-truth labels. \model-s has more concentrated large weights while \model-a tends to have large diffusivity spreading out more. \model-a indeed learns more complex underlying structures than \model-s due to its better capacity for diffusivity modeling. More visualization results are provided in Appendix~\ref{appx-dynamics}.

\vspace{-4pt}

\section{Conclusions}
\vspace{-4pt}

\looseness=-1 This paper proposes an energy-driven geometric diffusion model with latent diffusivity function for data representations. The model encodes all the instances as a whole into evolving states aimed at minimizing a principled energy as implicit regularization. We further design two practical implementations with enough scalability and capacity for learning complex interactions over the underlying data geometry. Extensive experiments demonstrate the effectiveness and superiority of the model.

\bibliographystyle{iclr2023_conference}
\bibliography{ref}

%%%%%%%%%%%%%%%%%%%%%%%%%%%%%%%%%%%%%%%%%%%%%%%%%%%%%%%%%%%%

\appendix

\section{Further Related Works and Connection with Ours}

We discuss more related works that associate with ours from different aspects to properly position this paper with different areas. Based on this, we further shed more lights on the technical contributions of our work and its potential impact in different communities.

\subsection{Neural Diffusion Models} 

\textbf{PDE-based Learning.} The diffusion-based learning has gained increasing research interests, as the continuous dynamics can serve as an inductive bias incorporated with prior knowledge of the tasks at hand. One category directly solves a continuous process of differential equations~\citep{lagaris1998artificial,chen2018neuralode}, e.g.,~\cite{grand} and its follow-ups~\citep{beltrami,GRAND++} reveal the analogy between the discretization of diffusion process and GNNs' feedforward rules, and devise new (continuous) models on graphs whose training requires PDE-solving tools. A concurrent work \citep{wang2023acmp} explores how to derive neural networks from gradient flows and proposes Allen-Cahn Message Passing that combines attractive and repulsive effects.

\textbf{PDE-inspired Learning.} Another category is PDE-inspired learning using the diffusion perspective as principled guidelines on top of which new (discrete) neural network-based approaches are designed for node classification~\citep{atwood2016diffusion,klicpera2019diffusion,xu2020heat}, graph comparison~\citep{NetLSD}, and geometric knowledge distillation~\citep{yang2022geometric}.
Our work leans on PDE-inspired learning and introduce a new diffusion model that is implicitly defined as minimizing a regularized energy. Our theory also reveals the underlying equivalence between the numerical iterations of the diffusion process and unfolding the minimization dynamics of a corresponding energy. \modify{The new results can serve to extend the class of diffusion process and illuminate its fundamental connection with energy optimization systems. More importantly, as we will in the maintext, such a new perspective brings up a unifying view for existing models like MLP and GNNs.}

\subsection{Theoretical Interpretations of GNNs} 

\textbf{Optimization-induced Models.} As GNNs have proven to be powerful encoders for modeling data with geometry, there arise quite a few studies that aim at understanding the efficacy of GNNs with respect to node representations yielded by the propagation. For instance, \cite{RWLS-icml21} founds that the layer-wise updates of common GNNs can be derived from a classical iterative algorithm solving an optimization problem defined over an energy. Furthermore, \cite{graphdenoise} points out the relationship between GNNs' message passing rules and the graph denoising problem. These works interpret GNNs as optimization-induced models. \modify{In our model, the diffusion process is implicitly defined by an optimization problem and our theory further reveals the fundamental equivalence between diffusion iterations and energy optimization dynamics, which can bridge two schools of thinking and facilitate the understandings from both perspectives.}

\textbf{Generalization of GNNs.} Another line of works turn attention to GNNs' generalization capability. As increasing evidence shows that GNNs are sensitive to distribution shifts~\citep{eerm}, it becomes urgent to investigate and figure out how to improve GNNs' generalization. A concurrent work~\citep{yang2023graph} identifies an intriguing phenomenon that GNNs without propagation (i.e., an MLP architecture) during training can still yield competitive or even close performance at inference time to standard GNNs (using feature propagation at both training and inference stage). This suggests that the propagation design contributes to better generalization instead of the model capacity, and the former plays as a critical factor that makes GNNs a more powerful encoder than MLP. This result can further interprets the success of \model: the new propagation of \model that aggregates other nodes' embeddings can enhance the generalization at inference time on testing data.

\subsection{Transformers} 

\textbf{General Transformers.} The key design of Transformers lies in the (all-pair) attention mechanism that captures the interactions among tokens of an input sequence. The attention design is originally motivated by the inter-dependence among words of a sentence in NLP tasks~\citep{transformer}. Furthermore, a surge of follow-up works extend the attention mechanisms for capturing the inter-dependence among patches of an image~\citep{vit}, atoms of a molecule~\citep{graphtransformer-2020}, etc. 
From different aspects, our work explores a Transformer model that targets learning inter-dependence among instances in a dataset.

\textbf{Efficient Transformers.} When the length of inputs goes large,  the quadratic complexity of all-pair attention becomes the computational bottleneck of Transformers. To address the long sequences in NLP tasks, quite a few recent works propose various strategies via, e.g., Softmax kernel~\citep{performer-iclr21}, local-global attention~\citep{bigbird}, and low-rank approximation~\citep{linformer}. In our problem setting, if one treats the whole dataset as an input `sequence' for the Transformer model, then the length of such a sequence could go to even million-level, which poses demanding scalability challenges. Our model \model-s achieves strictly linear complexity w.r.t. $N$ through introducing a new attention function that keeps a simple form yet accommodates the influence among arbitrary node pairs.

\textbf{Transformers on Graphs.} There is also increasing research interest on building powerful Transformers for graph-structured data, due to the inherent good ability of Transformer models for capturing long-range dependence and potential node-pair interactions that are unobserved in input structures. Recent works, e.g.,~\cite{graphtransformer-2020, graphbert-2020, graphformer-neurips21} explore various effective architectures for graph-level prediction (GP) tasks, e.g., molecular property prediction, whereby each graph itself generally has a label to predict and a dataset contains many graph instances. Our target problem, as mentioned in Section~\ref{sec:related}, can be seen an embodiment of node-level prediction (NP) studied in graph learning community. These two problems are typically tackled separately in the literature~\citep{ogb-nips20} with disparate technical considerations. This is because input instances are inter-dependent in NP (due to the instance interactions involved in the data-generating process), while in GP tasks the instances can be treated as IID samples. For node-level prediction, a recent work \citep{wunodeformer} proposes to leverage the random feature map to construct kernelized Gumbel-Softmax-based message passing that can achieve all-pair feature propagation with linear complexity. %Concurrently at the same conference, \citep{chen2023nagphormer} proposes to transform neighborhood features into tokens fed into Transformers to enable efficient training on large graphs

\subsection{Graph Neural Networks} 

\textbf{Scalable Graph Neural Networks.} Due to the large graph sizes that could be frequently encountered in graph-based predictive tasks (particularly NP), many recent works focus on designing scalable GNNs. From different technical aspects, these works can be generally grouped into graph sampling-based partition~\citep{graphsaint,clustergcn-kdd19}, approximation-based propagation~\citep{scale} and simplified architectures~\citep{SGC-icml19,zhang2022scalegcn}. In contrast with them, \model can scale to medium-sized graphs with the linear complexity, and to large-scale graphs with mini-batch training that is simple, stable and also flexible for balancing efficiency and precision with proper mini-batch sizes. The later is allowed by the all-pair message passing schemes that do not rely on input structures (the input graphs play an auxiliary role in the model). 

\textbf{Graph Structure Learning.} Learning latent graph structures is a well-established research problem in graph learning community. Motivated by the observation that input graphs can often be unreliable or unavailable, based on which the message passing of GNNs could yield undesired results, graph structure learning aims at learning adaptive structures that can boost GNNs towards better representations and downstream prediction~\citep{LDS-icml19,IDGL-neurips20,jiang2019glcn,fatemi2021slaps,lao2022variational}. Since graph structure learning requires estimation for $N^2$ potential edges that connect all node pairs in a dataset, the time and space complexity of most existing models are at least quadratic w.r.t. node numbers, which is prohibitive for large graphs. The recent work~\citep{wunodeformer} proposes a Transformer model with linear complexity that can scale graph structure learning to graphs with millions of nodes. The proposed model \model-s can be another powerful approach for learning latent structures in large-scale systems, without sacrificing the accommodation of all-pair interactions at each layer.

\section{Proof for Theorem~\ref{thm-main}}\label{appx-proof}

First of all, we can convert the minimization of Eq.~\ref{eqn-energy} into a minimization of its variational upper bound, shown in the following proposition.
\begin{proposition}\label{prop1}
    The energy function $E(\mathbf Z, k; \delta)$ is upper bounded by
    \begin{equation}\label{eqn-prop1-bound}
        \tilde E(\mathbf Z, k; \{\omega_{ij}\}, \tilde \delta) = \|\mathbf Z - \mathbf Z^{(k)}\|_{\mathcal F}^2 + \lambda \left[\sum_{i,j} \omega_{ij} \|\mathbf z_i - \mathbf z_j\|_2^2 - \tilde \delta(\omega_{ij})\right ],
    \end{equation}
    where $\tilde \delta$ is the concave conjugate of $\delta$, and the equality holds if and only if the variational parameters satisfy
    \begin{equation}\label{eqn-prop1-cond}
        \omega_{ij} = \left. \frac{\partial \delta(z^2)}{\partial z^2} \right |_{z = \|\mathbf z_i - \mathbf z_j\|_2}.
    \end{equation}
\end{proposition}
\begin{proof}
    The proof of the proposition follows the principles of convex analysis and Fenchel duality~\citep{convex-1970}. For any concave and non-decreasing function $\rho: \mathbb R^+\rightarrow \mathbb R$, one can express it as the variational decomposition
    \begin{equation}\label{eqn-proof-convex}
        \rho(z^2) = \min_{\omega \geq 0} [\omega z^2 - \tilde \rho(\omega)] \geq \omega z^2 - \tilde \rho(\omega),
    \end{equation}
    where $\omega$ is a variational parameter and $\tilde \rho$ is the concave conjugate of $\rho$.
    Eq.~\ref{eqn-proof-convex} essentially defines $\rho(z^2)$ as the minimal envelope of a series of quadratic bounds $\omega z^2 - \tilde \rho(\omega)$ defined by a different values of $\omega\geq 0$ and the upper bound is given for a fixed $\omega$ when removing the minimization operator. Based on this, we obtain the result of Eq.~\ref{eqn-prop1-bound}. In terms of the sufficient and necessary condition for equality, we note that for any optimal $\omega^*$ we have
    \begin{equation}
        \omega^* z^2 - \tilde \rho(\omega^*) = \rho(z^2),
    \end{equation}
    which is tangent to $\rho$ at $z^2$ and $\omega^* = \frac{\partial \delta(z^2)}{\partial z^2}$. We thus obtain the result of Eq.~\ref{eqn-prop1-cond}.
\end{proof}

We next continue to prove the main result of Theorem~\ref{thm-main}. According to Proposition~\ref{prop1}, we can minimize the upper bound surrogate Eq.~\ref{eqn-prop1-bound} and it becomes equivalent to a minimization of the original energy on condition that the variational parameters are given by Eq.~\ref{eqn-prop1-cond}. Then with a one-step gradient decent of Eq.~\ref{eqn-prop1-bound}, the instance states could be updated via (assuming $l$ as the index of steps and $\alpha$ as step size)
\begin{equation}\label{eqn-iter-unfolding}
\begin{split}
    \mathbf Z^{(l+1)} &= \mathbf Z^{(l)} - \alpha \left.\frac{\partial \tilde E(\mathbf Z; \delta)}{ \partial \mathbf Z}\right|_{\mathbf Z = \mathbf Z^{(l)}} \\
    & = \mathbf Z^{(l)} - \alpha\left(  \lambda (\mathbf D^{(l)} - \bm \Omega^{(l)}) \mathbf Z^{(l)}  + \mathbf Z^{(l)} - \mathbf Z^{(l)}  \right) \\
    & = \mathbf Z^{(l)} - \alpha'(\mathbf D^{(l)} - \bm \Omega^{(l)}) \mathbf Z^{(l)}
    %& = (1 - \alpha)\mathbf Z^{(l)} + \alpha(\bm \Omega^{(l)} - \mathbf D^{(l)}) \mathbf Z^{(l)} + \alpha g(\mathbf Z^{(0)}, W_F),
\end{split}
\end{equation}
where $\bm \Omega^{(l)} = \{\omega_{ij}^{(l)}\}_{N\times N}$, $\mathbf D^{(l)}$ denotes the diagonal degree matrix associated with $\bm \Omega^{(l)}$ and we introduce $\alpha'=\alpha\lambda$ to combine two parameters as one. Common practice to accelerate convergence adopts a positive definite preconditioner term, e.g., $(\mathbf D^{(l)})^{-1}$, to re-scale the updating gradient and the final updating form becomes
% \begin{equation}\label{eqn-update}
% \begin{split}
%     \mathbf Z &= \mathbf Z^{(l)} - \alpha\left( (1 - (\mathbf D^{(l)})^{-1}  \bm \Gamma^{(l)})\mathbf Z^{(l)}  + (\mathbf D^{(l)})^{-1} \mathbf Z^{(l)} - (\mathbf D^{(l)})^{-1} f(\mathbf Z^{(l)}, W_F)  \right ) \\
%     & = (1 - \alpha - \alpha (\mathbf D^{(l)})^{-1}) \mathbf Z^{(l)} + \alpha (\mathbf D^{(l)})^{-1}  \bm \Gamma^{(l)})\mathbf Z^{(l)} + \alpha(\mathbf D^{(l)})^{-1} f(\mathbf Z^{(l)}, W_F).
% \end{split}
% \end{equation}
% We assume $f(\mathbf Z^{(l)}, W_F) = W_F \mathbf Z^{(l)} + \mathbf Z^{(l)}$ and then Eq.~\ref{eqn-update} becomes
\begin{equation}\label{eqn-iter-unfolding2}
    \mathbf Z^{(l+1)} = (1 - \alpha') \mathbf Z^{(l)} + \alpha' (\mathbf D^{(l)})^{-1}  \bm \Omega^{(l)}\mathbf Z^{(l)}. %+ \alpha (\mathbf D^{(l)})^{-1}  g(\mathbf Z^{(0)}, W_F).
\end{equation}

% On the other side, common practice to solve a PDE harnesses numerical method such as explicit Euler scheme which yields the approximated dynamic for Eq.~\ref{eqn-diffuse2} as
% \begin{equation}
%     \frac{\mathbf z_u^{(k+1)} - \mathbf z_u^{(k)}}{\tau} = \sum_{v\in V} \mathbf S_{uv}(\mathbf Z^{(k)}, \mathbf A) (\mathbf z_v^{(k)} - \mathbf z_u^{(k)}). %+ f(\mathbf z^{(0)}, W_F)
% \end{equation}
% The iteration induces layer-wise updating rule:
% \begin{equation}
%     \mathbf z_u^{(k+1)} = \left (1 - \tau \sum_{v\in V} \mathbf S_{uv}^{(k)} \right ) \mathbf z_u^{(k)} + \tau \sum_{v\in V} \mathbf S_{uv}^{(k)} \mathbf z_v^{(k)}. %+ \tau f(\mathbf z^{(0)}, W_F).
% \end{equation}
One can notice that Eq.~\ref{eqn-iter-unfolding2} shares similar forms as the numerical iteration Eq.~\ref{eqn-diffuse-iter} for the PDE diffusion system, in particular if we write Eq.~\ref{eqn-diffuse-iter} as a matrix form:
\begin{equation}\label{eqn-iter-pde}
    \mathbf Z^{(k+1)} = \left ( 1 - \tau \tilde{\mathbf D}^{(k)} \right ) \mathbf Z^{(k)} + \tau \mathbf S^{(k)} \mathbf Z^{(k)}. %+ \tau f(\mathbf Z^{(0)}, W_F),
\end{equation}
where $\tilde{\mathbf D}^{(k)}$ is the degree matrix associated with $\mathbf S^{(k)}$. Pushing further,
we can see that the effect of Eq.~\ref{eqn-iter-pde} is the same as \ref{eqn-iter-unfolding2} when we let $\tau = \alpha'$, $k = l$, $\mathbf S^{(k)} = (\mathbf D^{(l)})^{-1}\bm \Omega^{(l)}$ and $\mathbf S^{(k)}$ is row-normalized, i.e., $\sum_{v\in V} \mathbf S^{(k)}_{uv} = 1$ and $\tilde{\mathbf D}^{(k)} = \mathbf I$.  

Thereby, we have proven by construction that a one-step numerical iteration by the explicit Euler scheme, specifically shown by Eq.~\ref{eqn-iter-unfolding2} is equivalent to a one-step gradient descent on the surrogate Eq.~\ref{eqn-proof-convex} which further equals to the original energy Eq.~\ref{eqn-energy}. We thus have the result $E(\mathbf Z^{(k+1)}, k; \delta) \leq E(\mathbf Z^{(k)}, k; \delta)$. Besides, we notice that for a fixed $\mathbf Z$, $E(\mathbf Z, k; \delta) = \|\mathbf Z - \mathbf Z^{(k)}\|^2_{\mathcal F} + \lambda \sum_{ij} \rho(\|\mathbf z_i - \mathbf z_j\|^2_2)$ becomes a function of $k$ and its optimum is achieved if and only if $\mathbf Z^{(k)} = \mathbf Z$. Such a fact yields that $E(\mathbf Z^{(k)}, k; \delta) \leq E(\mathbf Z^{(k)}, k-1; \delta)$. The result of the main theorem follows by noting that $E(\mathbf Z^{(k+1)}, k; \delta) \leq E(\mathbf Z^{(k)}, k; \delta) \leq E(\mathbf Z^{(k)}, k-1; \delta)$.

%\textbf{Implication} \emph{1. This result suggests a close-form solution for the diffusivity function that could rigorously guarantee minimization of the energy through iteratively feed-forward executing the diffusion process. 2. Also, since it remains room for us to flexibly specify the form of $\delta$ associated with the expression of optimal $\omega_{uv}$, we can carefully customize the model design for expressiveness, scalability and stability.}

\begin{figure}[t!]
\centering
\begin{minipage}{0.995\linewidth}
\subfigure[\model-s: $\delta(z^2) = 2z^2 - \frac{1}{4}z^4$]{
\begin{minipage}[t]{0.48\linewidth}
\centering
\label{fig-sparse}
\includegraphics[width=0.98\textwidth,angle=0]{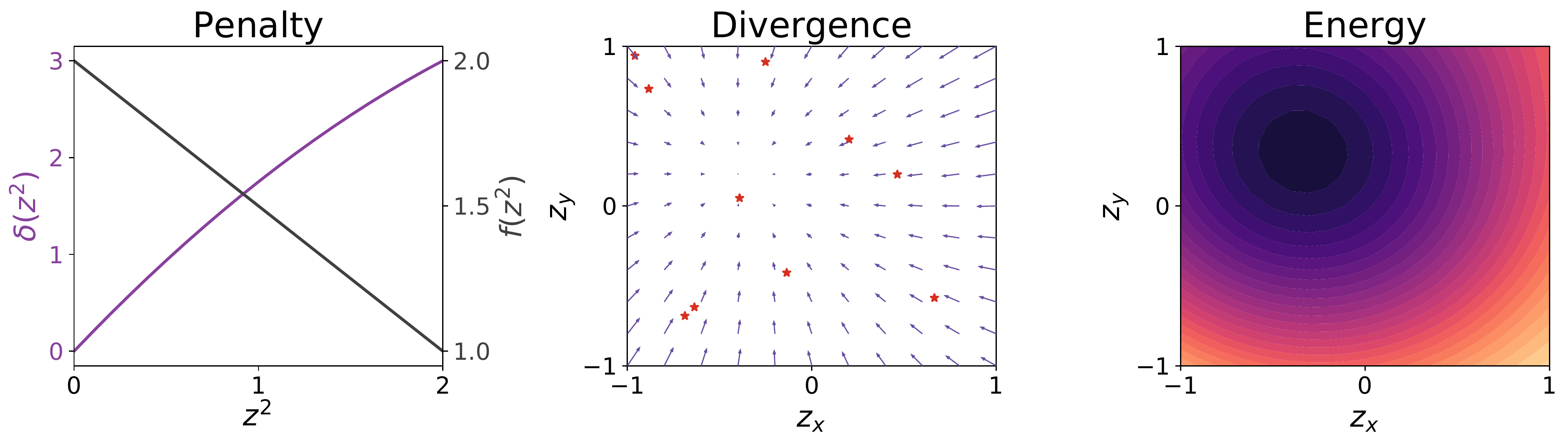}
\end{minipage}%
}%
\hspace{5pt}
\subfigure[\model-a: $\delta(z^2) = z^2 - 2 \log (e^{z^2/2-1} + 1)$]{
\begin{minipage}[t]{0.48\linewidth}
\centering
\label{fig-attn-a}
\includegraphics[width=0.98\textwidth,angle=0]{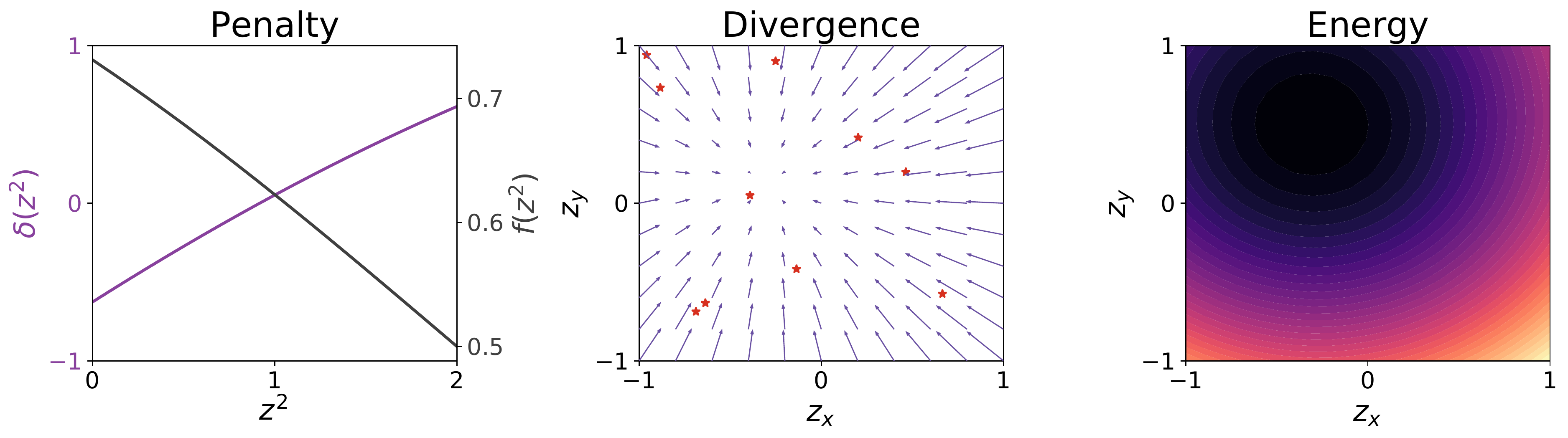}
\end{minipage}%
}
\end{minipage}
\vspace{-10pt}
\caption{Plot of penalty curves $\delta(z^2)$ and $f(z^2)=\frac{\partial \delta(z^2)}{\partial z^2}$, divergence field (produced by 10 randomly generated instances marked as red stars) and cross-section energy field of an individual.}
\label{fig-endiff-vis}
\vspace{-10pt}
\end{figure}

\section{Extension of our Theory to Incorporate Non-Linearity}\label{appx-non}

In Section~\ref{sec:model}, we mainly focus on the situation without non-linearity in each model layer to stay consistent with our implementation where we empirically found that omitting the non-linearity in the middle diffusion layers works smoothly in practice (see the pseudo code of Alg.~\ref{alg-endiff} in appendix for details). Even so, our theory can also be extended to incorporate the layer-wise non-linearity in diffusion propagation. Specifically, the non-linear activation can be treated as a proximal operator (which projects the output into a feasible region) and the gradient descent used in our analysis Eqn.~\ref{eqn-iter-unfolding} and \ref{eqn-iter-unfolding2} can be modified to add a proximal operator:
\[
\mathbf Z^{(l+1)} = \mbox{Prox}_{\Omega} \left ((1 - \alpha) \mathbf Z^{(l)} + \alpha (\mathbf D^{(l)})^{-1}  \mathbf \Omega^{(l)}\mathbf Z^{(l)} \right ),\]
where $\mbox{Prox}_{\Omega} (z) = \arg\min_{x\in \Omega} \|x - z\|_2$ and $\Omega$ defines a feasible region. The updating above corresponds to proximial gradient descent which also guarantees a strict minimization for the energy function and our Theorem~\ref{thm-main} will still hold. In particular, if one uses ReLU activation, the proximal operator will be $\mbox{Prox}_{\Omega} (z) = \max(0, z)$.

\section{Different Energy Forms}\label{appx-inst}

We present more detailed illustration for the choices of $f$ and specific energy function forms in Eq.~\ref{eqn-energy}.

\textbf{Simple Diffusivity Model.} As discussed in Section~\ref{sec:inst}, the simple model assumes $f(z^2) = 2 - \frac{1}{2}z^2$ that corresponds to $g(x) = 1+x$, where we define $z = \|\mathbf z_i - \mathbf z_j\|_2$ and $x = \mathbf z_i^\top \mathbf z_j$. The corresponding penalty function $\delta$ whose first-order derivative is $f$ would be $\delta(z^2) = 2z^2 - \frac{1}{4}z^4$. We plot the penalty function curves in Fig.~\ref{fig-endiff-vis}(a). As we can see, the $f$ is a non-negative, decreasing function of $z^2$, which implies that the $\delta$ satisfies the non-decreasing and concavity properties to guarantee a valid regularized energy function. Also, in Fig.~\ref{fig-endiff-vis}(a) we present the divergence field produced by 10 randomly generated instances (marked as red stars) and the cross-section energy field of one instance.  

\textbf{Advanced Diffusivity Model.} The diffusivity model defines $f(z^2) = \frac{1}{1+e^{z^2/2 - 1}}$ with $g(x) = \frac{1}{1+\exp(-x)}$, and the corresponding penalty function $\delta(z^2) = z^2 - 2 \log (e^{z^2/2-1} + 1)$. The penalty function curves, divergence field and energy field are shown in Fig.~\ref{fig-endiff-vis}(b).

\textbf{Incorporating Input Graphs.} In Section~\ref{sec:inst}, we present an extended version of our model for incorporating input graphs. Such an extension defines a new diffusion process whose iterations are equivalent (up to a re-scaling factor on the adjacency matrix) to a sequence of descending steps for the following regularized energy: 
\begin{equation}\label{eqn-energy-graph}
    E(\mathbf Z, k; \delta) = \|\mathbf Z - \mathbf Z^{(k)}\|_{\mathcal F}^2 + \frac{\lambda}{2}\sum_{i,j} \delta(\|\mathbf z_i - \mathbf z_j\|_2^2) + \frac{\lambda}{2}\sum_{(i,j)\in \mathcal E} \|\mathbf z_i - \mathbf z_j\|_2^2,
\end{equation}
where the last term contributes to a penalty for observed edges in the input graph.

\section{Details of DIFFormer Model}\label{appx-alg}

In this section, we present the details for the feed-forward computation of \model. 

\subsection{DIFFormer's Feed-forward with A Matrix View}

\textbf{Input Layer.} For input data $\mathbf X = \{\mathbf x_i\}_{i=1}^N \in \mathbb R^{N\times D}$ where $\mathbf x_i$ denotes the $D$-dimensional input features of the $i$-th instance, we first use a shallow fully-connected layer to convert it into a $d$-dimensional embedding in the latent space:
\begin{equation}
    \mathbf Z = \sigma\left (\mbox{LayerNorm}(\mathbf W_I \mathbf X + \mathbf b_I) \right ),
\end{equation}
where $\mathbf W_I \in \mathbb R^{d\times D}$ and $\mathbf b_I \in \mathbb R^d$ are trainable parameters, and $\sigma$ is a non-linear activation (i.e., ReLU). Then the node embeddings $\mathbf Z$ will be used for feature propagation with our diffusion-induced Transformer model by letting $\mathbf Z^{(0)} = \mathbf Z$ as the initial states.

\textbf{Propagation Layer.} The initial embeddings $\mathbf Z^{(0)}$ will be transformed into $\mathbf Z^{(1)}, \cdots, \mathbf Z^{(L)}$ with $L$ layers of propagation. We next illustrate one-layer propagation from $\mathbf Z^{(k)}$ to $\mathbf Z^{(k+1)}$. We use the superscript $(k, h)$ to denote the $k$-th layer and the $h$-th head:
\begin{equation}
    \mathbf K^{(k, h)} =  \mathbf W_{K}^{(k, h)}\mathbf Z^{(k)}, \quad \mathbf Q^{(k, h)} =  \mathbf W_{Q}^{(k, h)}\mathbf Z^{(k)}, \quad
    \mathbf V^{(k, h)} =  \mathbf W_{V}^{(k, h)}\mathbf Z^{(k)},
\end{equation}
where $\mathbf W_{K}^{(k, h)}\in \mathbb R^{d\times d}, \mathbf W_{Q}^{(k, h)}\in \mathbb R^{d\times d}, \mathbf W_{V}^{(k, h)}\in \mathbb R^{d\times d}$ are trainable parameters of the $h$-th head at the $k$-th layer. Then the transformed embeddings will be fed into the all-pair propagation unit of \model-s or \model-a.
\begin{itemize}
    \item For \model-s: we adopt L2 normalization for key and query vectors
    \begin{equation}
        \tilde{\mathbf K}^{(k, h)} = \left [\frac{\mathbf K^{(k, h)}_{i}}{\|\mathbf K^{(k, h)}_{i}\|_2} \right ]_{i=1}^N, \quad \tilde{\mathbf Q}^{(k, h)} = \left [\frac{\mathbf Q^{(k, h)}_{i}}{\|\mathbf Q^{(k, h)}_{i}\|_2}\right ]_{i=1}^N,
    \end{equation}
    where $\mathbf K^{(k,h)}_{i}$ denotes the $i$-th row vector of $\mathbf K^{(k, h)}\in \mathbb R^{N\times d}$. Then the all-pair propagation of the $h$-th head is achieved by
    \begin{equation}
        \mathbf R^{(k, h)} = \mbox{diag}^{-1}\left (N + \tilde{\mathbf Q}^{(k, h)} \left ((\tilde{\mathbf K}^{(k, h)} )^\top \mathbf 1 \right ) \right ), 
    \end{equation}
    \begin{equation}
    \mathbf P^{(k, h)} = \mathbf R^{(k, h)} \left [\mathbf 1 \left(\mathbf 1^\top \mathbf V^{(k, h)} \right ) + \tilde{\mathbf Q}^{(k, h)} \left ((\tilde{\mathbf K}^{(k, h)}  )^\top \mathbf V^{(k, h)} \right ) \right ],
    \end{equation}
    where $\mathbf 1_{N\times 1}$ is an all-one vector. The above computation only requires linear complexity w.r.t. $N$ since the bottleneck computation lies in $\tilde{\mathbf Q}^{(k, h)} \left ((\tilde{\mathbf K}^{(k, h)}  )^\top \mathbf V^{(k, h)} \right )$ where the two matrix products both require $O(Nd^2)$.
    \item For \model-a: we need to compute the all-pair similarity before aggregating the results
    \begin{equation}
        \tilde{\mathbf A}^{(k, h)} = \mbox{Sigmoid}\left (\mathbf Q^{(k, h)} (\mathbf K^{(k, h)})^\top \right ),
    \end{equation}
    \begin{equation}
        \mathbf R^{(k, h)} = \mbox{diag}^{-1} \left( \tilde{\mathbf A}^{(k, h)} \mathbf 1 \right ),
    \end{equation}
    \begin{equation}
        \mathbf P^{(k, h)} = \mathbf R^{(k, h)} \tilde{\mathbf A}^{(k, h)} \mathbf V^{(k, h)}. 
    \end{equation}
    The above computation requires $O(Nd^2+N^2d)$ due to the explicit computation of the $N\times N$ matrix $\tilde{\mathbf A}^{(k, h)}$.
\end{itemize}
If using input graphs, we add the updated embeddings of GCN-based propagation to the all-pair propagation's ones:
\begin{equation}
    \overline{\mathbf P}^{(k, h)} = \mathbf P^{(k, h)} + \mathbf D^{-\frac{1}{2}} \mathbf A \mathbf D^{-\frac{1}{2}} \mathbf V^{(k, h)},
\end{equation}
where $\mathbf A$ is the input graph and $\mathbf D$ denotes its corresponding diagonal degree matrix.

We then average the propagated results of multiple heads:
\begin{equation}
    \overline{\mathbf P}^{(k)} = \frac{1}{H} \sum_{h=1}^H \overline{\mathbf P}^{(k, h)}.
\end{equation}
The next-layer embeddings will be updated by
\begin{equation}
    \mathbf Z^{(k+1)} = \sigma' \left ( \mbox{LayerNorm}\left ( \tau \overline{\mathbf P}^{(k)} + (1 - \tau) \mathbf Z^{(k)}\right ) \right ),
\end{equation}
where $\sigma'$ can be identity mapping or non-linear activation (e.g., ReLU). 

\textbf{Output Layer.} After $K$ layers of propagation, we then use a shallow fully-connected layer to output the predicted logits:
\begin{equation}
    \hat {\mathbf Y} = \mathbf Z^{(K)} \mathbf W_O + \mathbf b_O,
\end{equation}
where $\mathbf W_O\in \mathbb R^{d\times C}$ and $\mathbf b_O \in \mathbb R^{C}$ are trainable parameters, and $C$ denotes the number of classes. And, the predicted logits $\hat {\mathbf Y}$ will be used for computing a loss of the form $l(\hat {\mathbf Y}, \mathbf Y)$ where $l$ can be cross-entropy for classification or mean square error for regression.

\subsection{Pseudo Codes}

We provide the Pytorch-style pseudo codes for \model class in Alg.~\ref{alg-endiff} and the one-layer propagation of two model versions (shown in Alg.~\ref{alg-endiff-l} for \model-s and Alg.~\ref{alg-endiff-n} for \model-a). %In particular, as shown in Alg.~\ref{alg-training}, the training procedure of our model follows a standard paradigm where the input data, i.e., a dataset of (or a mini-batch of) instance features and graph structures (if available), are fed into the model for computing prediction results for each input instance, and then a supervised loss is computed with labeled data for gradient-based optimization. 
The key design of our methodology lies in the model architectures which are shown in detail in Alg.~\ref{alg-endiff-l} for \model-s and Alg.~\ref{alg-endiff-n} for \model-a, where for each case, the model takes the data as input and outputs prediction for each individual instance. For more details concerning the implementation, please refer to our provided codes at the first page.

\begin{algorithm}[H]
        \caption{PyTorch-style Code for \model}
        \label{alg-endiff}
        \definecolor{codeblue}{rgb}{0.25,0.5,0.5}
        \definecolor{codekw}{rgb}{0.85, 0.18, 0.50}
        \lstset{
          backgroundcolor=\color{white},
          basicstyle=\fontsize{7.5pt}{7.5pt}\ttfamily\selectfont,
          columns=fullflexible,
          breaklines=true,
          captionpos=b,
          commentstyle=\fontsize{7.5pt}{7.5pt}\color{codeblue},
          keywordstyle=\fontsize{7.5pt}{7.5pt}\color{codekw},
        }
\begin{lstlisting}[language=python]
# fcs: fully-connected layers
# bns: layer normalization layers
# convs: DIFFormer layers (see implementation in Alg. 2 and 3)
# activation: activation function for the input layer
# x: input data sized [N, D], N for instance number, D for input feature dimension
# edge_index: input graph structure if available, None otherwise
# tau: step size for each iteration update
# use_act: whether to use activation for propagation layers

layer_ = []

# input MLP layer
x = fcs[0](x)
x = bns[0](x)
x = activation(x)

# store as residual link
layer_.append(x)

for i, conv in enumerate(convs):
    # graph convolution with global all-pair attention (specified by Alg. 2 for DIFFormer-s and Alg. 3 for DIFFormer-a)
    x = conv(x, x, edge_index)
    x = tau * x + (1-tau) * layer_[i]
    x = bns[i+1](x)
    if use_act:
        x = activation(x)
    layer_.append(x)

# output MLP layer
out = fcs[-1](x)

# supervised loss calculation, negative log-likelihood
y_logp = F.log_softmax(out, dim=1)
loss = criterion(y_logp[train_idx], y_true[train_idx])

\end{lstlisting}
\vspace{-7pt}
\end{algorithm}
  
\begin{algorithm}[H]
        \caption{PyTorch-style Code for One-layer Feed-forward of \model-s}
        \label{alg-endiff-l}
        \definecolor{codeblue}{rgb}{0.25,0.5,0.5}
        \definecolor{codekw}{rgb}{0.85, 0.18, 0.50}
        \lstset{
          backgroundcolor=\color{white},
          basicstyle=\fontsize{7.5pt}{7.5pt}\ttfamily\selectfont,
          columns=fullflexible,
          breaklines=true,
          captionpos=b,
          commentstyle=\fontsize{7.5pt}{7.5pt}\color{codeblue},
          keywordstyle=\fontsize{7.5pt}{7.5pt}\color{codekw},
        }
\begin{lstlisting}[language=python]
# x: data embeddings sized [N, d], N for instance number, d for hidden size
# edge_index: input graph structure if available, None otherwise
# H: head number
# use_graph: whether to use input graph
# use_weight: whether to use feature transformation for each layer
# graph_conv: graph convolution operatior using the normalized adjacency matrix D^{-1/2}AD^{-1/2}
# Wq, Wk, Wv: weight matrices for feature transformation

Q = Wq(x) # [N, H, D]
K = Wk(x) # [N, H, D]
V = use_weight * Wv(x) + (1 - use_weight) * x # [N, H, D]

# numerator
KV = torch.einsum("lhm,lhd->hmd", K, V)
num = torch.einsum("nhm,hmd->nhd", Q, KV) # [N, H, D]
num += N * V

# denominator
all_ones = torch.ones(N)
K_sum = torch.einsum("lhm,l->hm", K, all_ones)
den = torch.einsum("nhm,hm->nh", Q, K_sum)  # [N, H]

# aggregated results
den += torch.ones_like(den) * N
agg = num / den.unsqueeze(2) # [N, H, D]

# use input graph for graph conv
if use_graph:
    agg = agg + graph_conv(V, edge_index)
output = agg.mean(dim=1)

\end{lstlisting}
\vspace{-7pt}
  \end{algorithm}
  
 \begin{algorithm}[H]
        \caption{PyTorch-style Code for One-layer Feed-forward of \model-a}
        \label{alg-endiff-n}
        \definecolor{codeblue}{rgb}{0.25,0.5,0.5}
        \definecolor{codekw}{rgb}{0.85, 0.18, 0.50}
        \lstset{
          backgroundcolor=\color{white},
          basicstyle=\fontsize{7.5pt}{7.5pt}\ttfamily\selectfont,
          columns=fullflexible,
          breaklines=true,
          captionpos=b,
          commentstyle=\fontsize{7.5pt}{7.5pt}\color{codeblue},
          keywordstyle=\fontsize{7.5pt}{7.5pt}\color{codekw},
        }
\begin{lstlisting}[language=python]
# x: data embeddings sized [N, d], N for instance number, d for hidden size
# edge_index: input graph structure if available, None otherwise
# H: head number
# use_graph: whether to use input graph
# use_weight: whether to use feature transformation for each layer
# graph_conv: graph convolution operatior using the normalized adjacency matrix D^{-1/2}AD^{-1/2}
# Wq, Wk, Wv: weight matrices for feature transformation

Q = Wq(x) # [N, H, D]
K = Wk(x) # [N, H, D]
V = use_weight * Wv(x) + (1 - use_weight) * x # [N, H, D]

# numerator
num = torch.sigmoid(torch.einsum("nhm,lhm->nlh", Q, K))  # [N, N, H]

# denominator
all_ones = torch.ones(N)
den = torch.einsum("nlh,l->nh", num, all_ones)
den = den.unsqueeze(1).repeat(1, N, 1)  # [N, N, H]

# aggregated results
attn = num / den # [N, N, H]
agg = torch.einsum("nlh,lhd->nhd", attn, V)  # [N, H, D]

# use input graph for graph conv
if use_graph:
    agg = agg + graph_conv(V, edge_index)
output = agg.mean(dim=1)
\end{lstlisting}
\vspace{-7pt}
  \end{algorithm}

\section{Connections with Existing Models}\label{appx-connection}

\textbf{MLP.} MLPs can be viewed as a simplified diffusion model with only non-zero diffusivity values on the diagonal line, i.e., $\mathbf S^{(k)}_{ij} = 1$ if $i = j$ and 0 otherwise. From a graph convolution perspective, MLP only considers propagation on the self-loop connection in each layer. Correspondlying, the energy function the feed-forward process of the model essentially optimizes would be $E(\mathbf Z, k) = \|\mathbf Z - \mathbf Z^{(k)}\|^2_{\mathcal F}$, which only counts for the local consistency term and ignores the global information. 

\textbf{GCN.} Graph Convolution Networks~\citep{GCN-vallina} define a convolution operator on a graph $\mathcal G=(\mathcal V, \mathcal E)$ by multiplying the node feature matrix with $D^{-\frac{1}{2}}AD^{-\frac{1}{2}}$ where $A$ denotes the adjacency matrix and $D$ is its associated degree matrix. The layer-wise updating rule can be written as
\begin{equation}\label{eqn-model-gcn}
    \begin{split}
        & \mathbf {\hat S}^{(k)}_{ij} = \left\{ 
    \begin{aligned}
         &\frac{1}{\sqrt{d_id_j}}, \quad \mbox{if} \; (i,j) \in \mathcal E,  \\
         &0, \quad otherwise,
    \end{aligned}
    \right. \\
        & \mathbf z_i^{(k+1)} = \left (1 - \tau \sum_{j=1}^N \mathbf {\hat S}_{ij}^{(k)} \right ) \mathbf z_i^{(k)}  + \tau \sum_{j=1}^N \mathbf {\hat S}_{ij}^{(k)} \mathbf z_j^{(k)}, \quad 1 \leq i \leq N.
    \end{split}
\end{equation}
Eq.~\ref{eqn-model-gcn} generalizes the original message-passing rule of GCN by adding an additional self-loop links with adaptive weights for different nodes. The diffusivity matrix is defined with the observed adjacency, i.e., $\mathbf {\hat S}^{(k)} = D^{-\frac{1}{2}}AD^{-\frac{1}{2}}$. Though such a design leverages the geometric information from input structures as a guidance for feature propagation, it constrains the efficiency of layer-wise information flows within the receptive field of local neighbors and could only exploit partial global information with a finite number of iterations.

\textbf{GAT.} Graph Attention Networks~\citep{GAT} extend the GCN architecture to incorporate attention mechanisms as a learnable function producing adaptive weights for each observed edge. From our diffusion perspective, the attention matrix for layer-wise convolution can be treated as the diffusivity in our updating rule:
\begin{equation}\label{eqn-model-gat}
    \begin{split}
        & \mathbf {\hat S}_{ij}^{(k)} = \frac{f(\|\mathbf z_i^{(k)} - \mathbf z_j^{(k)}\|_2^2)}{\sum_{(i,l)\in \mathcal E} f(\|\mathbf z_i^{(k)} - \mathbf z_l^{(k)}\|_2^2)}, \quad (i,j)\in \mathcal E,\\
        & \mathbf z_i^{(k+1)} = \left (1 - \tau \sum_{j=1}^N \mathbf {\hat S}_{ij}^{(k)} \right ) \mathbf z_i^{(k)}  + \tau \sum_{j=1}^N \mathbf {\hat S}_{ij}^{(k)} \mathbf z_j^{(k)}, \quad 1 \leq i \leq N.
    \end{split}
\end{equation}
We notice that $\sum_{j=1}^N \mathbf {\hat S}_{ij}^{(k)} = 1$ due to the normalization in the denominator. Therefore, the layer-wise updating can be viewed as an attentive aggregation over graph structures and a subsequent residual link. In fact, the original implementation for the GAT model~\citep{GAT} specifies the function $f$ as a particular form, i.e., $\exp(\mbox{LeakyReLU}(\mathbf a^\top [\mathbf W \mathbf z_i^{(k)} \| \mathbf W \mathbf z_j^{(k)}]))$, which can be viewed as a generalized similarity function given trainable $\mathbf W$, $\mathbf a$ towards optimizing a supervised loss.

\section{Dataset Information}\label{appx-dataset}

In this section, we present the detailed information for all the experimental datasets, the pre-processing and evaluation protocol used in Section~\ref{sec:exp}. 

\begin{table*}[htbp]
    \centering
    \caption{Information for node classification datasets.}
    \label{tbl:dataset}
    \resizebox{0.75\textwidth}{!}{
    \begin{tabular}{lccccccc}
    \toprule
    Dataset & Type & \# Nodes & \# Edges & \# Node features & \# Class  \\ 
    \midrule
    Cora & Citation network & 2,708 & 5,429 & 1,433 & 7 \\
    Citeseer & Citation network & 3,327 & 4,732 & 3,703 & 6 \\
    Pubmed & Citation network & 19,717 & 44,338 & 500 & 3 \\
    Proteins & Protein interaction & 132,534 & 39,561,252 & 8  & 2 \\ 
    Pokec & Social network & 1,632,803 & 30,622,564 & 65 & 2 \\
    \bottomrule
    \end{tabular}
    }
\end{table*}

\subsection{Node Classification Datasets}
\texttt{Cora}, \texttt{Citeseer} and \texttt{Pubmed}~\citep{Sen08collectiveclassification} are commonly used citation networks for evaluating models on node classification tasks., These datasets are small-scale networks (with 2K$\sim$20K nodes) and the goal is to classify the topics of documents (instances) based on input features of each instance (bag-of-words representation of documents) and graph structure (citation links). 
Following the semi-supervised learning setting in \cite{GCN-vallina}, we randomly choosing 20 instances per class for training, 500/1000 instances for validation/testing for each dataset. 
\texttt{OGBN-Proteins}~\citep{ogb-nips20} is a multi-task protein-protein interaction network whose goal is to predict molecule instances' property. We follow the original splitting of \cite{ogb-nips20} for evaluation. \texttt{Pokec} is a large-scale social network with features including profile information, such as geographical region, registration time, and age, for prediction on users' gender. For semi-supervised learning, we consider randomly splitting the instances into train/valid/test with 10\%/10\%/80\% ratios. Table~\ref{tbl:dataset} summarizes the statistics of these datasets.

\subsection{Image and Text Classification Datasets}
We evaluate our model on two image classification datasets: STL-10 and CIFAR-10. We use all 13000 images from STL-10, each of which belongs to one of ten classes. We choose 1500 images from each of 10 classes of CIFAR-10 and obtain a total of 15,000 images. For STL-10 and CIFAR-10, we randomly select 10/50/100 instances per class as training set, 1000 instances for validation and the remaining instances for testing. We first use the self-supervised approach SimCLR~\citep{chen2020simple} (that does not use labels for training) to train a ResNet-18 for extracting the feature maps as input features of instances.
We also evaluate our model on 20Newsgroup, which is a text classification dataset consisting of 9607 instances. We follow \cite{LDS-icml19} to take 10 classes from 20 Newsgroup and use words (TFIDF) with a frequency of more than 5\% as features.

\subsection{Spatial-Temporal Datasets}
The spatial-temporal datasets are from the open-source library PyTorch Geometric Temporal \citep{rozemberczki2021pytorch}, with properties and summary statistics described in Table~\ref{tab:desc_discrete}. Node features are evolving for all the datasets considered here, i.e., we have different node features for different snapshots. For each dataset, we split the snapshots into training, validation, and test sets according to a 2:2:6 ratio in order to make it more challenging and close to the real-world low-data learning setting.
In details:
\begin{itemize}
    \item \texttt{Chickenpox} describes weekly officially reported cases of chickenpox in Hungary from 2005 to 2015, whose nodes are counties and edges denote direct neighborhood relationships. Node features are lagged weekly counts of the chickenpox cases (we included 4 lags). The target is the weekly number of cases for the upcoming week (signed integers).
    \item \texttt{Covid} contains daily mobility graph between regions in England NUTS3 regions, with node features corresponding to the number of confirmed COVID-19 cases in the previous days from March to May 2020. The graph indicates how many people moved from one region to the other each day, based on Facebook Data For Good disease prevention maps. Node features correspond to the number of COVID-19 cases in the region in the past 8 days. The task is to predict the number of cases in each node after 1 day.
    \item \texttt{WikiMath} is a dataset whose nodes describe Wikipedia pages on popular mathematics topics and edges denote the links from one page to another. Node features are provided by the number of daily visits between 2019 March and 2021 March. The graph is directed and weighted. Weights represent the number of links found at the source Wikipedia page linking to the target Wikipedia page. The target is the daily user visits to the Wikipedia pages between March $16^\text{th}$ 2019 and March $15^\text{th}$ 2021 which results in 731 periods.
\end{itemize}
   
\begin{table}[htb]
\small

\centering
%\small
\caption{Properties and summary statistics of the spatial-temporal datasets used in the experiments with information about whether the graph structure is dynamic or static, meaning of node features (the same as the prediction target) and the corresponding dimension ($D$), the number of snapshots ($T$), the number of nodes ($|V|$), as well as the meaning of edges/edge weights.}\label{tab:desc_discrete}
{
\setlength{\tabcolsep}{1.5pt}
\centering
\resizebox{0.9\textwidth}{!}{
\begin{tabular}{ccccccccc}
\toprule
Dataset  & Graph structure &Node features/ Target&$D$& Frequency & $T$ & $|V|$&Edges/ Edge weights \\
\midrule
Chickenpox&Static&Weekly Chickenpox Cases&4 & Weekly & 522 & 20 &Direct Neighborhoods\\
Covid  & Dynamic&Daily Covid Cases&8 & Daily & 61 & 129&Daily Mobility \\
WikiMath &Static&Daily User Visits&14&Daily & 731 & 1,068&Page Links \\
 \bottomrule
\end{tabular}
}}
\end{table}

\section{Implementation Details and Hyper-parameters}\label{appx-implementation}
\vspace{-5pt}

\subsection{Node Classification Experiment}
We use feature transformation for each layer on two large datasets and omit it for citation networks. The head number is set as 1. We set $\tau=0.5$ and incorporate the input graphs for \model. For other hyper-paramters, we adopt grid search for all the models with learning rate from $\{0.0001, 0.001, 0.01, 0.1\}$, weight decay for the Adam optimizer from $\{0, 0.0001, 0.001, 0.01, 0.1, 1.0\}$, dropout rate from $\{0, 0.2, 0.5\}$, hidden size from $\{16, 32, 64\} $, number of layers from $\{2, 4, 8, 16\}$. For evaluation, we compute the mean and standard deviation of the results with five repeating runs with different initializations. For each run, we run for a maximum of 1000 epochs and report the testing performance achieved by the epoch yielding the best performance on validation set.

\subsection{Image and Text Classification Experiment}
For image and text datasets, we consider feature transformation for layer-wise updating. The head number is set as 1. We set $\tau=0.5$. These datasets do not have input graphs so we only consider learning new structures for the diffusion model. For hyper-parameter settings, we conduct grid search for all the models with learning rate from $\{0.0001, 0.0005, 0.005, 0.01, 0.05\}$, weight decay for the Adam optimizer from $\{0.0001, 0.001, 0.01, 0.1\}$, dropout rate from $\{0, 0.2, 0.5\}$, hidden size from $\{32, 64, 100, 200, 300, 400 \} $, number of layers from $\{1, 2, 4, 6, 8, 10, 12\}$. We average the results for five repeating runs and report as well the standard deviation. For each run, we run for a maximum of 600 epochs and report the testing accuracy achieved by the epoch yielding the highest accuracy on validation set.

\subsection{Spatial-Temporal Prediction Experiment}
%For dynamics prediction, we employ the \emph{cumulative} backpropagation scheme for Chickenpox and Covid datasets and the \emph{incremental} scheme for WikiMath, as described in \cite{rozemberczki2021pytorch}. 
We do not use feature transformation for these datasets due to their small sizes and also set $\tau=0.5$. The head number is set as 1. These spatial-temporal dynamics prediction datasets contain available graph structures, we consider both cases, using the input graphs and not, in our experiments and discuss their impact on the performance. For other hyper-parameters, we also consider grid search for all models here with learning rate from $\{0.01, 0.05, 0.005\}$, weight decay for the Adam optimizer from $\{0, 0.005\}$, dropout rate from $\{0, 0.2, 0.5\}$, and report the test mean squared error (MSE) based on the lowest validation MSE. We average the results for five repeating runs and report as well the standard deviation for each MSE result. For each run, we run for a maximum of 200 epochs in total and stop the training process with 20-epoch early stopping on the validation performance. The data split is done in time order, and hence is deterministic. We report the results using the same hidden size (4) and number of layers (2) for all methods. 

\section{More Experiment Results}\label{appx-results}

We supplement more experiment results including extra ablation studies, hyper-parameter studies and visualization results on more datasets that are not presented in Section~\ref{sec:exp} due to the limit of space. 

\subsection{Ablation Studies}\label{appx-abl}

In Fig.~\ref{fig:res-appx-abl} we present more experiment results for ablation studies w.r.t. the energy function forms used by \model. See discussions and analysis in Section~\ref{sec:exp}.

\begin{figure}[tb!]
\centering
\begin{minipage}{\linewidth}
\centering
\subfigure[\texttt{STL}]{
\begin{minipage}[t]{0.44\linewidth}
\centering
\includegraphics[width=\linewidth]{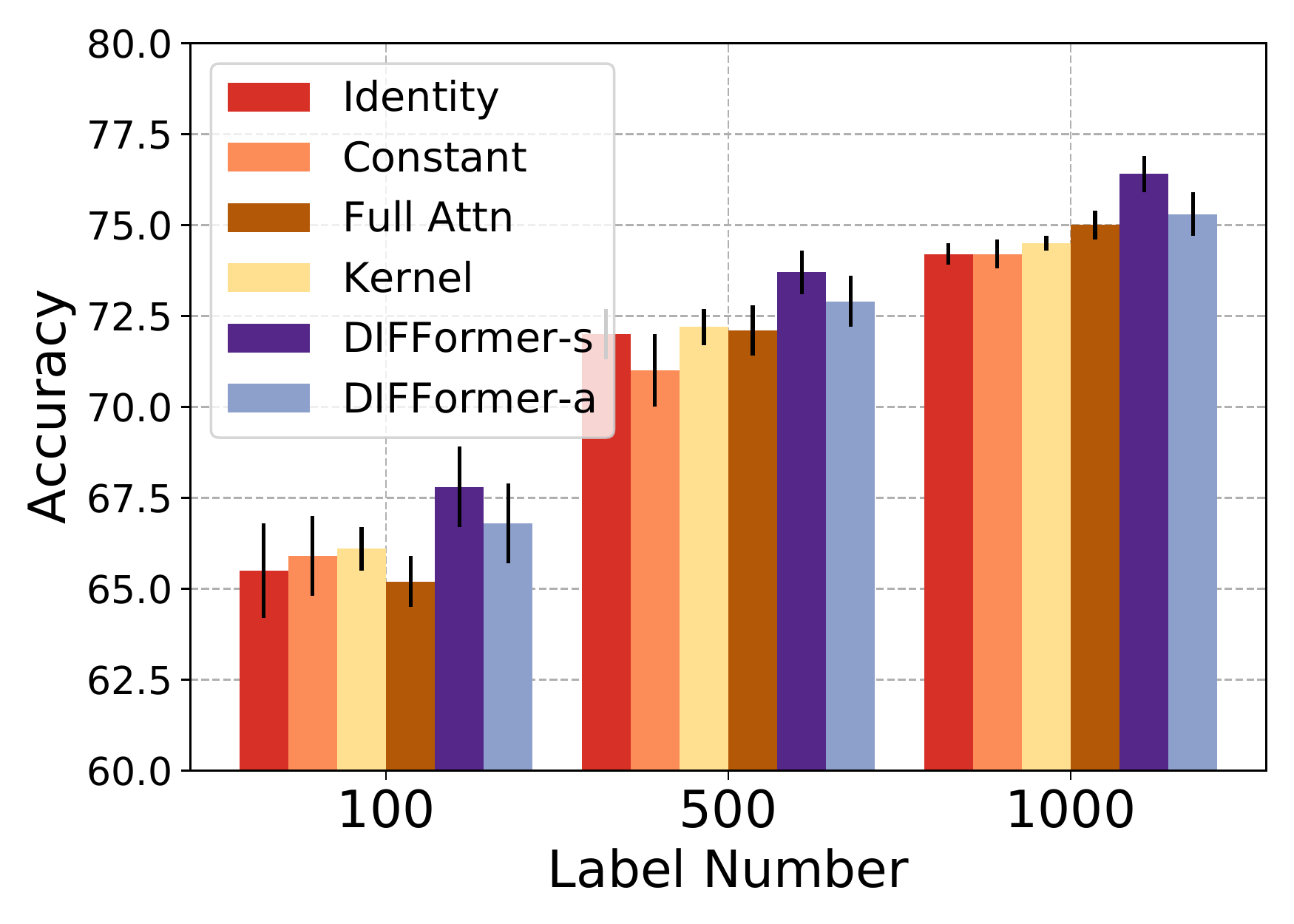}
\end{minipage}
}
\subfigure[\texttt{20News}]{
\begin{minipage}[t]{0.44\linewidth}
\centering
\includegraphics[width=\linewidth]{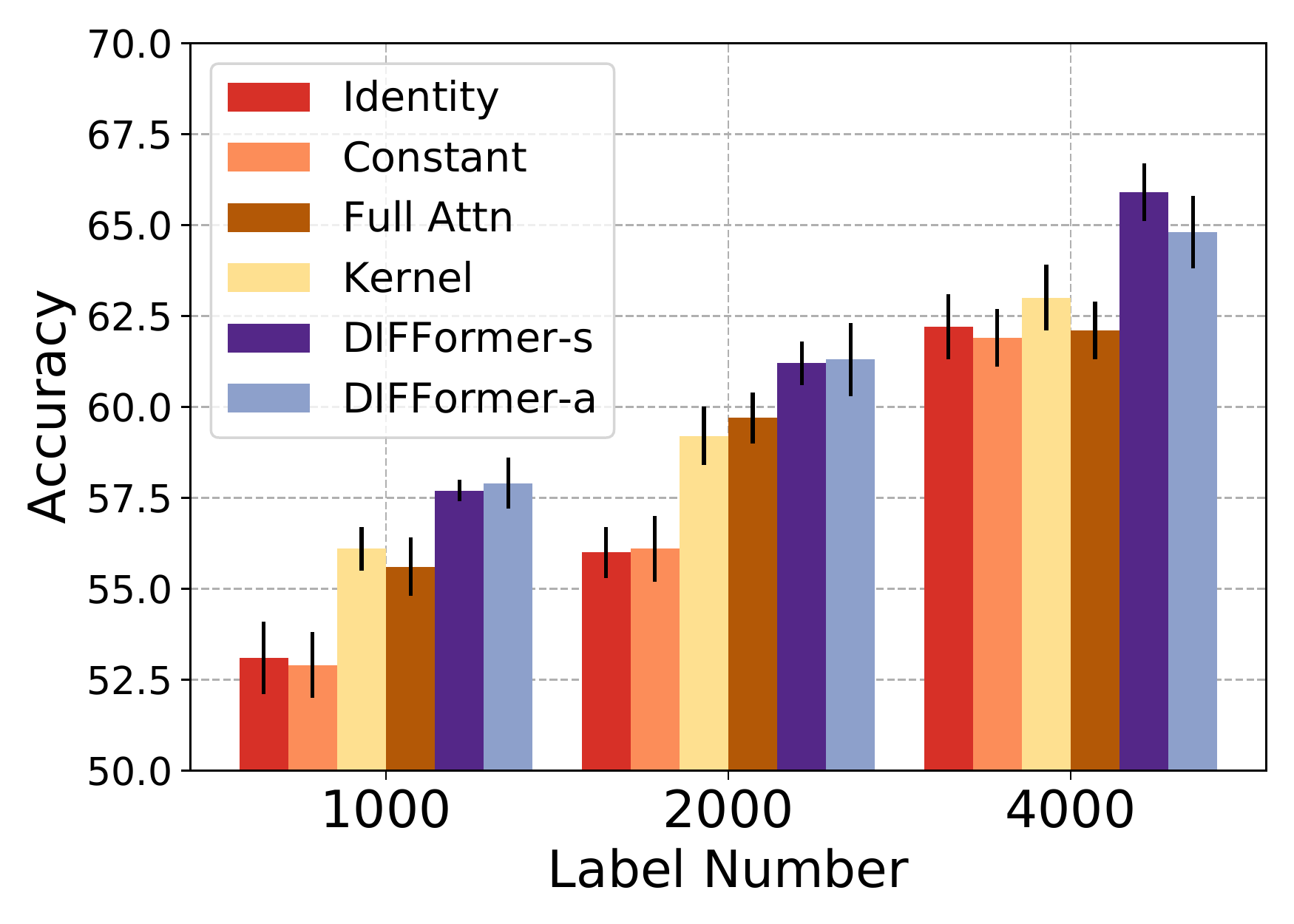}
\end{minipage}
}
\end{minipage}
\caption{Ablation studies for different energy function forms on image and text datasets.~\label{fig:res-appx-abl}}
\end{figure}

\subsection{Hyper-parameter Analysis}\label{appx-hyper}

We plot the testing performance of several baselines and \model with different step size $\tau$ as the model size $K$ increases in Fig.~\ref{fig:res-appx-hyper}. See discussions and analysis in Section~\ref{sec:exp}.

\begin{figure}[tb!]
\centering
\begin{minipage}{\linewidth}
\centering
\subfigure[\texttt{Citeseer}]{
\begin{minipage}[t]{0.44\linewidth}
\centering
\includegraphics[width=\linewidth]{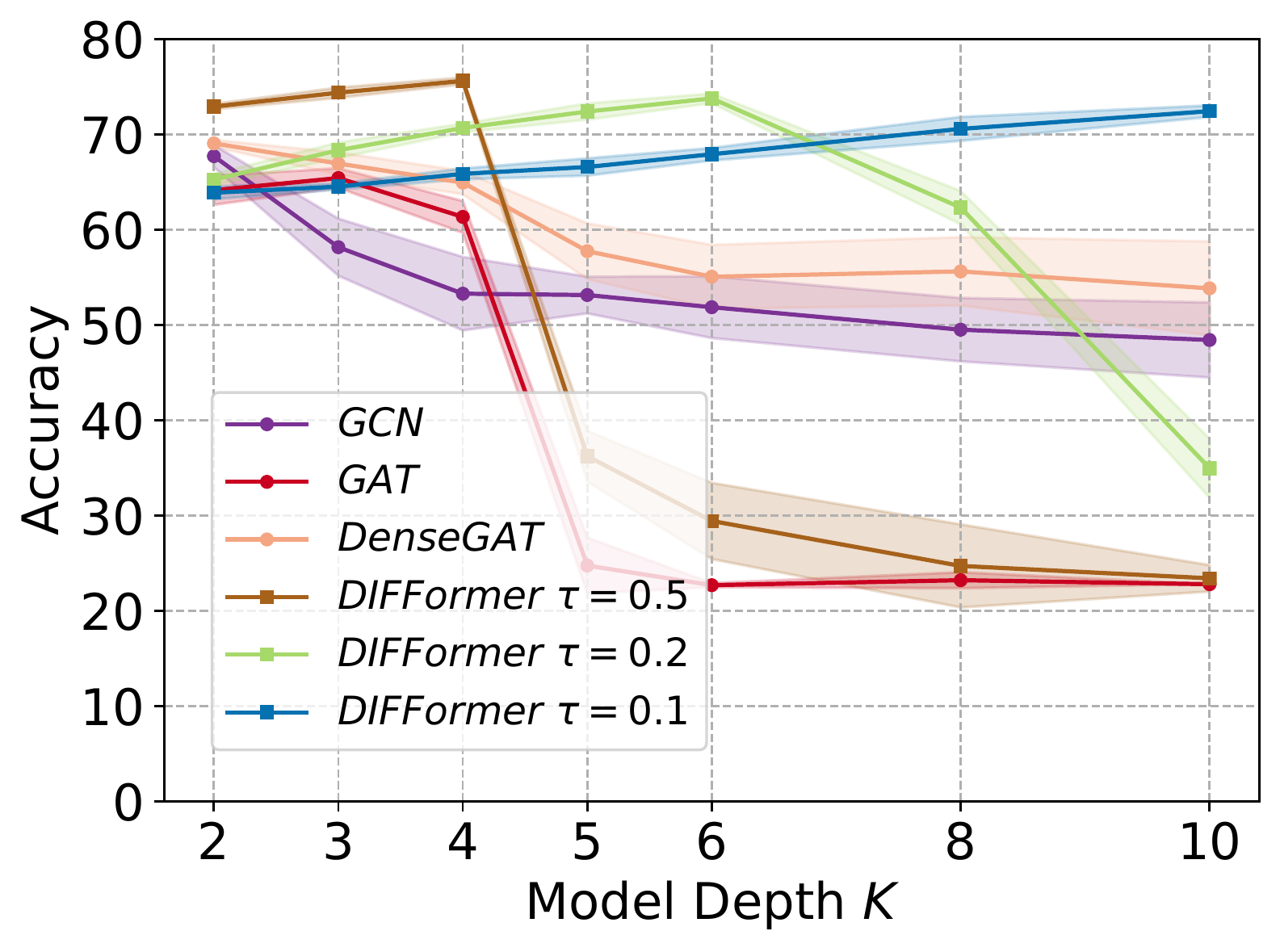}
\end{minipage}
}
\subfigure[\texttt{Pubmed}]{
\begin{minipage}[t]{0.44\linewidth}
\centering
\includegraphics[width=\linewidth]{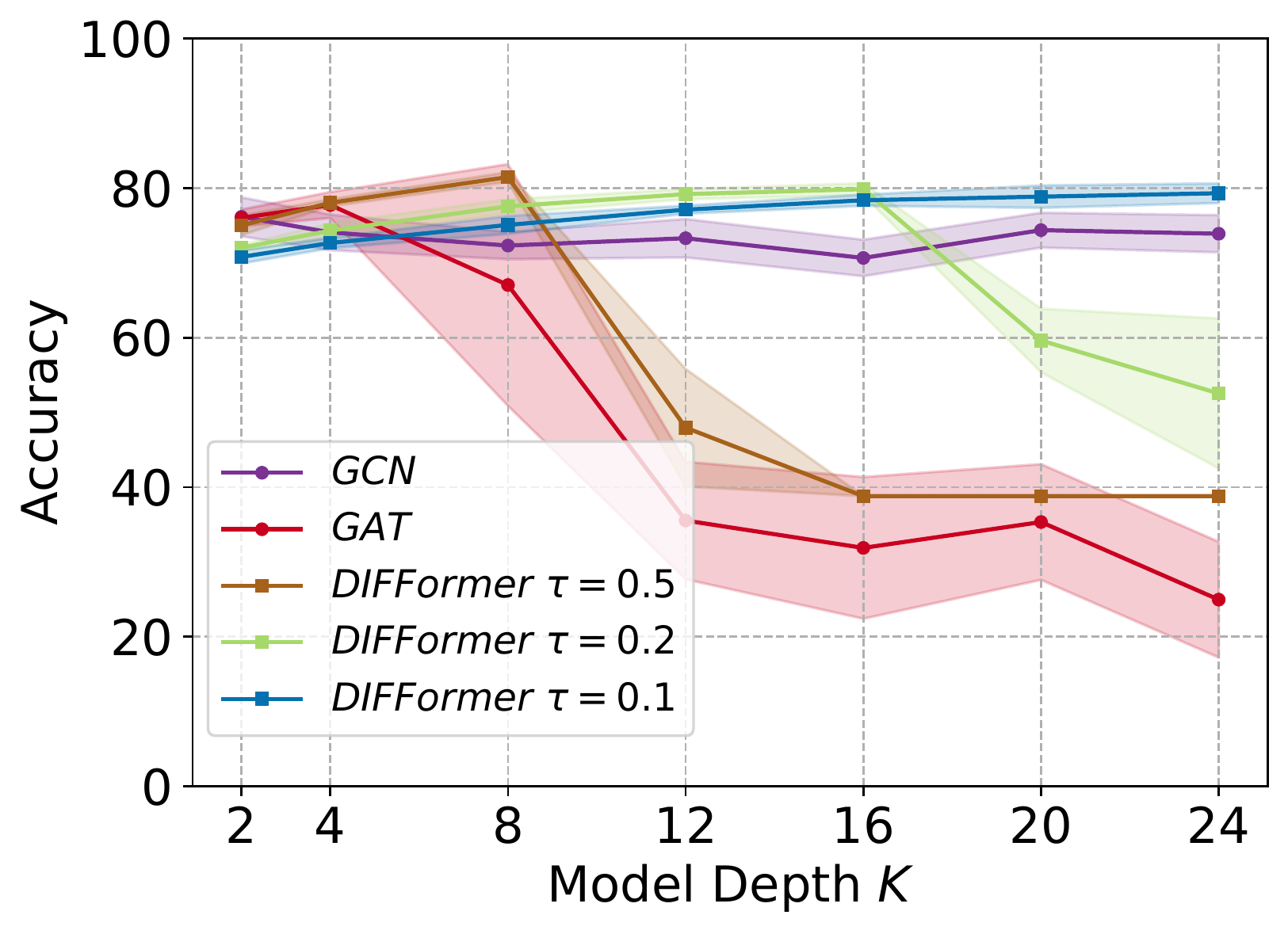}
\end{minipage}
}
\end{minipage}
\caption{Hyper-parameter studies of model depth $K$ and step size $\tau$ on two citation networks. The baseline DenseGAT suffers out-of-memory issue on \textsc{Pubmed}.~\label{fig:res-appx-hyper}}
\end{figure}

\subsection{Visualization}\label{appx-vis}

Fig.~\ref{fig:vis-20news} and \ref{fig:vis-stl} plot the produced instance-level representations and diffusion strength estimates by the model on \texttt{20News} and \texttt{STL}, respectively. We observe that the diffusivity estimates tend to connect nodes with different classes, which contribute to increasing the global connectivity and facilitate absorbing other instances' information for informative representations. The node embeddings produced by our model have small intra-class distance and large inter-class distance, making it easier for the classifier to distinguish.

Fig.~\ref{chickenpox_main_text_visualization} visualizes the diffusivity estimates on \texttt{Chickenpox}. We conclude that large diffusion strength usually exists between nodes with similar ground-truth labels. \model-s has more concentrated large weights while \model-a tends to have large diffusivity spreading out more. \model-a indeed learns more complex underlying structures than \model-s due to its better capacity for diffusivity modeling.

% Fig.~\ref{fig:res-appx-vis-chick-l} and Fig.~\ref{fig:res-appx-vis-chick-n} visualize the produced diffusivity weights on \texttt{Chickenpox} across the first five snapshots and two layers for \model-s and \model-a, respectively. Fig.~\ref{fig:res-appx-vis-covid-l} and Fig.~\ref{fig:res-appx-vis-covid-n} are the counterparts on \texttt{Covid} for \model-s and \model-a, respectively. 
% Overall, for both datasets and both model variants, we observe that large diffusivity weights usually exist between instances with similar ground-truth labels. For different layers of the same model, we can see clear differences of the diffusivity. For example, for \model-a on \texttt{Chickenpox}, the first layer usually have more large weights spreading out while the second layer usually have slightly more concentrated large weights. In contrast, for \model-s on \texttt{Covid}, we observe the opposite. Furthermore, comparing \model-s and \model-a, we find that the latent structures predicted by \model-a are overall more sophisticated than those of \model-s, which accords with our conjecture that \model-a has better capacity for learning more complex structures.

\subsection{Impact of Mini-batch Sizes on Large Graphs}\label{appx-batch}
The randomness of mini-batch partition on large graphs has negligible effect on the performance since we use large batch sizes for training, which is facilitated by the linear complexity of \model-s. Even setting the batch size to be 100000, our model only costs 3GB GPU memory on Pokec. As a further investigation on this, we add more experiments using different batch sizes on \texttt{Pokec} and the results are shown in Table~\ref{tab:dis-minibatch}.

\begin{table}[htb]
\small

\centering
%\small
\caption{Discussions on using different mini-batch sizes for training on \texttt{Pokec}. We report testing accuracy and training memory for comparison.}\label{tab:dis-minibatch}
{
\setlength{\tabcolsep}{1.5pt}
\centering
\resizebox{0.8\textwidth}{!}{
\begin{tabular}{c|c|c|c|c|c|c}
\toprule
Batch size & 5000 & 10000 &20000&50000&100000&200000 \\
\midrule
Test Acc (\%) & 65.24 $\pm$ 0.34&67.48 $\pm$ 0.81&68.53 $\pm$ 0.75&68.96 $\pm$ 0.63&69.24 $\pm$ 0.76&69.15 $\pm$ 0.52 \\
GPU Memory (MB)&1244&1326&1539&2060&2928&4011 \\
 \bottomrule
\end{tabular}
}}
\end{table}

One can see that using small batch sizes would indeed sacrifice the performance yet large batch sizes can produce decent and low-variance results with acceptable memory costs.

\subsection{Comparison of Running Time and Memory Costs}\label{appx-time}

To further study the efficiency and scalability of our model, we provide more comparison regarding the training time per epoch and memory costs of two \model's variants, GCN, GAT and DenseGAT in Table~\ref{tab-time}. One can see that compared to GAT, \model-s costs comparable time on small datasets such as \texttt{Cora} and \texttt{WikiMath}, and is much faster on large dataset \texttt{Pokec}. As for memory consumption, \model-s reduces the costs by several times over DenseGAT, which clearly shows the efficiency of our new diffusion function designs. Overall, \model-s has nice scalability, decent efficiency and yield significantly better accuracy. In contrast, \model-a costs much larger time and memory costs than \model-s, due to its quadratic complexity induced by the explicit computation for the all-pair diffusivity. Still, \model-a accommodates non-linearity for modeling the diffusion strengths which enables better capacity for learning complex layer-wise inter-interactions.

\begin{table}[htb]
\small

\centering
%\small
\caption{Comparison of training time and memory of different models on \texttt{Cora}, \texttt{Pokec}, \texttt{STL-10} and \texttt{WikiMath}. OOM refers to out-of-memory when training on a GPU with 16GB memory.}\label{tab-time}
{
\setlength{\tabcolsep}{1.5pt}
\centering
\resizebox{0.8\textwidth}{!}{
\begin{tabular}{c|c|c|c|c|c|c}
\toprule
 \multicolumn{2}{c|}{Method}    & GCN & GAT & DenseGAT  & \model-s & \model-a \\
\midrule
\multirow{2}{*}{Cora} & Train time (s)   &  0.0584 & 0.0807 & 0.5165  & 0.1438 & 0.3292 \\
 & Training memory (MB) & 1168 & 1380 & 8460 &  1350 & 3893 \\
\midrule
\multirow{2}{*}{Pokec} & Train time (s)   &  1.069 & 14.87 & 88.07 &  2.206 & OOM \\
& Training memory (MB) & 1812 & 2014 & 13174 &  2923 & OOM \\
\midrule
\multirow{2}{*}{STL} & Train time (s)   & 0.0069  & 0.0424 & OOM  & 0.0323 & 0.3298 \\
 & Training memory (MB) & 1224 & 1980 & OOM &  1342 & 7680 \\
 \midrule
\multirow{2}{*}{WikiMath} & Train time (s)   &  0.0081 & 0.0261 & 0.0364  & 0.0281 & 0.0350 \\
 & Training memory (MB) & 1048 & 1054 & 1316 &  1046 & 1142 \\
 \bottomrule
\end{tabular}
}}
\end{table}

\subsection{Incorporation of Pseudo Labels}\label{appx-pl}

For semi-supervised learning, there is a line of approaches that leverage pseudo labels to augment the training data. Our model \model essentially has orthogonal technical aspects compared to this line of work in that we focus on building a new encoder backbone and only train the model with a standard supervised loss on the labeled data. This means that pseudo-label-based approaches are equally applicable for enhancing the training of our model as well as the competitors we used in our experiments.

As an initial verification of this claim, we use the Meta Pseudo Labels (MPL)~\citep{mpl} as a plug-in module to boost \model as well as our competitors GCN-kNN and GAT-kNN, and empirically compare the relative improvement. Specifically, we use \model-s, \model-a, GCN and GAT as the encoder backbone of teacher and student models, respectively, and use the MPL algorithm to generate pseudo labels for augmenting the training data used for computing the supervised loss. The results on \texttt{CIFAR-10} and \texttt{STL-10} are shown in Table~\ref{tbl-pl}. As we can see, the MPL contributes to some performance gains across all four encoders, while our two \model variants still maintain the superiority over the competitors. Note also that as a proof-of-concept here we did not use an additional consistency loss that requires careful manual tuning. However, in practice this type of more sophisticated MPL implementation could in principle be applied to further improve performance (across all models).

\begin{table}[htb]
\small
\centering
%\small
\caption{Comparison of using and not using Meta Pseudo Labels (MPL) as a plug-in module to boost different encoder backbones on \texttt{CIFAR-10} and \texttt{STL-10}.}\label{tbl-pl}
{
\setlength{\tabcolsep}{1.5pt}
\centering
\resizebox{0.65\textwidth}{!}{
\begin{tabular}{c|c|c|c|c|c}
\toprule
 \multicolumn{2}{c|}{Method}    & GCN & GAT & \model-s & \model-a \\
\midrule
\multirow{2}{*}{STL} & w/o MPL & $73.7 \pm 0.4$ & $73.9 \pm 0.6$ & $76.4 \pm 0.5$ & $75.3 \pm 0.6$  \\
& w/ MPL & $74.3 \pm 0.5$ & $74.5 \pm 0.7$ & $77.0 \pm 0.6$ & $75.9 \pm 0.4$  \\
\midrule
\multirow{2}{*}{CIFAR}& w/o MPL & $74.7 \pm 0.5$ & $74.1 \pm 0.5$ & $76.6 \pm 0.3$ & $75.9 \pm 0.3$ \\
& w/ MPL & $75.3 \pm 0.4$ & $74.8 \pm 0.5$ & $77.1 \pm 0.3$ & $76.3 \pm 0.3$ \\
 \bottomrule
\end{tabular}
}}
\end{table}

\begin{figure}[t!]
\centering
\begin{minipage}{\linewidth}
\centering
\subfigure[The first layer]{
\begin{minipage}[t]{0.48\linewidth}
\centering
\includegraphics[width=0.9\linewidth,trim=0cm 0cm 0cm 0cm,clip]{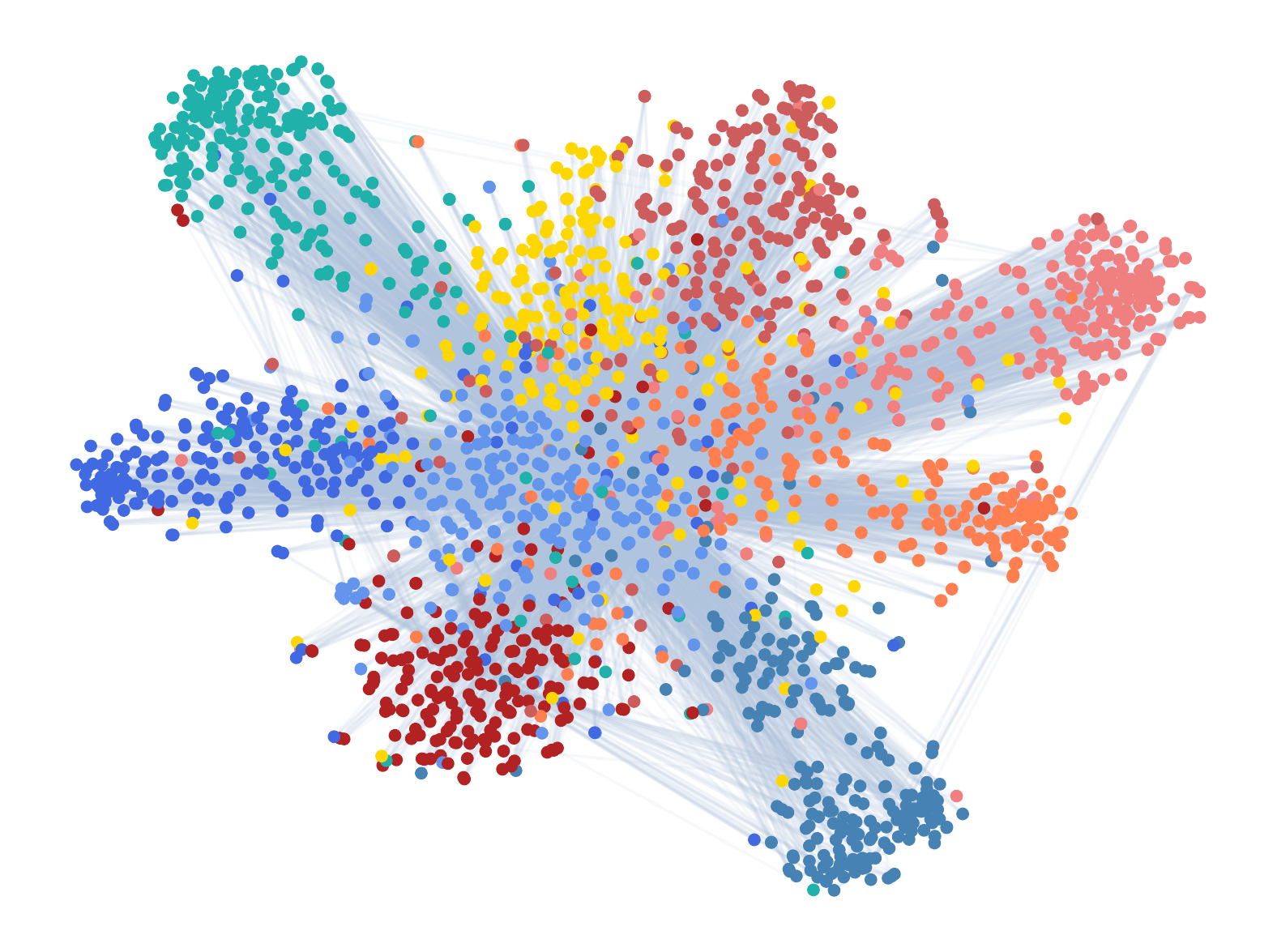}
\end{minipage}
}
\subfigure[The last layer]{
\begin{minipage}[t]{0.48\linewidth}
\centering
\includegraphics[width=0.9\linewidth,trim=0cm 0cm 0cm 0cm,clip]{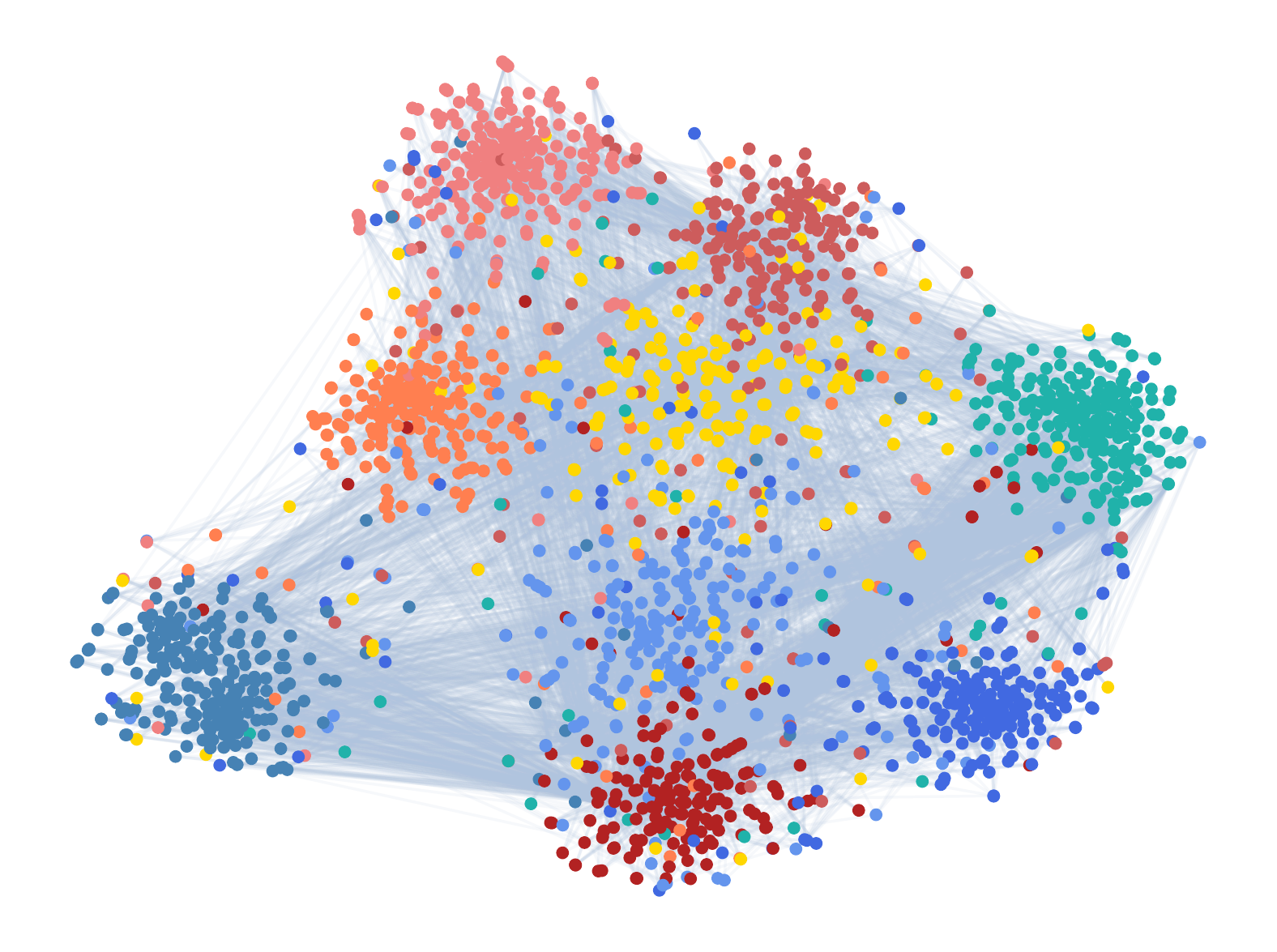}
\end{minipage}
}
\end{minipage}
\vspace{-10pt}
\caption{Visualization of instance representations and diffusivity strengths (we set a threshold and only plot the edges with weights more than the threshold) at different layers given by \model-s on \texttt{20News}.}
\label{fig:vis-20news}
\vspace{-15pt}
\end{figure}

\begin{figure}[t!]
\centering
\begin{minipage}{\linewidth}
\centering
\subfigure[The first layer]{
\begin{minipage}[t]{0.48\linewidth}
\centering
\includegraphics[width=0.9\linewidth,trim=0cm 0cm 0cm 0cm,clip]{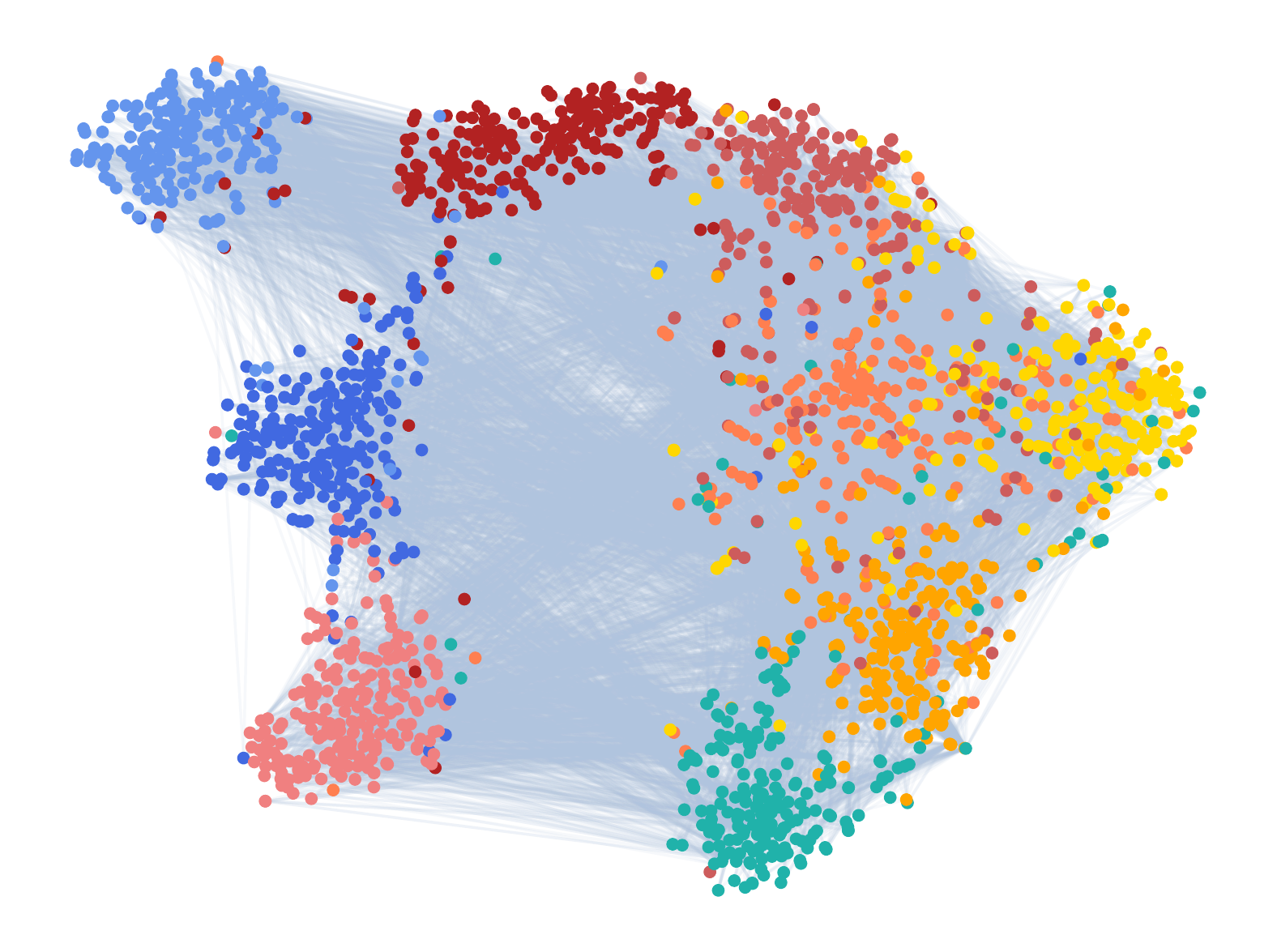}
\end{minipage}
}
\subfigure[The last layer]{
\begin{minipage}[t]{0.48\linewidth}
\centering
\includegraphics[width=0.9\linewidth,trim=0cm 0cm 0cm 0cm,clip]{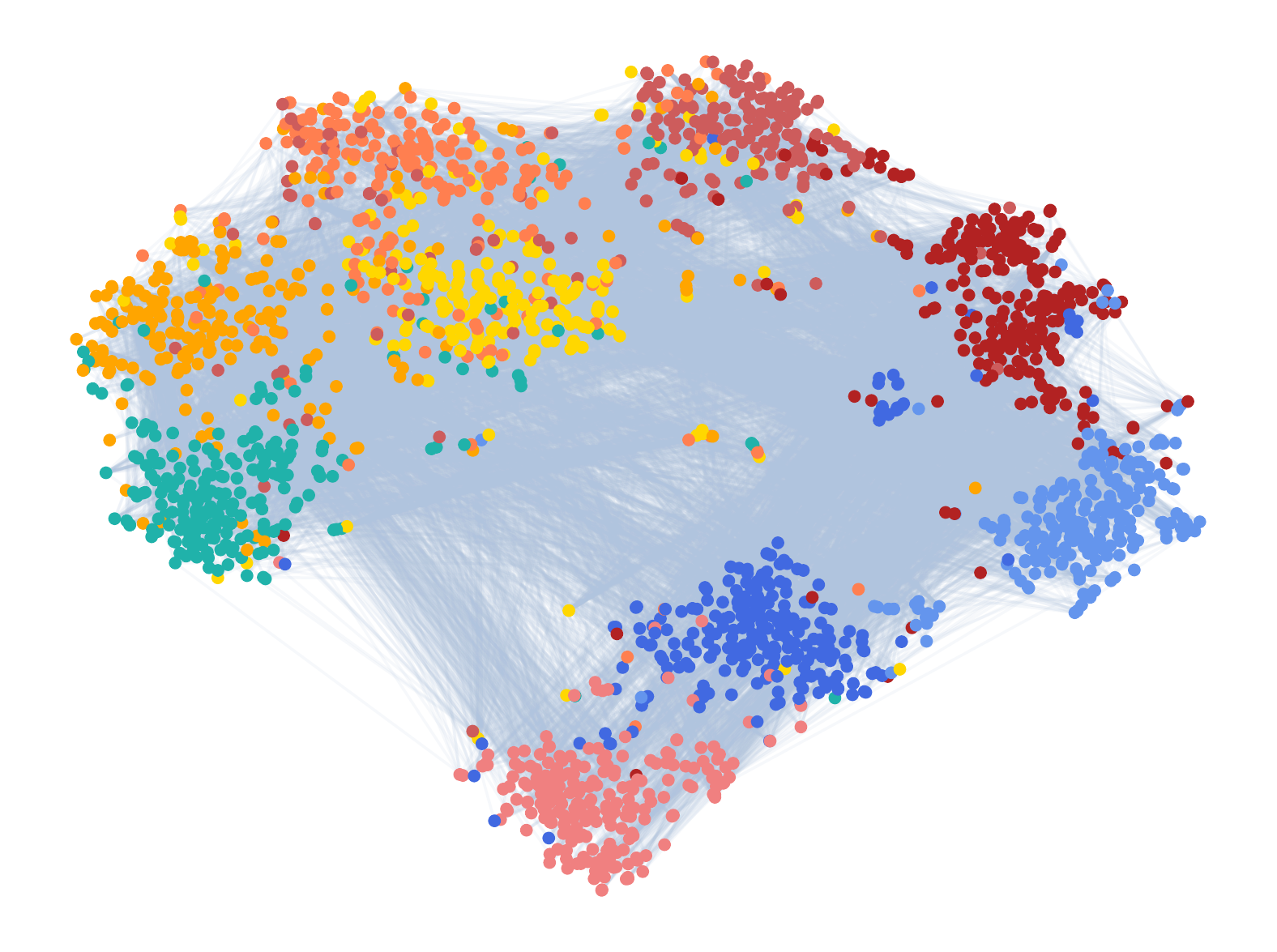}
\end{minipage}
}
\end{minipage}
\vspace{-10pt}
\caption{Visualization of instance representations and diffusivity strengths (we set a threshold and only plot the edges with weights more than the threshold) at different layers given by \model-s on \texttt{STL}.}
\label{fig:vis-stl}
\vspace{-15pt}
\end{figure}

% \subsection{More Comparison with Graph Transformer Models}

% The implementation of the diffusivity inference appears to be related with the attention in Transformer, though motivated from different perspectives. We thus add more comparison with graph-transformer (GT) models. Yet we notice that current graph-transformers are mostly focused on graph-level tasks (where graphs are instances with small sizes) and our model aims at node-level prediction on larger graphs where existing GTs are hard to scale. Even so, we adapt the recently proposed Graph Transformer (GT)~\cite{graphformer-neurips21} as a baseline to node-level prediction, using the ego-graph-based mini-batch partition used by a recent survey paper~\cite{gt-survey} to make it scalable for training. The results on \texttt{Proteins} and \texttt{Pokec} datasets are shown below.

% \begin{table}[htb]
% \small

% \centering
% %\small
% \caption{Comparison with Graph Transformer (GT) on \texttt{Proteins} and \texttt{Pokec}.}\label{tab-time}
% {
% \setlength{\tabcolsep}{1.5pt}
% \centering
% \resizebox{0.9\textwidth}{!}{
% \begin{tabular}{c|c|c|c|c|c|c|c}
% \toprule
%    & MLP  & LP  & SGC  & GCN  & GAT  & GT  & ENDiff-l \\
% \midrule
% Proteins  &  72.41 ± 0.10  & 74.73  & 49.03 ± 0.93   & 74.22 ± 0.49  & 75.11 ± 1.45   & 77.21 ± 0.42  & 79.49 ± 0.44 \\
% Pokec  & 60.15 ± 0.03  & 52.73   & 52.03 ± 0.84  & 62.31 ± 1.13  & 65.57 ± 0.34  & 67.18 ± 0.24  & 69.24 ± 0.76   \\
%  \bottomrule
% \end{tabular}
% }}
% \end{table}

\begin{figure}[t!]
\centering
\begin{minipage}{\linewidth}
\centering
\subfigure[]{
\begin{minipage}[t]{0.148\linewidth}
\centering
\includegraphics[width=0.9\linewidth,trim=0cm 0cm 0cm 0cm,clip]{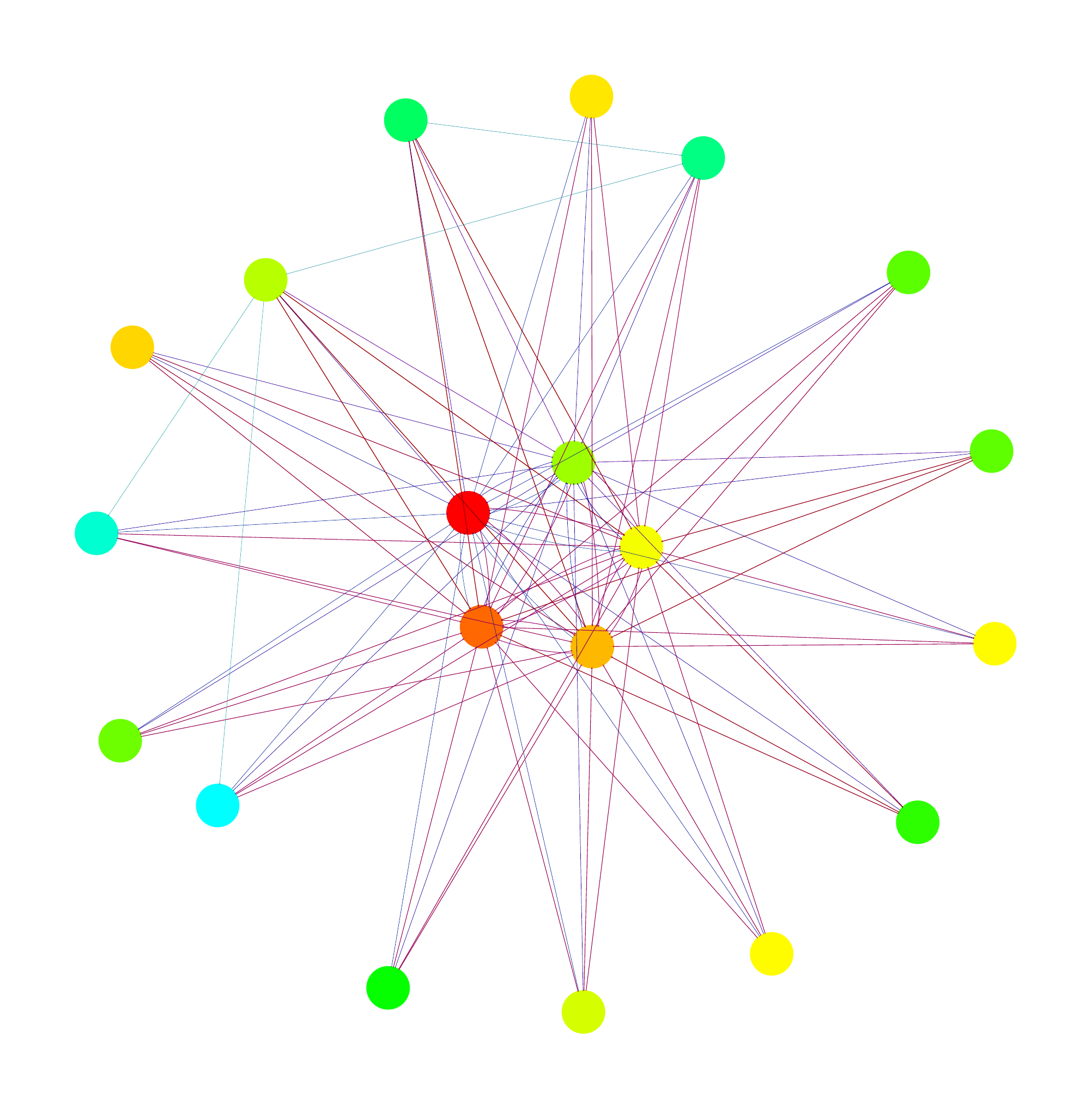}
\end{minipage}
}
\subfigure[]{
\begin{minipage}[t]{0.148\linewidth}
\centering
\includegraphics[width=0.9\linewidth,trim=0cm 0cm 0cm 0cm,clip]{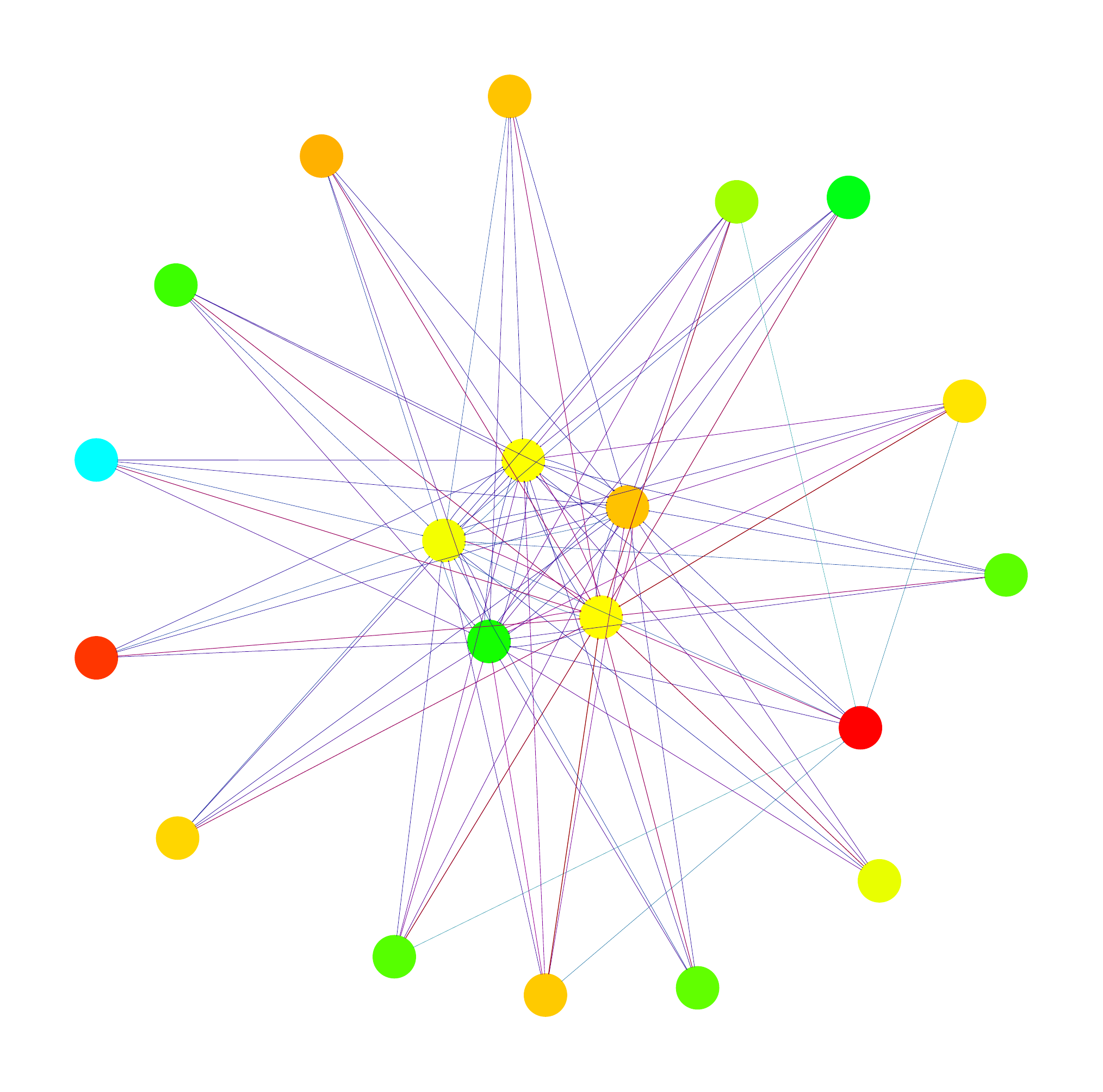}
\end{minipage}
}
\subfigure[]{
\begin{minipage}[t]{0.148\linewidth}
\centering
\includegraphics[width=0.9\linewidth,trim=0cm 0cm 0cm 0cm,clip]{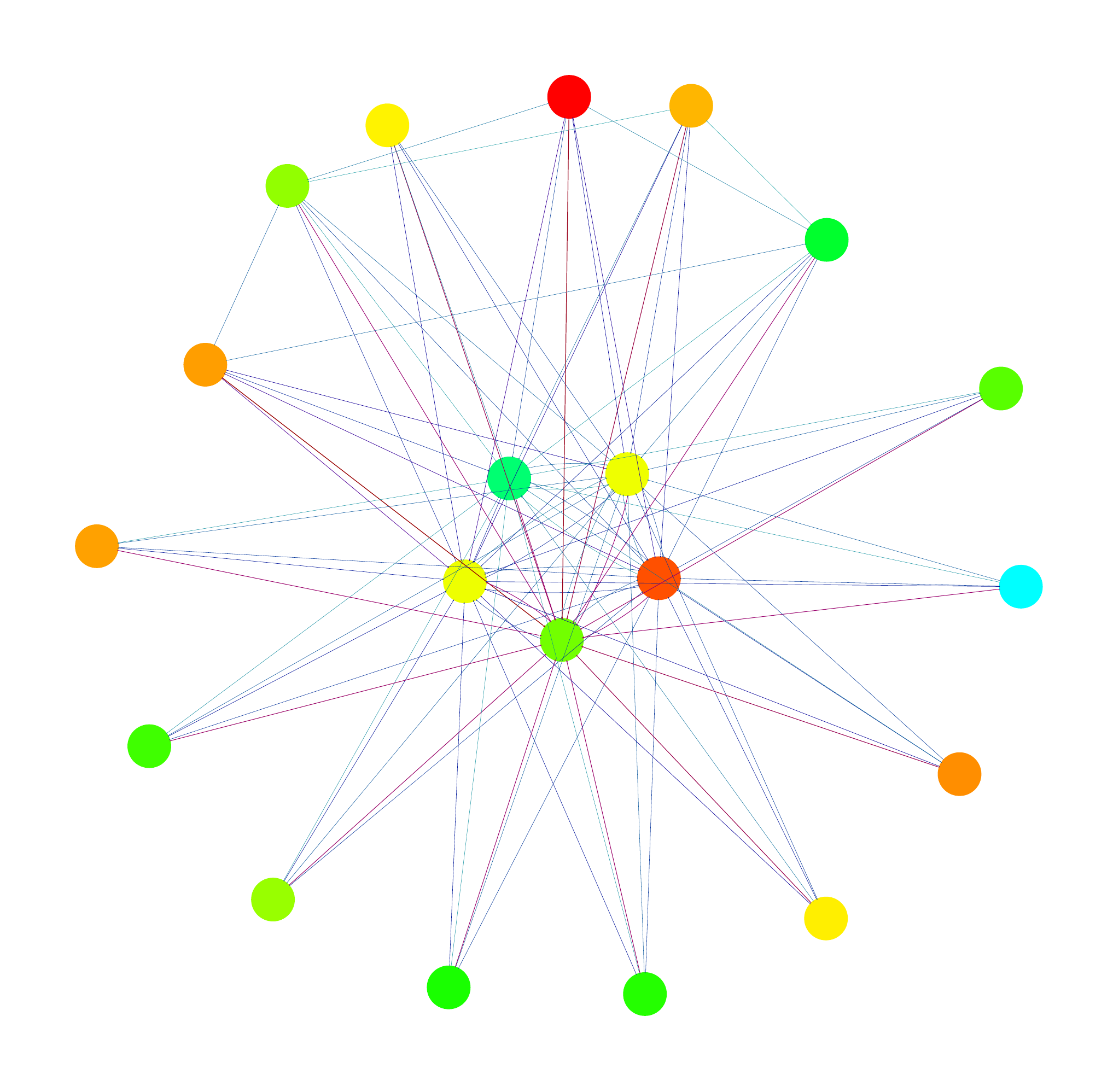}
\end{minipage}
}
\subfigure[]{
\begin{minipage}[t]{0.148\linewidth}
\centering
\includegraphics[width=0.9\linewidth,trim=0cm 0cm 0cm 0cm,clip]{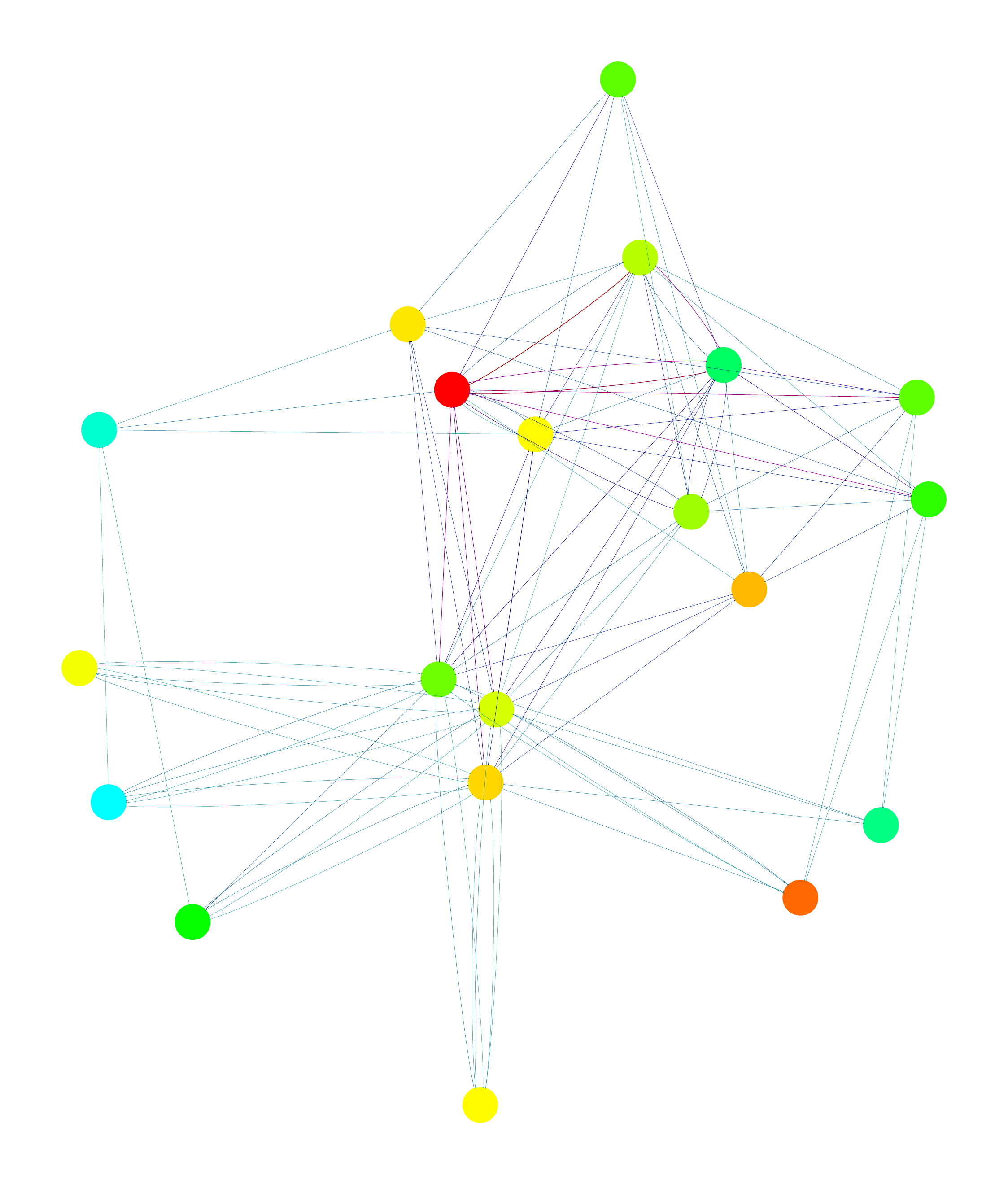}
\end{minipage}
}
\subfigure[]{
\begin{minipage}[t]{0.148\linewidth}
\centering
\includegraphics[width=0.9\linewidth,trim=0cm 0cm 0cm 0cm,clip]{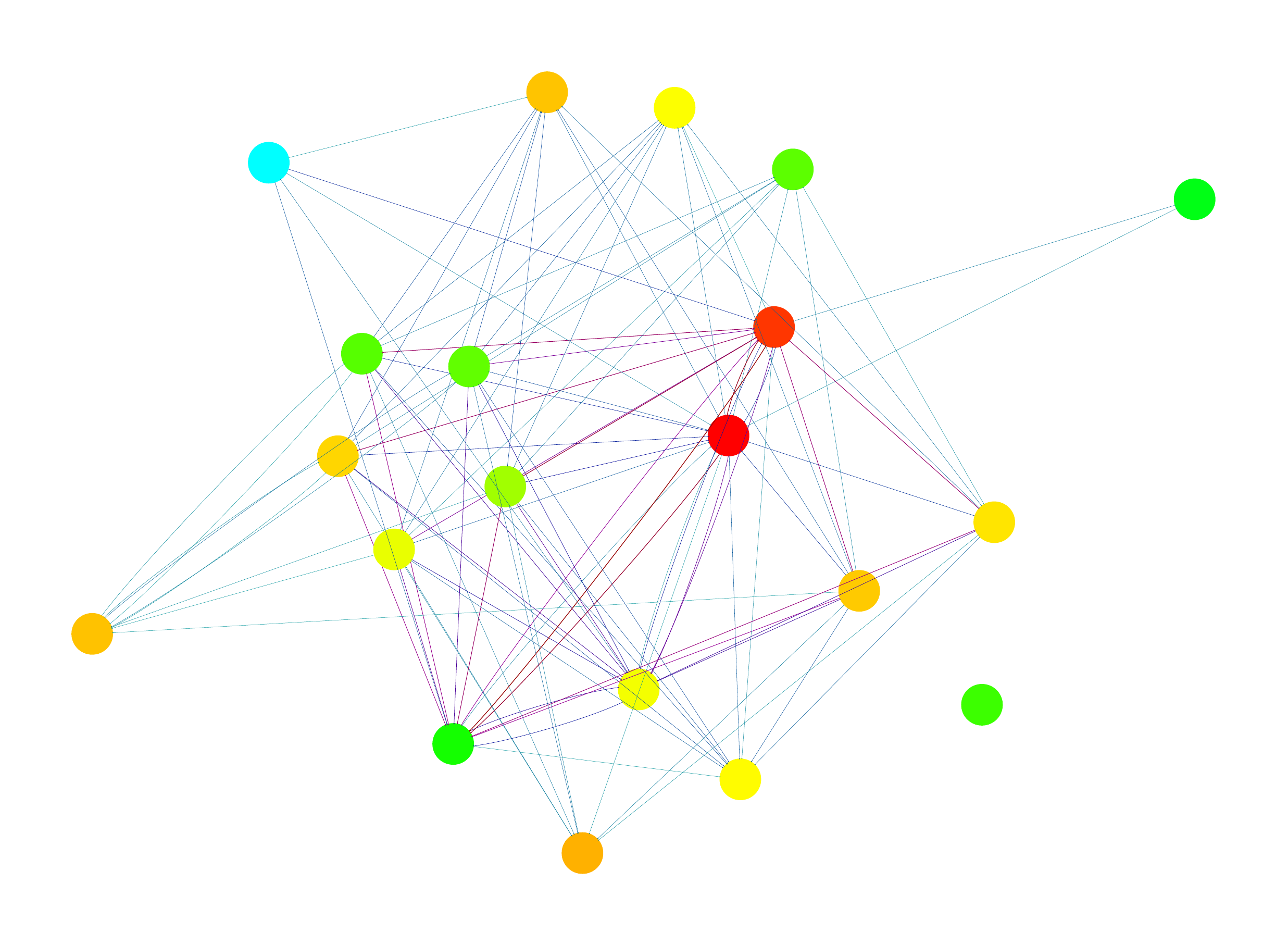}
\end{minipage}
}
\subfigure[]{
\begin{minipage}[t]{0.148\linewidth}
\centering
\includegraphics[width=0.9\linewidth,trim=0cm 0cm 0cm 0cm,clip]{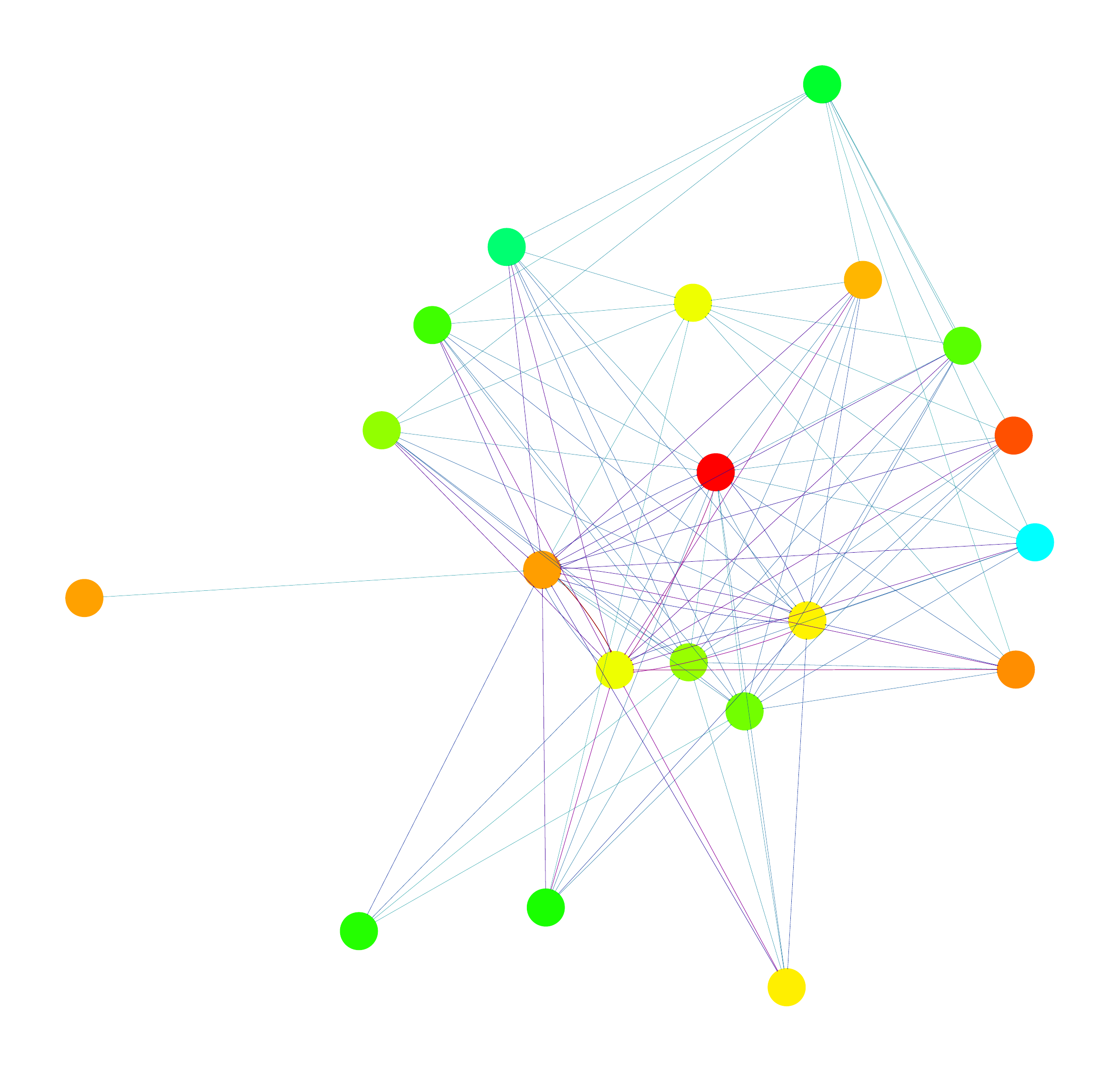}
\end{minipage}
}
\end{minipage}
\vspace{-10pt}
\caption{The produced diffusivity of the first layer (i.e., $\mathbf {\hat S}^{(1)}$) on \texttt{Chickenpox} across the first three snapshots, yielded by \model-s, shown
 in (a)$\sim$(c), and \model-a, shown in (d)$\sim$(f). Node colors correspond to ground-truth labels (i.e., reported cases), varying from red to blue as the label increases. We visualize the edges with top 100 diffusion strength, where edge colors change from blue to red as $\mathbf {\hat S}^{(1)}_{ij}$ increases.}
\label{chickenpox_main_text_visualization}
\vspace{-15pt}
\end{figure}

\end{document}